\documentclass[twoside,11pt]{article}

\usepackage[abbrvbib]{jmlr2e}

\usepackage{amsmath,amsfonts}
\usepackage{mathtools}

\usepackage{lastpage}

\usepackage{enumerate,color,xcolor}
\usepackage{graphicx}
\usepackage{subcaption}
\usepackage{url}
\usepackage{multirow}

\numberwithin{equation}{section}

\newcommand\tab[1][1cm]{\hspace*{#1}}
\usepackage{comment}

\numberwithin{equation}{section}

\usepackage{caption}
\graphicspath{ {./images/} }
\firstpageno{1}

\usepackage{lastpage}

\begin{document}

\title{Wasserstein Convergence Guarantees for a General Class of Score-Based Generative Models}

\author{\name Xuefeng Gao \email xfgao@se.cuhk.edu.hk \\
       \addr Department of Systems Engineering and Engineering Management\\
       The Chinese University of Hong Kong\\
       Shatin, NT, Hong Kong
       \AND
       \name Hoang M. Nguyen \email hmnguyen@fsu.edu \\
       \addr 
       Department of Mathematics \\
       Florida State University \\
       Tallahassee, FL 32306, United States of America
       \AND
       \name Lingjiong Zhu \email zhu@math.fsu.edu \\
       \addr Department of Mathematics \\
       Florida State University \\
       Tallahassee, FL 32306, United States of America\\
       }


\editor{}

\maketitle

\begin{abstract}
Score-based generative models (SGMs) are a recent class of deep generative models with state-of-the-art performance in many applications. In this paper, we establish convergence guarantees for a general class of SGMs in the 2-Wasserstein distance, assuming accurate score estimates and smooth log-concave data distribution. We specialize our results to several concrete SGMs with specific choices of forward processes modeled by stochastic differential equations, and obtain an upper bound on the iteration complexity for each model, which demonstrates the impacts of different choices of the forward processes. We also provide a lower bound when the data distribution is Gaussian. Numerically, we experiment with SGMs with different forward processes for unconditional image generation on CIFAR-10. We find that the experimental results are in good agreement with our theoretical predictions on the iteration complexity.
\end{abstract}


\section{Introduction}

Diffusion models are a powerful family of probabilistic generative models which can generate approximate samples from high-dimensional distributions \citep{Sohl2015, SongErmon2019, Ho2020}. The key idea in diffusion models is to use a forward process to progressively corrupt samples from a target data distribution with noise and then learn to reverse this process for generation of new samples. 
Diffusion models have achieved state-of-the-art performance in various applications such as image and audio generations, and they are the main components in popular content generators including Stable Diffusion \citep{rombach2022high} and Dall-E 2 \citep{ramesh2022hierarchical}. 
We refer the readers to \citep{yang2022diffusion, croitoru2023diffusion} for surveys on diffusion models.

A predominant formulation of diffusion models is score-based generative models (SGM) through Stochastic Differential Equations (SDEs) \citep{SongICLR2021}, referred to as \textit{Score SDEs} in \cite{yang2022diffusion}. At the core of this formulation, there are two stochastic processes in $\mathbb{R}^{d}$: a forward process and a reverse process. In this paper, we consider a general class of forward process $(\mathbf{x}_{t})_{t\in[0,T]}$ described by the following SDE:
\begin{equation}\label{OU:SDE}
d\mathbf{x}_{t}=-f(t)\mathbf{x}_{t}dt+g(t)d\mathbf{B}_{t}, \quad  \mathbf{x}_{0}\sim p_{0},
\end{equation}
where both $f(t)$ and $g(t)$ are scalar-valued non-negative continuous functions of time $t$, $(\mathbf{B}_{t})$ is the standard $d$-dimensional Brownian motion, and $p_0$ is the (unknown) target data distribution. The forward process has the interpretation of slowly injecting noise to data and transforming them to a noise-like distribution. 
If we reverse the forward process \eqref{OU:SDE} in time, i.e., letting $(\tilde{\mathbf{x}}_{t})_{t\in[0,T]}=(\mathbf{x}_{T-t})_{t\in[0,T]}$, 
then under mild assumptions, the reverse process $(\tilde{\mathbf{x}}_{t})_{t\in[0,T]}$ satisfies another SDE (see e.g. \cite{Anderson1982, cattiaux2023time}):
\begin{equation}\label{eq:Reverse} 
d\tilde{\mathbf{x}}_{t}=\left[f(T-t)\tilde{\mathbf{x}}_{t}+(g(T-t))^{2}\nabla_{\mathbf{x}}\log p_{T-t}(\tilde{\mathbf{x}}_{t})\right]dt
+g(T-t)d\bar{\mathbf{B}}_{t}, \quad \tilde{\mathbf{x}}_{0}\sim p_{T},
\end{equation}
where $p_{t}(\mathbf{x})$ is the probability density function of $\mathbf{x}_t$ (the forward process at time $t$), $(\bar{\mathbf{B}}_{t})$ is a standard Brownian motion in $\mathbb{R}^{d}$ and  $\tilde{\mathbf{x}}_{T} = \mathbf{x}_{0} \sim p_{0}$. Hence, the reverse process \eqref{eq:Reverse} transforms a noise-like distribution $p_T$ into samples from $p_0$, which is the goal of generative modeling. 
Note that the reverse SDE \eqref{eq:Reverse} involves the \textit{score function}, $\nabla_{\mathbf{x}}\log p_{t}(\mathbf{x})$, which is unknown. An important subroutine in SGM is to estimate this score function based on samples drawn from the forward process, typically by modeling time-dependent score functions as neural networks and training them on certain score matching objectives \citep{ hyvarinen2005estimation, vincent2011connection, song2020sliced}. After the score is estimated, one can numerically solve the reverse SDE to generate new samples that approximately follows the data distribution (see Section~\ref{sec:prelim}). 
The Score-SDEs formulation is attractive because it generalizes and unifies several well-known diffusion models. In particular, the noise perturbation in Score Matching with Langevin dynamics (SMLD) \citep{SongErmon2019} corresponds to the discretization of so-called variance exploding (VE) SDEs where $f \equiv 0$ in \eqref{OU:SDE}; The noise perturbation in Denoising Diffusion Probabilistic Modeling (DDPM) \citep{Sohl2015,Ho2020} corresponds to an appropriate discretization of the variance preserving (VP) SDEs where $f(t) = \frac{1}{2}\beta(t)$ and $g(t) = \sqrt{ \beta(t) }$ for some noise variance schedule $\beta(t)$, see \cite{SongICLR2021} for details.

Despite the impressive empirical performances of diffusion models in various applications, the theoretical understandings of these models are relatively limited. 
In the past few years, there has been a rapidly growing body of literature on the convergence theory of diffusion models, assuming access to accurate estimates of the score function, see e.g. 
\cite{block2020generative, de2022convergence, leeconvergence, lee2023convergence, chen2022improved, chen2022sampling,  benton2023linear, chen2023score, tang24}. While these studies have established polynomial convergence bounds for fairly general data distribution,
this line of work has mostly focused on the analysis of the DDPM model where the corresponding forward SDE \eqref{OU:SDE} satisfies $f = g^2/{2}>0$ (in fact many studies simply consider $f\equiv 1/2$ and $g\equiv 1$). However, it is important to understand the impacts of different choices of forward processes in diffusion models, which can potentially help provide theoretical guidance for selecting the functions $f$ and $g$ in the design of diffusion models. In this work, we make some progress towards addressing these issues using a theoretical approach based on convergence analysis. We summarize our contributions in the following.

\textbf{Our Contributions.}
\begin{itemize}
\item We establish convergence guarantees for a general class of SGMs in the 2-Wasserstein distance, assuming accurate score estimates and smooth log-concave data distribution ({with unbounded support}), see Theorem~\ref{thm:discrete:2}. In particular, we allow general functions $f$ and $g$ in the forward SDE \eqref{OU:SDE}.  Theorem~\ref{thm:discrete:2} directly translates to an upper bound on the iteration complexity, which is the number of sampling steps/iterations needed (in running the reverse process) to yield $\epsilon-$accuracy in the 2-Wasserstein distance between the data distribution and the generative distribution of the SGMs. 

\item We specialize our result to SGMs with specific functions $f$ and $g$ in the forward process. We find that under mild assumptions,  
the class of VP-SDEs (as forward processes) will lead to an iteration complexity bound $\widetilde{\mathcal{O}}\left(d/\epsilon^{2}\right)$
where $\widetilde{\mathcal{O}}$ ignores the logarithmic factors and hides dependency on other parameters. On the other hand, in the class of VE-SDEs, the choice of an exponential function $g$ in \cite{SongICLR2021} leads to an iteration complexity of $\widetilde{\mathcal{O}}\left(d/\epsilon^{2}\right)$, while other simple choices including polynomials for $g$ lead to a worse complexity bound. We also find that VP-SDEs with polynomial and exponential noise schedules, which appear to be new to the literature, lead to better iteration complexity bounds (in terms of logarithmic dependence on $d,\epsilon$) compared with the existing models. See Table~\ref{table:summary} and Proposition~\ref{prop:VP}.

\item We also establish two new results on lower bounds. We first show that if we use the upper bound in Theorem~\ref{thm:discrete:2}, 
then in order to achieve $\epsilon$ accuracy, 
 the iteration complexity is 
$\widetilde{\Omega}\left(d/\epsilon^{2}\right)$ for quite general functions $f$ and $g$,
where $\widetilde{\Omega}$ ignored the logarithmic dependence on $\epsilon$ and $d$ (Proposition~\ref{prop:lower:bound}). This result, however, does not show whether our upper bound in Theorem~\ref{thm:discrete:2} is tight or not. We next show that if the data distribution $p_{0}$ is Gaussian, then the lower bound for 
the iteration complexity is $\Omega\left(\sqrt{d}/\epsilon\right)$ (Proposition~\ref{prop:lower:bound:Gaussian}). 

\item Numerically, we experiment SGMs with different forward SDEs for unconditional image generation on the CIFAR-10 image dataset, using the neural network architectures from \cite{SongICLR2021}. We find that the experimental results are in good agreement with our theoretical predictions: models with lower order of iteration complexity generally perform better, in the sense that they achieve lower FID scores and higher Inception scores (with the same stochastic sampler and number of sampling steps) over training iterations.

\item 
Our main proof techniques (for the upper bound on the iteration complexity) rely on obtaining an explicit contraction rate for the reverse SDE in Wasserstein distance using properties of strongly log-concave distributions and It\^{o}'s formula, and controlling the discretization and score-matching errors by using synchronous coupling. This approach is significantly different from the existing studies on convergence analysis of SDE-based samplers such as \cite{chen2022sampling}, where Girsanov theorem and data processing inequality are used to obtain convergence guarantees in total variation distance or Kullback-Leibler (KL) divergence.  Their techniques require the forward process to be contractive that excludes VE-SDEs, whereas our methodology enables us to cover a general class of models including both VP and VE-SDEs as special cases. We also emphasize that the reverse SDE is non-homogeneous in nature due to the score function, and one needs to perform
a delicate analysis based on the result in Theorem~\ref{thm:discrete:2} in order to spell out the leading order terms to obtain bounds on the iteration complexities for various VE-SDE and VP-SDE examples.
\end{itemize}

\subsection{Related Work}

The majority of the existing studies on the convergence analysis of diffusion models focus on the SDE-based implementation, that is, the sampling process for data generation is based on the discretization of the reverse-time SDE.  
For instance, \citep{lee2023convergence, chen2022improved, chen2022sampling, benton2023linear} have established polynomial convergence rates in Total Variation (TV) distance or KL divergence for the DDPM model, where they consider $f\equiv 1/2$ and $g\equiv 1$ in the forward SDE. 
In contrast to these studies, our work provides a unified convergence analysis for a more general class of diffusion models in the 2-Wasserstein distance ($\mathcal{W}_{2}$), thus illustrating how the choice of $f$ and $g$ impacts the iteration complexity. Focusing on $\mathcal{W}_{2}$ distance is not just of theoretical interest, but also of practice interest. Indeed, one of the most popular performance metrics for the quality of generated samples in image applications is Fr\'{e}chet Inception Distance (FID), which measures the $\mathcal{W}_{2}$ distance between the distributions of generated images with the distribution of real images \citep{heusel2017gans}. It is known that to obtain polynomial convergence rates in $\mathcal{W}_{2}$ distance, one often needs to assume some form of log-concavity for the data distribution, see e.g. \cite[Section 4]{chen2022sampling} for discussions. Hence, we impose such (strong) assumptions on the data distribution. 
Although some studies (see e.g. \cite{chen2022sampling}) also provided Wasserstein convergence bound for the DDPM model, they often assume that the data distribution is bounded, in which case the $\mathcal{W}_{2}$ distance can be bounded by the TV distance. In this work, we consider unbounded data distribution. 
We also emphasize that TV distance does not upper bound Wasserstein distance on $\mathbb{R}^{d}$ (see e.g. \cite{GibbsSu2002}) and KL divergence does not imply Wasserstein convergence either (unless some additional conditions are satisfied \citep{BV}). Hence, our results are not implied by the existing convergence results for SDE-based samplers. 

There is also a small but rapidly growing body of literature on convergence analysis of the probability flow ODE implementation of diffusion models \citep{SongICLR2021}. See, e.g., \cite{chen2023restoration, chen2023probability, li2023towards, gao2024convergence, li2024accelerating}. 
Our work differs from this line of studies in that we consider SDE-based samplers instead of ODE-based samplers. 

Finally, our work is broadly related to the literature on the choice of noise schedules for diffusion models. 
In classical models, the choice of the forward process, i.e. $f$ and $g$ in \eqref{OU:SDE}, is often handcrafted and designed heuristically based on numerical performances of the corresponding models. For instance, the majority of existing DDPM models (in which $f=g^2/2$) use the linear noise schedule (i.e. $(g(t))^2=b+at$ for some $a, b>0$) proposed firstly in \cite{Ho2020}. \cite{Nichol2021} proposed a cosine noise schedule and showed that it can improve the log-likelihood numerically. For the SMLD model, the noise schedule is often chosen as an exponential function, i.e. $g(t)=ab^t$ for some $a, b>0$, following \cite{SongErmon2019, song2020improved}.  
In contrast to these studies, our work provides a theoretical analysis on the impact of
different choices of forward processes based on the convergence analysis of diffusion models in $\mathcal{W}_{2}$ distance. 

\textbf{Notations.} 
For any $d$-dimensional random vector $\mathbf{z}$ with finite second moment, 
the $L_{2}$-norm of $\mathbf{z}$ is defined as
$\Vert\mathbf{z}\Vert_{L_{2}}=\left(\mathbb{E}\Vert\mathbf{z}\Vert^{2}\right)^{1/2}$, 
where $\Vert\cdot\Vert$ denotes the Euclidean norm. 
We denote $\mathcal{L}(\mathbf{z})$ as the law of $\mathbf{z}$.
Define $\mathcal{P}_{2}(\mathbb{R}^{d})$
as the space consisting of all the Borel probability measures $\nu$
on $\mathbb{R}^{d}$ with the finite second moment
(based on the Euclidean norm).
For any two Borel probability measures $\nu_{1},\nu_{2}\in\mathcal{P}_{2}(\mathbb{R}^{d})$, 
the standard $2$-Wasserstein
distance \cite{villani2008optimal} is defined by
$\mathcal{W}_{2}(\nu_{1},\nu_{2}):=\left(\inf\mathbb{E}\left[\Vert\mathbf{z}_{1}-\mathbf{z}_{2}\Vert^{2}\right]\right)^{1/2},$
where the infimum is taken over all joint distributions of the random vectors $\mathbf{z}_{1},\mathbf{z}_{2}$ with marginal distributions $\nu_{1},\nu_{2}$. Finally, a differentiable function $F:\mathbb{R}^{d}\rightarrow\mathbb{R}$ is 
$\mu$-strongly convex and $L$-smooth (i.e. $\nabla F$ is $L$-Lipschitz) if for every $\mathbf{x},\mathbf{y}\in\mathbb{R}^{d}$,
$\frac{L}{2}\Vert\mathbf{x}-\mathbf{y}\Vert^{2} \geq F(\mathbf{x})-F(\mathbf{y}) - \nabla F(\mathbf{y})^{\top}(\mathbf{x}-\mathbf{y}) \geq \frac{\mu}{2}\Vert \mathbf{x}-\mathbf{y}\Vert^{2}$.

\section{Preliminaries on SDE-Based Diffusion Models}\label{sec:prelim}

We first recall the background on score-based generative modeling with SDEs \citep{SongICLR2021}. Denote by $p_{0} \in \mathcal{P} (\mathbb{R}^d) $ the unknown continuous data distribution, where $\mathcal{P}(\mathbb{R}^{d})$ is the space of all probability measures on $\mathbb{R}^{d}$. Given i.i.d. samples from $p_{0}$, the goal is to generate new samples whose distribution closely resembles the data distribution.

\textbf{Forward process and reverse process.}
Let $T>0$. We consider a $d-$dimensional forward process $(\mathbf{x}_{t})_{t \in [0, T]}$ given in \eqref{OU:SDE}.
One can easily solve \eqref{OU:SDE} to obtain
\begin{equation}\label{SDE:solution}
\mathbf{x}_{t}=e^{-\int_{0}^{t}f(s)ds}\mathbf{x}_{0}+\int_{0}^{t}e^{-\int_{s}^{t}f(v)dv}g(s)d\mathbf{B}_{s}.
\end{equation}
Denote by $p_t(\mathbf{x})$ the probability density of $\mathbf{x}_{t}$ for $t \ge 0.$ Before we proceed, we give two popular classes of the forward SDEs in the literature (\cite{SongICLR2021}):
\begin{itemize}
\item Variance Exploding (VE) SDE: $f(t) \equiv 0$ and $g(t) = \sqrt{\frac{d[\sigma^2(t)]}{dt}}$ for some nondecreasing function $\sigma(t)$, e.g., $g(t)=ae^{bt}$ for some positive constants $a, b$.
\item Variance Preserving (VP) SDE:  $f(t) = \frac{1}{2}\beta(t) $ and $g(t) = \sqrt{ \beta(t) }$ for some nondecreasing {function} $\beta(t)$, e.g., $\beta(t) = at +b$ for some positive constants $a, b.$
\end{itemize}
Under mild assumptions, the reverse (in time) process $(\tilde{\mathbf{x}}_{t})_{t\in[0,T]}=(\mathbf{x}_{T-t})_{t\in[0,T]}$ satisfies the SDE in \eqref{eq:Reverse}.
Hence, by starting from samples of $p_T$, we can run the SDE \eqref{eq:Reverse} to time $T$ and obtain samples from the desired distribution $p_0$. 
However, the distribution $p_{T}$ is not explicit and hard to sample from because of its dependency on the initial distribution $p_{0}$.
With the choice of our forward SDE \eqref{OU:SDE}, which has an explicit solution \eqref{SDE:solution}
such that we can take
\begin{equation} \label{eq:hatp}
\hat{p}_{T}:=\mathcal{N}\left(0,\int_{0}^{T}e^{-2\int_{t}^{T}f(s)ds}(g(t))^{2}dt\cdot I_{d}\right),
\end{equation}
as an approximation of $p_T$, where $I_d$ is the $d-$dimensional identity matrix. Note that $\hat{p}_{T}$ is simply the distribution of the random variable $\int_{0}^{t}e^{-\int_{s}^{t}f(v)dv}g(s)d\mathbf{B}_{s}$ in 
\eqref{SDE:solution}, which is easy to sample from because it is Gaussian, and will be referred to as the prior distribution hereafter. Note that it directly follows from \eqref{SDE:solution} that
\begin{equation}\label{eq:phatp} 
\mathcal{W}_{2}(p_{T},\hat{p}_{T})
\leq
e^{-\int_{0}^{T}f(s)ds}\Vert\mathbf{x}_{0}\Vert_{L_{2}}.
\end{equation}
In view of \eqref{eq:Reverse}, we now consider the SDE:
\begin{equation}\label{eq:zt}
d\mathbf{z}_{t}=\left[f(T-t)\mathbf{z}_{t}+(g(T-t))^{2}\nabla\log p_{T-t}(\mathbf{z}_{t})\right]dt
+g(T-t)d\bar{\mathbf{B}}_{t}, \quad \mathbf{z}_{0}\sim\hat{p}_{T}.
\end{equation}
Because $p_{T}\neq\hat{p}_{T}$, this creates an error that when
running the reverse SDE, and as a result, the distribution
of $\mathbf{z}_{T}$ differs from $\tilde{\mathbf{x}}_{T}\sim p_{0}$.

\textbf{Remark.} Several studies (see, e.g. \cite{de2022convergence,chen2022sampling}) consider VP-SDE, which is a time-inhomogeneous Ornstein-Uhlenbeck (OU) process, and use the stationary distribution $p_{\infty}$ of the OU process as the prior distribution in their convergence analysis. 
If we use $p_{\infty}$ instead of $\hat{p}_{T}$ in \eqref{eq:hatp} as the prior distribution, our main result can be easily modified. Indeed, we only need to replace the error estimate \eqref{eq:phatp} by an estimate on $\mathcal{W}_{2}(p_{T},p_{\infty})$, which can be easily bounded 
(by a coupling approach) due to the geometric convergence of the OU process to its stationarity. Because VE-SDE does not have a stationary distribution, we choose the prior distribution $\hat{p}_{T}$ in \eqref{eq:hatp} in order to provide a unifying analysis based on the general forward SDE \eqref{OU:SDE} that does not require
the existence of a stationary distribution.

\textbf{Score matching.}
Next, we consider score-matching. Note that the data distribution $p_0$ is unknown, and hence the true score function $\nabla_{\mathbf{x}}\log p_{t}(\mathbf{x})$ in \eqref{eq:zt} is also unknown. In practice, it needs to be estimated/approximated by a time-dependent score model $s_{\theta}(\mathbf{x},t)$, which is often a deep neural network parameterized by $\theta$. Estimating the score function from data has established methods, including score matching \cite{ hyvarinen2005estimation}, denoising score matching \cite{vincent2011connection}, and sliced score matching \cite{song2020sliced}. For instance, \cite{SongICLR2021} use denoising score matching where the training objective for optimizing the neural network is given by
\begin{align}\label{eq:score-matching}
\min_{\theta} \mathbb{E}_{t \sim U[0, T]}  \left[ \lambda(t) \mathbb{E}_{\mathbf{x}_{0}}  \mathbb{E}_{\mathbf{x}_{t} | \mathbf{x}_{0}} \left\Vert s_{\theta}(\mathbf{x}_t,t)-\nabla_{\mathbf{x}_{t}}\log p_{t|0}(\mathbf{x}_{t} | \mathbf{x}_{0})\right\Vert^2   \right].
\end{align}
Here, $\lambda(\cdot): [0, T] \rightarrow \mathbb{R}_{>0}$ is some positive weighting function (e.g. $\lambda(t)= g(t)^2$), $U[0, T]$ is the uniform distribution on $[0, T],$ $\mathbf{x}_{0} \sim p_0$ is the data distribution, and $p_{t|0}(\mathbf{x}_{t} | \mathbf{x}_{0})$ is the density of $\mathbf{x}_{t}$ given $\mathbf{x}_{0}$. With the forward process in \eqref{OU:SDE}, we can easily infer from its solution \eqref{SDE:solution} that the transition kernel $p_{t|0}(\mathbf{x}_{t} | \mathbf{x}_{0})$ follows a Gaussian distribution where the mean the variance can be computed in closed form using $f$ and $g$.
Because we also have access to i.i.d. samples from $p_0$, the distribution of $\mathbf{x}_{0}$, the objective in \eqref{eq:score-matching} can be approximated by Monte Carlo methods in practice, and the resulting loss function can be then optimized. 
After the score function is estimated, we introduce a continuous-time process that approximates \eqref{eq:zt}: 
\begin{equation}\label{eq:u}
d\mathbf{u}_{t}=\left[f(T-t)\mathbf{u}_{t}+(g(T-t))^{2}s_{\theta}(\mathbf{u}_{t},T-t)\right]dt
+g(T-t)d\bar{\mathbf{B}}_{t}, \quad \mathbf{u}_{0}\sim\hat{p}_{T},
\end{equation}
where we replace the true score function in \eqref{eq:zt} by the estimated score $s_{\theta}$.

\textbf{Discretization and algorithm.}
To obtain an implementable algorithm, one can apply different numerical methods for solving the reverse SDE \eqref{eq:u}, see Section~4 of \cite{SongICLR2021}. In this paper, 
we consider the following Euler-type discretization of the continuous-time stochastic process \eqref{eq:u}. 
Let $\eta>0$ be the stepsize and without loss of generality,
let us assume that $T=K\eta$, where $K$ is a positive integer. 
Let $\mathbf{y}_{0}\sim\hat{p}_{T}$ and
for any $k=1,2,\ldots,K$, we have
\begin{align} \label{eq:yk}
\mathbf{y}_{k}
&=\mathbf{y}_{k-1}
+\left(\int_{(k-1)\eta}^{k\eta}f(T-t)dt\right)\mathbf{y}_{k-1}
\\
& \qquad +\left(\int_{(k-1)\eta}^{k\eta}(g(T-t))^{2}dt\right)s_{\theta}(\mathbf{y}_{k-1},T-(k-1)\eta)
+\left(\int_{(k-1)\eta}^{k\eta}(g(T-t))^{2}dt\right)^{1/2}\xi_{k}, \nonumber
\end{align}
where $\xi_{k}$ are i.i.d. Gaussian random vectors $\mathcal{N}(0,I_{d})$.

We are interested in the convergence of the generated distribution $\mathcal{L}(\mathbf{y}_{K})$ to the data distribution $p_{0}$, where $\mathcal{L}(\mathbf{y}_{K})$ denotes the law or distribution of $\mathbf{y}_{K}$. Specifically, 
our goal is to bound the 2-Wasserstein distance $\mathcal{W}_{2}(\mathcal{L}(\mathbf{y}_{K}),p_{0})$, and investigate the number of iterates $K$ that is needed
in order to achieve $\epsilon$ accuracy, i.e. $\mathcal{W}_{2}(\mathcal{L}(\mathbf{y}_{K}),p_{0})\leq\epsilon$. At a high level, there are three sources of errors in analyzing the convergence: (1) the initialization of the algorithm at $\hat p_T$ instead of $p_T$, (2) the estimation error of the score function, and (3) the discretization error of the continuous-time process \eqref{eq:u}.

\section{Main Results}

In this section, we state our main results. 

\subsection{Assumptions}

We first state our assumptions and 
provide some discussions regarding these assumptions.


\begin{assumption}\label{assump:p0}
Assume that $p_{0}$ is differentiable and positive everywhere. 
Moreover, $-\log p_{0}$ is $m_{0}$-strongly convex and $L_{0}$-smooth for some $m_0, L_0>0$. 
\end{assumption}
\begin{remark}
Our strong-log-concave assumption on the (unbounded) data distribution has also been used in e.g. \citep{bruno2023diffusion, gao2024convergence}, which is a strong assumption. 
We need this assumption mainly because we consider Wasserstein convergence for score-based diffusions on a non-compact domain. 
In particular, we need the Wasserstein contractions in the reverse process to obtain convergence and control the discretization and score-matching errors, and this is achieved by 
deriving strong-log-concavity and smoothness for the score function at any time $t$ based on this assumption. 
In the literature, some studies \citep{chen2022sampling} considered compactly supported data (without log-concavity) and establish Wasserstain convergence guarantees by early stopping the algorithm and projecting the algorithm
output to the compact domain, whereas our analysis considers unbounded data distributions. 
It is an enormously interesting question how to relax Assumption~\ref{assump:p0}. See the conclusion section for further discussions.

\end{remark}


Our next assumption is about the true score function $\nabla_{\mathbf{x}}\log p_{t}(\mathbf{x})$ for $t \in [0, T]$.
We assume that $\nabla_{\mathbf{x}}\log p_{t}(\mathbf{x})$
is Lipschitz in time, where
the Lipschitz constant has at most
linear growth in $\Vert\mathbf{x}\Vert$. 
Assumption~\ref{assump:M:1} is needed in controlling the discretization error of the continuous-time process \eqref{eq:u}. 

\begin{assumption}\label{assump:M:1}
There exists some constant $M_{1}$ such that
\begin{equation}
\sup_{1\leq k\leq K}\sup_{(k-1)\eta\leq t\leq k\eta}\left\Vert\nabla\log p_{T-t}(\mathbf{x})-\nabla\log p_{T-(k-1)\eta}(\mathbf{x})\right\Vert
\leq M_{1}\eta(1+\Vert\mathbf{x}\Vert).
\end{equation}
\end{assumption}

To motivate Assumption~\ref{assump:M:1}, consider the idealized case where  $\mathbf{x}_{0}\sim\mathcal{N}(0,\sigma_{0}^{2}I_{d})$. Then, 
one can compute that
$\nabla_{\mathbf{x}}\log p_{T-t}(\mathbf{x})
=-\frac{\mathbf{x}}{(a_{1}(T-t))^{2}\sigma_{0}^{2}+a_{2}(T-t)}$,
where
$a_{1}(T-t):=e^{-\int_{0}^{T-t}f(s)ds}$
and $a_{2}(T-t):=\int_{0}^{T-t}e^{-2\int_{s}^{T-t}f(v)dv}(g(s))^{2}ds$.
This implies that Assumption~\ref{assump:M:1} is satisfied with
$M_{1}=\sup_{ t \ge 0 }\left|\frac{d}{dt}\frac{1}{(a_{1}(t))^{2}\sigma_{0}^{2}+a_{2}(t)}\right|
=\sup_{t \ge 0 }\frac{|2a_{1}(t)a'_{1}(t)\sigma_{0}^{2}+a'_{2}(t)|}{\left((a_{1}(t))^{2}\sigma_{0}^{2}+a_{2}(t)\right)^{2}}$, provided that $M_{1}\in(0,\infty)$. 
Indeed, for VE-SDE, 
$M_{1}=\sup_{t \ge 0}\frac{(g(t))^{2}}{(\sigma_{0}^{2}+\int_{0}^{t}(g(s))^{2}ds)^{2}}<\infty$,
provided that $(g(t))^{2}\leq c_{1}+c_{2}\left(\int_{0}^{t}(g(s))^{2}ds\right)^{2}$ uniformly in $t$ for some $c_{1},c_{2}>0$,
which is a very mild assumption
that is satisfied by all our examples in Section~\ref{sec:examples}.
For VP-SDE, i.e. $f(t)=\frac{1}{2}\beta(t)$ and $g(t)=\sqrt{\beta(t)}$,
then $a_{2}(t)=1-e^{-\int_{0}^{t}\beta(s)ds}$ and
$M_{1}=\sup_{t\geq 0}\frac{\beta(t)e^{-\int_{0}^{t}\beta(s)ds}|1-\sigma_{0}^{2}|}{\left(e^{-\int_{0}^{t}\beta(s)ds}(\sigma_{0}^{2}-1)+1\right)^{2}}<\infty$,
provided that $\sup_{t\geq 0}\beta(t)e^{-\int_{0}^{t}\beta(s)ds}<\infty$ which is a very mild assumption
that is satisfied by all our examples in Section~\ref{sec:examples}.

We also make the following assumption on the score-matching approximation. Recall $(\mathbf{y}_{k})$ are the iterates defined in \eqref{eq:yk}.
\begin{assumption}\label{assump:M}
Assume that there exists $M>0$ such that
\begin{equation} 
\sup_{k=1,\ldots,K}
\left\Vert\nabla\log p_{T-(k-1)\eta}\left(\mathbf{y}_{k-1}\right)-s_{\theta}\left(\mathbf{y}_{k-1},T-(k-1)\eta\right)\right\Vert_{L_{2}}\leq M.
\end{equation}
\end{assumption}

We make a remark here that the main results in our paper will still hold
if Assumption~\ref{assump:M}
is to be replaced by the following $L_{2}$-type assumption on the score function of the continuous-time forward process $(\mathbf{x}_{t})$ in \eqref{OU:SDE}:
\begin{equation}\label{weakened:assump:2}
\sup_{k=1,\ldots,K}
\left\Vert\nabla\log p_{k\eta}\left(\mathbf{x}_{k\eta}\right)-s_{\theta}\left(\mathbf{x}_{k\eta},k\eta\right)\right\Vert_{L_{2}}\leq M,
\end{equation}
 under the additional assumption that $s_{\theta}(\cdot,k\eta)$ is Lipschitz for every $k$
and the observation that  $ \nabla \log p_{k\eta}(\cdot)$ is Lipschitz under Assumption~\ref{assump:p0} (see Lemma~\ref{lem:smooth}). Assumption~\eqref{weakened:assump:2} is considered in e.g. \cite{chen2022sampling}.

Finally, we make the following assumption on the stepsize $\eta$ in the algorithm \eqref{eq:yk}. 
\begin{assumption}\label{assump:stepsize}
Assume that the stepsize $\eta$ is small such that it satisfies the conditions:
\begin{equation}\label{assump:stepsize:1}
\eta\leq
\min_{0\leq t\leq T}\left\{\frac{\frac{(g(t))^{2}}{\frac{1}{m_{0}}e^{-2\int_{0}^{t}f(s)ds}+\int_{0}^{t}e^{-2\int_{s}^{t}f(v)dv}(g(s))^{2}ds}-f(t)}
{(f(t))^{2}+(g(t))^{4}(L(t))^{2}+M_{1}(g(t))^{2}}\right\},
\end{equation}
and
\begin{equation}\label{assump:stepsize:2}
\eta\leq
\min_{0\leq t\leq T}
\left\{\frac{1}{\frac{(g(t))^{2}}{\frac{1}{m_{0}}e^{-2\int_{0}^{t}f(s)ds}+\int_{0}^{t}e^{-2\int_{s}^{t}f(v)dv}(g(s))^{2}ds}-f(t)}\right\},
\end{equation}
where for any $0\leq t\leq T$, 
\begin{align}
L(t):=\min\left(\left(\int_{0}^{t}e^{-2\int_{s}^{t}f(v)dv}(g(s))^{2}ds\right)^{-1},
\left(e^{\int_{0}^{t}f(s)ds}\right)^{2}L_{0}\right).\label{eq:Lt} 
\end{align}
\end{assumption}

Note that in Assumption~\ref{assump:stepsize}, 
$L(t)$ defined in \eqref{eq:Lt} can be interpreted
as the Lipschitz constant of $\nabla_{\mathbf{x}}\log p_{t}(\mathbf{x})$ (Lemma~\ref{lem:smooth}). Under Assumption~\ref{assump:stepsize}, the stepsize $\eta$
is sufficiently small so that the discretization and score-matching errors
will be controllable.
For VE-SDE where $f=0$, \eqref{assump:stepsize:1}-\eqref{assump:stepsize:2} can be simplified as
$\eta\leq\min_{0\leq t\leq T}\frac{1}{(\frac{1}{m_{0}}+\int_{0}^{t}(g(s))^{2}ds)((g(t))^{2}(L(t))^{2}+M_{1})}$ and $\eta\leq\min_{0\leq t\leq T}\frac{\frac{1}{m_{0}}+\int_{0}^{t}(g(s))^{2}ds}{(g(t))^{2}}$,
where $L(t)=\min\left(\left(\int_{0}^{t}(g(s))^{2}ds\right)^{-1},
L_{0}\right)$. On the other hand, For VP-SDEs where $f(t)=\frac{1}{2}\beta(t)$ and $g(t)=\sqrt{\beta(t)}$, 
 \eqref{assump:stepsize:1}-\eqref{assump:stepsize:2} can be simplified as
\begin{equation}
\eta\leq\min_{0\leq t\leq T}\left\{\frac{m_{0}-(1-m_{0})e^{-\int_{0}^{t}\beta(s)ds}}{(m_{0}+(1-m_{0})e^{-\int_{0}^{t}\beta(s)ds})(\frac{1}{2}\beta(t)+2\beta(t)(L(t))^{2}+2M_{1})}\right\},
\end{equation}
and 
\begin{equation}
\eta\leq\min_{0\leq t\leq T}\frac{2}{\beta(t)}\cdot\frac{m_{0}+(1-m_{0})e^{-\int_{0}^{t}\beta(s)ds}}{m_{0}-(1-m_{0})e^{-\int_{0}^{t}\beta(s)ds}}, 
\end{equation}
where $L(t)=\min\left(\left(1-e^{-\int_{0}^{t}\beta(s)ds}\right)^{-1},e^{\int_{0}^{t}\beta(s)ds}L_{0}\right)$.  To ensure that the right-hand sides of \eqref{assump:stepsize:1}-\eqref{assump:stepsize:2} are positive for general VP-SDEs, a sufficient condition is $m_{0}>1/2$.

\subsection{Main Result}

In this section, we state our main theoretical result, which provides a bound on the 2-Wasserstein distance $\mathcal{W}_{2}(\mathcal{L}(\mathbf{y}_{K}),p_{0})$. To facilitate the presentation, we introduce a few quantities with their interpretations given in Table~\ref{table:quantities}. 
For any $0\leq t\leq T$, we define:
\begin{align}
c(t) & :=\frac{m_{0}(g(t))^{2}}{e^{-2\int_{0}^{t}f(s)ds}+m_{0}\int_{0}^{t}e^{-2\int_{s}^{t}f(v)dv}(g(s))^{2}ds}, \label{c:t:defn} 
\\
\mu(T-t)&:=\frac{(g(T-t))^{2}}{\frac{1}{m_{0}}e^{-2\int_{0}^{T-t}f(s)ds}+\int_{0}^{T-t}e^{-2\int_{s}^{T-t}f(v)dv}(g(s))^{2}ds}-f(T-t)
\nonumber
\\
&\qquad\qquad\qquad\qquad\qquad
-\eta(f(T-t))^{2}-\eta(g(T-t))^{4}(L(T-t))^{2},\label{mu:definition}
\\
m(t)&:=\frac{2(g(t))^{2}}{\frac{1}{m_{0}}e^{-2\int_{0}^{t}f(s)ds}+\int_{0}^{t}e^{-2\int_{s}^{t}f(v)dv}(g(s))^{2}ds}-2f(t), \label{eq:mt} 
\end{align}
\begin{align}
c_{1}(T) &:=\sup_{0\leq t\leq T}e^{-\frac{1}{2}\int_{0}^{t}m(T-s)ds}e^{-\int_{0}^{T}f(s)ds}\Vert\mathbf{x}_{0}\Vert_{L_{2}},\label{c:1:defn}
\\
c_{2}(T) & :=\sup_{0\leq t\leq T}\left(e^{-2\int_{0}^{t}f(s)ds} \Vert\mathbf{x}_{0}\Vert_{L_{2}}^2 
+d\int_{0}^{t}e^{-2\int_{s}^{t}f(v)dv}(g(s))^{2}ds\right)^{1/2}, \label{c:2:defn}
\end{align}
and moreover for any $k=1,2,\ldots,K$,
\begin{align}
\gamma_{k,\eta}&:=1-\int_{(k-1)\eta}^{k\eta}\mu(T-t)dt+M_{1}\eta\int_{(k-1)\eta}^{k\eta}(g(T-t))^{2}dt,\label{gamma:k:defn}
\\
h_{k,\eta}&:=c_{1}(T)\int_{(k-1)\eta}^{k\eta}[f(T-s)+(g(T-s))^{2}L(T-s)]ds
\nonumber
\\
&\qquad
+c_{2}(T)\int_{T-k\eta}^{T-(k-1)\eta}f(s)ds
+\left(\int_{T-k\eta}^{T-(k-1)\eta}(g(s))^{2}ds\right)^{1/2}\sqrt{d}.\label{h:k:eta:main}
\end{align}

\begin{table}[htp]
\begin{center}
\begin{tabular}{|c | c | c |} 
 \hline
Quantities & Interpretations & Sources/References \\ [0.5ex] 
 \hline 
 $c(t)$ in \eqref{c:t:defn}  & Contraction rate of $\mathcal{W}_{2}(\mathcal{L}(\mathbf{z}_{T}),p_{0}) $ &  \eqref{eq:contraction1} \\ 
 \hline
$L(T-t) $ in \eqref{eq:Lt} & Lipschitz constant of $\nabla_{\mathbf{x}}\log p_{T-t}(\mathbf{x})$  & Lemma~\ref{lem:smooth} \\
 \hline
 \multirow{2}{*}{$\gamma_{k,\eta}$ in \eqref{gamma:k:defn}}  & Contraction rate of discretization  & \multirow{2}{*}{Proposition~\ref{prop:iterates}} \\
 & and score-matching errors in $\mathbf{y}_{k}$ & \\
 \hline
 $m(t)$ in \eqref{eq:mt} & Contraction rate of $\mathbb{E}\Vert\tilde{\mathbf{x}}_{T}-\mathbf{z}_{T}\Vert^{2} $  & \eqref{eq:int-mt} \\
 \hline
 $c_1(T)$ in \eqref{c:1:defn} & Bound for $ \sup_{0 \le t \le T}\Vert\mathbf{z}_{t}-\tilde{\mathbf{x}}_{t}\Vert_{L_{2}}$   & \eqref{eq:c1Tsource} \\  
 \hline
  $c_2(T)$ in \eqref{c:2:defn} & $ \sup_{0\leq t\leq T}\Vert\mathbf{x}_{t}\Vert_{L_{2}}$  &  \eqref{c:2:source} \\  
 \hline
  $h_{k,\eta}$ in \eqref{h:k:eta:main} & Bound for $ \sup_{(k-1)\eta\leq t\leq k\eta} \left\Vert\mathbf{z}_{t}-\mathbf{z}_{(k-1)\eta}\right\Vert_{L_{2}}$  &  Lemma~\ref{lem:second:term} \\  
 \hline
\end{tabular}
\caption{\label{table:quantities} 
Summary of quantities, their interpretations and the sources}
\end{center}
\end{table}

We are now ready to state our bound on the 2-Wasserstein distance $\mathcal{W}_{2}(\mathcal{L}(\mathbf{y}_{K}),p_{0})$.

\begin{theorem}\label{thm:discrete:2}
Suppose that Assumptions~\ref{assump:p0}, \ref{assump:M:1},  \ref{assump:M} and \ref{assump:stepsize} hold.
Then, we have
\begin{align}
\mathcal{W}_{2}(\mathcal{L}(\mathbf{y}_{K}),p_{0})
\leq e^{-\int_{0}^{K\eta}c(t)dt}  \Vert\mathbf{x}_{0}\Vert_{L_{2}}  
+\mathcal{E}_{1}(f,g, K, \eta, M_1)
+\mathcal{E}_{2}(f,g, K, \eta, M, M_1),\label{main:thm:upper:bound}
\end{align}
where
\begin{align}
\mathcal{E}_{1}(f,g, K, \eta, M_1)
&:=\sum_{k=1}^{K}\prod_{j=k+1}^{K}\gamma_{j,\eta}\cdot\Bigg(M_{1}\eta\left(1+ \Vert\mathbf{x}_{0}\Vert_{L_{2}} 
+c_{2}(T)\right)\int_{(k-1)\eta}^{k\eta}(g(T-t))^{2}dt
\nonumber
\\
&\quad 
+\sqrt{\eta}h_{k,\eta}\left(\int_{(k-1)\eta}^{k\eta}[f(T-t)+(g(T-t))^{2}L(T-t)]^{2}dt\right)^{1/2}\Bigg),
\end{align}
and
\begin{align}
\mathcal{E}_{2}(f,g, K, \eta, M, M_1)
&:=\sum_{k=1}^{K}
\prod_{j=k+1}^{K}\gamma_{j,\eta}\cdot M\int_{(k-1)\eta}^{k\eta}(g(T-t))^{2}dt,
\end{align}
where $c(t)$ is given in \eqref{c:t:defn}, $\gamma_{j,\eta}$ is defined in \eqref{gamma:k:defn}, $c_{2}(T)$ is defined in \eqref{c:2:defn} and $h_{k,\eta}$ is given in \eqref{h:k:eta:main}.  
\end{theorem}

We can interpret Theorem~\ref{thm:discrete:2} as follows.
The first term in \eqref{main:thm:upper:bound} is the \emph{initialization error}; it characterizes the convergence of the continuous-time
reverse SDE $(\mathbf{z}_{t})$ in \eqref{eq:zt} to the distribution $p_{0}$ without discretization
or score-matching errors. Specifically, it bounds the error $\mathcal{W}_{2}(\mathbf{z}_{T},p_{0})$, which is introduced due to the initialization of the reverse SDE $(\mathbf{z}_{t})$ at $\hat p_T$ instead of $p_T$ (see Proposition~\ref{thm:1} for details). 
The second term $\mathcal{E}_{1}(f,g,K,\eta,M_{1})$ 
and the third term $\mathcal{E}_{2}(f,g,K,\eta,M,M_{1})$ in \eqref{main:thm:upper:bound} quantify
the \emph{discretization error} and the \emph{score-matching error} respectively 
in running the algorithm $(\mathbf{y}_{k})$ in \eqref{eq:yk}. 
Note that Assumption~\ref{assump:stepsize} implies
that $\mu(T-t)$ in \eqref{mu:definition} is positive when $\eta$ is sufficiently small, which further suggests from \eqref{gamma:k:defn} that $\gamma_{j,\eta}\in (0, 1)$ for any $j=1,2,\ldots,K$.   
This quantity $\gamma_{j,\eta}$ that appears
in the definitions of $\mathcal{E}_{1}(f,g,K,\eta,M_{1})$ 
and $\mathcal{E}_{2}(f,g,K,\eta,M,M_{1})$ is important, because it plays the role of a contraction rate of the error $\left\Vert\mathbf{z}_{k\eta}- \mathbf{y}_{k} \right\Vert_{L_{2}}$ over iterations (see Proposition~\ref{prop:iterates} for details). Conceptually, it guarantees that as the number of iterations increases, 
the discretization and score-matching errors in the iterates $(\mathbf{y}_{k})$
do not propagate and grow in time, which helps
us control the overall discretization and score-matching errors.


\subsection{Examples}\label{sec:examples}

In this section, we consider several examples of the forward processes and discuss the implications of  Theorem~\ref{thm:discrete:2}. In particular, we consider a variety of choices for $f$ and $g$ in the SDE \eqref{OU:SDE}, and investigate the iteration complexity, i.e., the number of iterates $K$ that is needed
in order to achieve $\epsilon$ accuracy, i.e. $\mathcal{W}_{2}(\mathcal{L}(\mathbf{y}_{K}),p_{0})\leq\epsilon$. While the bound in Theorem~\ref{thm:discrete:2} is quite complex in general, and it can be made more explicit when we consider special $f$ and $g$. 
 We summarize the results in Table~\ref{table:summary}. 
The main idea behind the results in Table~\ref{table:summary}
 is to analyze the three terms in \eqref{main:thm:upper:bound} in Theorem~\ref{thm:discrete:2} 
 carefully for each example. We first choose $T=K\eta$ to be sufficiently large and fixed such that the first term in \eqref{main:thm:upper:bound}, that controls
 the initialization error, is $\mathcal{O}(\epsilon)$. 
 Given $T=K\eta$ being fixed, the second term $\mathcal{E}_{1}$ in \eqref{main:thm:upper:bound}, that controls
 the discretization error, can be upper bounded by a function of $T=K\eta$ and $\eta$, which is $\mathcal{O}(\epsilon)$, by choosing $\eta$ to be sufficiently small. This also determines $K$ since $T=K\eta$ is chosen and fixed from the previous step. Finally, given $T=K\eta$ and $\eta$ being fixed, 
 the third term $\mathcal{E}_{2}$ in \eqref{main:thm:upper:bound}, that controls
 the score-matching error, can be upper bounded by a function of $T=K\eta$, $\eta$ and $M$, which is $\mathcal{O}(\epsilon)$ by ``choosing" $M$ to be sufficiently small.
 For each example, with the specific choice of $f$ and $g$, one needs to spell out the explicit 
 dependence on $T=K\eta$, $\eta$ and $M$ from \eqref{main:thm:upper:bound}
 before we can carry out the above analysis to obtain the results in Table~\ref{table:summary}.
The detailed derivation of these results will be given in Appendix~\ref{appendix:examples}.

\begin{table}[t]
\begin{center}
 \resizebox{\textwidth}{!}{
\begin{tabular}{|c|c|c|c|c|c|}
\hline
$f$ & $g$ & $K$ & $M$ & $\eta$ & References \\
\hline
\hline
0   & $ae^{bt}$  & $\mathcal{O}\left(\frac{d\log(\frac{d}{\epsilon})}{\epsilon^{2}}\right)$ & $\mathcal{O}\left(\frac{\epsilon}{\log(\frac{1}{\epsilon})}\right)$ & $\mathcal{O}\left(\frac{\epsilon^{2}}{d}\right)$ &  \cite{SongICLR2021} \\
\hline
0   & $a$ & $\mathcal{O}\left(\frac{d^{3/2}\log(\frac{d}{\epsilon})}{\epsilon^{3}}\right)$ & $\mathcal{O}\left(\frac{\epsilon}{\sqrt{\log(\frac{d}{\epsilon})}}\right)$ & $\mathcal{O}\left(\frac{\epsilon^{2}}{d\log(\frac{d}{\epsilon})}\right)$ &  \cite{de2021diffusion} \\
\hline
0   & $\sqrt{2at}$  & $\mathcal{O}\left(\frac{d^{5/4}}{\epsilon^{5/2}}\right)$ & $\mathcal{O}\left(\epsilon^{3/2}\right)$ & $\mathcal{O}\left(\frac{\epsilon^{2}}{d}\right)$ & our paper \\
\hline
\multirow{2}{*}{0}   & \multirow{2}{*}{$(b+at)^{c}$}  & \multirow{2}{*}{$\mathcal{O}\left(\frac{d^{\frac{1}{2(2c+1)}+1}}{\epsilon^{\frac{1}{2c+1}+2}}\right)$} & \multirow{2}{*}{$\mathcal{O}\left(\epsilon^{1+\frac{2c}{2c+1}}\right)$} & \multirow{2}{*}{$\mathcal{O}\left(\frac{\epsilon^{2}}{d}\right)$} & our paper \\
		&	&  &  &  & \\
\hline
\hline
$\alpha$   & $\sigma$  & $\mathcal{O}\left(\frac{d\log(\frac{d}{\epsilon})}{\epsilon^{2}}\right)$ & $\mathcal{O}(\epsilon)$ & $\mathcal{O}\left(\frac{\epsilon^{2}}{d}\right)$ &  \cite{de2021diffusion} \\
\hline
$\frac{b+at}{2}$   & $\sqrt{b+at}$  & $\mathcal{O}\left(\frac{d\sqrt{\log(\frac{d}{\epsilon})}}{\epsilon^{2}}\right)$ & $\mathcal{O}(\epsilon)$ & $\mathcal{O}\left(\frac{\epsilon^{2}}{d}\right)$ & \cite{Ho2020} \\
\hline
$\frac{(b+at)^{\rho}}{2}$   & $(b+at)^{\frac{\rho}{2}}$  & $\mathcal{O}\left(\frac{d(\log(\frac{d}{\epsilon}))^{\frac{1}{\rho+1}}}{\epsilon^{2}}\right)$ & $\mathcal{O}(\epsilon)$ & $\mathcal{O}\left(\frac{\epsilon^{2}}{d}\right)$ & our paper \\
\hline
$\frac{ae^{bt}}{2}$   & $\sqrt{a}e^{\frac{bt}{2}}$  & $\mathcal{O}\left(\frac{d\log(\log(\frac{d}{\epsilon}))}{\epsilon^{2}}\right)$ & $\mathcal{O}(\epsilon)$ & $\mathcal{O}\left(\frac{\epsilon^{2}}{d}\right)$ & our paper \\
\hline
\end{tabular}
}
\caption{\label{table:summary} 
The iteration complexity of the algorithm \eqref{eq:yk} in terms
of $\epsilon$ and dimension $d$. Here $f,g$ correspond to the drift and diffusion terms in the forward SDE \eqref{OU:SDE}, and $a$, $b$, $c$, $\alpha$, $\sigma$, $\rho$ are positive constants. $K$ is the number of iterates, $M$
is the score-matching approximation error, and $\eta$ is the stepsize required to achieve accuracy level $\epsilon$ (i.e. $\mathcal{W}_{2}(\mathcal{L}(\mathbf{y}_{K}),p_{0})\leq\epsilon$).}
\end{center}
\end{table}

From Table~\ref{table:summary}, we have the following observations. By ignoring the logarithmic dependence on $\epsilon$, 
we can see that the VE-SDE
from the literature \cite{SongErmon2019} ($f(t)=0$, $g(t)=ae^{bt}$ with $a,b>0$), as well as all the VP-SDE examples in Table~\ref{table:summary}, 
achieve the complexity $\widetilde{\mathcal{O}}\left(d/\epsilon^{2}\right)$.
In particular, the VP-SDEs that we proposed, 
i.e. $f(t)=\frac{1}{2}(b+at)^{\rho}$, $g(t)=(b+at)^{\rho/2}$ and $f(t)=\frac{1}{2}ae^{bt}$, $g(t)=\sqrt{a}e^{bt/2}$
have marginal improvement in terms of the logarithmic dependence on $d$ and $\epsilon$
compared to the existing models in the literature. 
Among the VP-SDE models, $f(t)=\frac{1}{2}ae^{bt}$, $g(t)=\sqrt{a}e^{bt/2}$ has the best complexity performance.
Indeed, the complexity for $f(t)=\frac{1}{2}(b+at)^{\rho}$, $g(t)=(b+at)^{\rho/2}$ is getting
smaller as $\rho$ increases. 
However, the complexity in Table~\ref{table:summary} only highlights the dependence
on $d,\epsilon$ and ignores the dependence on other model parameters including $\rho$. 
Therefore, in practice, we expect that for the performance of $f(t)=\frac{1}{2}(b+at)^{\rho}$, $g(t)=(b+at)^{\rho/2}$
is getting better when $\rho$ is getting bigger, as long as it is below a certain threshold,
since letting $\rho\rightarrow\infty$ will make the complexity $K$ explode with its hidden dependence on $\rho$.
This suggests that in practice, the optimal choice among the examples
from Table~\ref{table:summary} could be $f(t)=\frac{1}{2}(b+at)^{\rho}$, $g(t)=(b+at)^{\rho/2}$ 
for a reasonably large $\rho$. Later, in the numerical experiments (Section~\ref{sec:numerical}),
we will see that this is indeed the case.

Another observation from Table~\ref{table:summary}
is that the complexity $K$ has a phase transition, i.e. a discontinuity,
when $f$ decreases from $\alpha$ to $0$ at $\alpha=0$ (by considering the examples $f\equiv\alpha$, $g\equiv\sigma$ and $f\equiv 0$, $g\equiv a$), 
in the sense that the complexity jumps from $\mathcal{O}\left(\frac{d\log(d/\epsilon)}{\epsilon^{2}}\right)$ 
to $\mathcal{O}\left(\frac{d^{3/2}\log(d/\epsilon)}{\epsilon^{3}}\right)$. 
The intuition is that the drift term $\alpha>0$ in the forward process
creates a mean-reverting effect
to make the starting point of the reverse process $\hat{p}_{T}$ closer to the idealized starting point $p_{T}$. Since we hide the dependence of complexity on $\alpha$, and only keep track the dependence on $\epsilon,d$, it creates a phase transition at $\alpha=0$.
A similar phase transition phenomenon occurs for the complexity $K$
when we consider the examples $f\equiv 0$, $g(t)=ae^{bt}$ and $f\equiv 0$, $g(t)\equiv a$, 
where the complexity jumps from $\mathcal{O}\left(\frac{d\log(d/\epsilon)}{\epsilon^{2}}\right)$ 
to $\mathcal{O}\left(\frac{d^{3/2}\log(d/\epsilon)}{\epsilon^{3}}\right)$ 
as $b$ decreases to $0$.


\begin{remark}
In the iteration complexities in Table~\ref{table:summary}, we omit the dependence on the constant $M_{1}$ from Assumption~\ref{assump:M:1}, 
where it is assumed to be a universal constant.
One can easily include the dependence on $M_{1}$ for all the examples in Table~\ref{table:summary}, due to the explicit error bound we obtain in 
Theorem~\ref{thm:discrete:2}. For example, when $f(t)=0$, $g(t)=ae^{bt}$, the iteration complexity $K$
becomes $\mathcal{O}\left(\log\left(\frac{d}{\epsilon}\right)\max\left\{\frac{d}{\epsilon^{2}},\frac{M_{1}\log(1/\epsilon)d^{3/4}}{\epsilon^{3/2}}\right\}\right)$, see Remark~\ref{rk:VE:5} in Appendix~\ref{appendix:examples}; when $f(t)=\frac{b+at}{2}$, $g(t)=\sqrt{b+at}$, the iteration complexity $K$
becomes $\mathcal{O}\left(\sqrt{\log\left(\frac{d}{\epsilon}\right)}\max\left\{\frac{d}{\epsilon^{2}},\frac{M_{1}\sqrt{d}}{\epsilon}\right\}\right)$, see Remark~\ref{rk:VP:6} in Appendix~\ref{appendix:examples}.
\end{remark}

\subsection{Discussions}\label{sec:discussion}

Our results in Table~\ref{table:summary} naturally lead to several questions:  
\begin{itemize}
\item[(1)] 
For the VP-SDE examples in Table~\ref{table:summary}, we have seen
that the iteration complexity is always of order $\widetilde{\mathcal{O}}\left(d/\epsilon^2\right)$ where $\widetilde{\mathcal{O}}$  ignores the logarithmic dependence on $\epsilon,d$. 
One natural question is, for VP-SDE, is the complexity always of order $\widetilde{\mathcal{O}}\left(d/\epsilon^2\right)$? 
\item[(2)] 
For our examples in Table~\ref{table:summary}, the best complexity is of order $\widetilde{\mathcal{O}}\left(d/\epsilon^2\right)$. Another natural question is that are there other choices of $f, g$ so that the complexity becomes better than $\widetilde{\mathcal{O}}(d/\epsilon^{2})$? 
\item[(3)] 
If the answer to question (2) is negative, then are these upper bounds tight? 
\end{itemize}
We have answers and results to these questions as follows.

First, we will show that the answer to question (1) is yes
for a very wide class of general VP-SDE models.
In particular, we show in in the next proposition that under mild assumptions on the function $\beta(t)$, 
the class of VP-SDEs (i.e., $f(t) = \frac{1}{2}\beta(t) $ and $g(t) = \sqrt{ \beta(t) }$ for some nondecreasing {function} $\beta(t)$) will always lead to the complexity $\widetilde{\mathcal{O}}(\frac{d}{\epsilon^{2}})$
where $\widetilde{\mathcal{O}}$ ignores the logarithmic factors. 

\begin{proposition}\label{prop:VP}
Under the assumptions of Theorem~\ref{thm:discrete:2}, 
assume that $\beta(t)$ is positive and increasing in $t$
and there exist some $c_{1},c_{2},c_{3}>0$ such that
$\beta(t)\leq c_{1}\left(\int_{0}^{t}\beta(s)ds\right)^{c_{3}}+c_{2}<\infty$ for every $t\geq 0$. 
Then, we have $\mathcal{W}_{2}(\mathcal{L}(\mathbf{y}_{K}),p_{0})\leq\mathcal{O}(\epsilon)$ 
after
$K=\mathcal{O}\left(\frac{d(\log(d/\epsilon))^{3 c_3+1}}{\epsilon^{2}}\right)$
iterations, provided that 
$M\leq\frac{\epsilon}{(\log(\sqrt{d}/\epsilon))^{c_3}}$
and $\eta\leq\frac{\epsilon^{2}}{d(\log(1/\epsilon))^{3/c_3}}$.
\end{proposition}

It is easy to check that the assumptions in Proposition~\ref{prop:VP} are satisfied 
for all VP-SDEs examples in Table~\ref{table:summary}. 
Note that Proposition~\ref{prop:VP} is a general result for a wide class of VP-SDEs, and 
the dependence of the iteration complexity on the logarithmic factors of $d$
and $\epsilon$ may be improved for various examples considered in Table~\ref{table:summary}. The details will be discussed and provided in Appendix~\ref{app:VPSDE}.


Next, we will show that the answer to question (2) is negative.
In particular, we will show in the following proposition that
if we use the upper bound \eqref{main:thm:upper:bound} in Theorem~\ref{thm:discrete:2}, 
then in order to achieve $\epsilon$ accuracy, i.e. $\mathcal{W}_{2}(\mathcal{L}(\mathbf{y}_{K}),p_{0})\leq\epsilon$, 
 the complexity 
$K=\widetilde{\Omega}\left(d/\epsilon^{2}\right)$ under mild assumptions, 
where $\widetilde{\Omega}$ ignores the logarithmic dependence on $\epsilon$ and $d$.

\begin{proposition}\label{prop:lower:bound}
Suppose the assumptions in Theorem~\ref{thm:discrete:2} hold
and we further assume that $\min_{t\geq 0}g(t)>0$
and $\min_{t\geq 0}(f(t)+(g(t))^{2}L(t))>0$
and $\max_{0\leq s\leq t}c(s)\leq c_{1}\left(\int_{0}^{t}c(s)ds\right)^{\rho}+c_{2}$
uniformly in $t$ for some $c_{1},c_{2},\rho>0$, where $c(s)$ is defined in \eqref{c:t:defn}.
We also assume that $\mu(t)\geq\frac{1}{4}m(t)$ for any $0\leq t\leq T$ (which holds for any sufficiently small $\eta$), where $\mu(t),m(t)$ are defined in \eqref{mu:definition} and \eqref{eq:mt}. 
If we use the upper bound \eqref{main:thm:upper:bound}, 
then in order to achieve $\epsilon$ accuracy, i.e. $\mathcal{W}_{2}(\mathcal{L}(\mathbf{y}_{K}),p_{0})\leq\epsilon$, 
we must have
$K=\widetilde{\Omega}\left(\frac{d}{\epsilon^{2}}\right)$, 
where $\widetilde{\Omega}$ ignores the logarithmic dependence on $\epsilon$ and $d$.
\end{proposition}

The assumptions in Proposition~\ref{prop:lower:bound} are mild and one can readily check that they are satisfied for all the examples in Table~\ref{table:summary}. 
If we ignore the dependence on the logarithmic factors
of $d$ and $\epsilon$, we can see from Table~\ref{table:summary}
that the VE-SDE example $f(t)\equiv 0$, $g(t)=ae^{bt}$, 
and all the VP-SDE examples achieve the lower bound
in Proposition~\ref{prop:lower:bound}.

In Proposition~\ref{prop:lower:bound}, we showed
that using the upper bound \eqref{main:thm:upper:bound} in Theorem~\ref{thm:discrete:2}, 
we have the lower bound on the complexity 
$K=\widetilde{\Omega}\left(d/\epsilon^{2}\right)$.
Therefore, the answer to question (2) is negative.
This leads to question (3), which is, whether
the iteration complexity $\tilde{\mathcal{O}}(d/\epsilon^{2})$ (see Proposition~\ref{prop:VP} and Table~\ref{table:summary}) obtained
from the upper bound in Theorem~\ref{thm:discrete:2} is tight or not.
This leads us to investigate a lower bound for the number of iterates that is needed to achieve $\epsilon$ accuracy.
In the following proposition, we will show that the lower bound for 
the iteration complexity of algorithm \eqref{eq:yk} is at least $\Omega\left(\sqrt{d}/\epsilon\right)$ by constructing
a special example when the initial distribution $p_{0}$ is Gaussian.

\begin{proposition}\label{prop:lower:bound:Gaussian}
Consider the special case when
$\mathbf{x}_{0}$ follows a Gaussian distribution 
$\mathbf{x}_{0}\sim\mathcal{N}(0,\sigma_{0}^{2}I_{d})$.
Then, in order to achieve the 2-Wasserstein $\epsilon$ accuracy, 
the iteration complexity has a lower bound $\Omega\left(\frac{\sqrt{d}}{\epsilon}\right)$,
i.e. if there exists some $T=T(\epsilon)$ and $\bar{\eta}=\bar{\eta}(\epsilon)$ 
such that
$\mathcal{W}_{2}(\mathcal{L}(\mathbf{y}_{K}),p_{0})\leq\epsilon$
for any $K\geq\bar{K}:=T/\bar{\eta}$ (with $\eta=T/K\leq\bar{\eta}$),
then we must have
$\bar{K}=\Omega\left(\frac{\sqrt{d}}{\epsilon}\right)$. 
\end{proposition}

The answer to question (3) is complicated.
On the one hand, the lower bound $\Omega\left(\sqrt{d}/\epsilon\right)$ in
Proposition~\ref{prop:lower:bound:Gaussian} does not match
the upper bound $\tilde{\mathcal{O}}(d/\epsilon^{2})$. 
Note that the complexity $\Omega\left(\sqrt{d}/\epsilon\right)$ in Proposition~\ref{prop:lower:bound:Gaussian}
matches the upper bound for the complexity of an unadjusted Langevin algorithm 
under an additional assumption which is a growth condition on the third-order
derivative of the log-density of the target distribution, see \cite{Li2022}.
However, under our current assumptions, 
it is shown in \cite{DK2017} that the upper bound for the complexity of an unadjusted Langevin algorithm
matches the upper bound $\tilde{\mathcal{O}}(d/\epsilon^{2})$ in Table~\ref{table:summary}
that is deduced from Theorem~\ref{thm:discrete:2}. 
Hence, we speculate that the upper bound we obtained in Theorem~\ref{thm:discrete:2}
is tight under our current assumptions, and it may not be improvable unless
additional assumptions are imposed.
It will be left as a future research direction to explore
whether under additional assumptions on the data distribution $p_{0}$, 
one can improve the upper bound in Theorem~\ref{thm:discrete:2}
and hence improve the complexity to match the lower bound $\Omega\left(\sqrt{d}/\epsilon\right)$ in Proposition~\ref{prop:lower:bound:Gaussian}, and furthermore, whether under the current assumptions, 
there exists an example other than the Gaussian distribution as illustrated in Proposition~\ref{prop:lower:bound:Gaussian}
that can match the upper bound $\tilde{\mathcal{O}}(d/\epsilon^{2})$.



\section{Numerical Experiments}\label{sec:numerical}
In this section, we conduct numerical experiments based on various forward SDEs for unconditional image generation on the CIFAR-10 image dataset. Due to limitations in computational resources, 
the purpose of our experiments is not to beat or match the state-of-the-art numerical results such as FID scores. Instead, our goal is to compare the performances of diffusion models with different forward processes and better understand the impacts of such model choices numerically in addition to our theoretical findings. 

%

\subsection{SDEs for the Forward Process}\label{section:forming:SDE}

We consider Variance Exploding (VE) SDEs and Variance Preserving (VP) SDEs as the forward processes in our experiments.  

First, we consider various VE-SDEs which takes the form $d\mathbf{x}_{t} = \sqrt{\frac{d[\sigma^2(t)]}{dt}}d\mathbf{B}_{t} := g(t)d\mathbf{B}_{t},$
where $\sigma^2(t)$ is some non-decreasing function representing the scale of noise slowly added to the data over time $t \in [0,1]$. 
The transition kernel of VE-SDE is given by
    $p_{t|0}(\mathbf{x}_{t} | \mathbf{x}_{0}) = \mathcal{N}\left(\mathbf{x}_{t};\mathbf{x}_{0},\left[\sigma^2(t)-\sigma^2(0)\right]I_d\right)$.
We consider several different noise functions $\sigma(t)$ (equivalently $g(t)$) below, and in the experiments we maintain the choice of $\sigma_{\min}:= \sigma(0)$ and $\sigma_{\max}= \sigma(1)$ as in \cite{SongICLR2021}, where $\sigma_{\min} \ll \sigma_{\max}.$

\begin{itemize}
    \item [(1)] $g(t)=ab^t$ for some $a, b >0.$ (see \cite{SongErmon2019, SongICLR2021}). In this case, we have $\sigma(t)=\sigma_{\min}\left(\frac{\sigma_{\max}}{\sigma_{\min}}\right)^t$. 
    
    \item [(2)] $g(t)=\text{const}$, which yields $(\sigma(t))^2=\sigma^2_{\min}+(\sigma^2_{\max}-\sigma^2_{\min})t$.
    
    \item [(3)] $g(t)=\sqrt{2at}$ for some constant $a>0$, where $(\sigma(t))^2=\sigma^2_{\min}+(\sigma^2_{\max}-\sigma^2_{\min})t^2$.

    \item [(4)] $g(t) = (b+at)^{\rho-\frac{1}{2}}$ for some $a, b>0$, where
    $\sigma(t)=\left({\sigma_{\min}}^{\frac{1}{\rho}}+\left({\sigma_{\max}}^{\frac{1}{\rho}}-{\sigma_{\min}}^{\frac{1}{\rho}}\right)t\right)^{\rho}$. 
This noise schedule function $\sigma(t)$ is motivated by \cite{Karras2022}, where they consider non-uniform discretization and the discretization/time steps are defined according to a sequence of noise levels based on $\sigma(t)$. 
\end{itemize}
Next, we consider various VP-SDEs in the numerical experiments, where 
VP-SDE can be written as (see e.g. \cite{SongICLR2021})
   $ d\mathbf{x}_t = -\frac{1}{2}\beta(t)\mathbf{x}_tdt + \sqrt{\beta(t)}d\mathbf{B}_t$, 
where $\beta(t)$ represents the noise scales over time $t \in [0, 1]$. 
We consider the following choices of $\beta(t)$ in our experiments with $\beta(0) = \beta_{\min} \ll \beta(1) = \beta_{\max}$ (except the constant $\beta(t)$ case). 


\begin{itemize}
    \item [(1)] $\beta(t)=b+at=\beta_{\min}+(\beta_{\max}-\beta_{\min})t$ for $t\in(0,1]$ (see \cite{Ho2020}). 
    
    \item [(2)] $\beta(t)=\beta_{\text{const}}$, which is a constant. 
    
    \item [(3)] $\beta(t)=(b+at)^{\rho}=\left({\beta_{\min}}^{\frac{1}{\rho}} + \left({\beta_{\max}}^{\frac{1}{\rho}}-{\beta_{\min}}^{\frac{1}{\rho}}\right)t\right)^{\rho}$ for $t\in(0,1]$. 
    
    \item [(4)] $\beta(t)=ab^t=\beta_{\min}\cdot\left(\frac{\beta_{\max}}{\beta_{\min}}\right)^t$ for $t\in(0,1]$.
    
\end{itemize}


\subsection{Experiment Setup}\label{sec:experiment-setup}
In this section we discuss the setup of the experiment.

\noindent\textbf{Setup}\tab We focus on image generation with the ``DDPM++ cont. (VP)” and “NCSN++ cont. (VE)" architectures from \cite{SongICLR2021}, whose generation processes correspond to the discretizations of the reverse-time VP-SDE and VE-SDE respectively. The models are trained on the popular $32\times32$ image dataset CIFAR-10, and we use the code base and structures in \cite{SongICLR2021}.
However, we only have access to NVIDIA GeForce GTX 1080Ti, which has less available memory than the models in \cite{SongICLR2021} requires; thus we reduce the number of channels in the residual blocks of DDPM++ and NCSN++ from 128 to 32, which effectively downsizes the filters dimension of the convolutional layers of the blocks by four times (see Appendix~H in \cite{SongICLR2021} for details of the neural network architecture). This will likely reduce the neural network's capability to capture more intricate details in the original data. However, since the purpose of our experiments is not to beat the state-of-the-art results, we focus on the performance comparison of different forward SDE models based on (non-deep) neural networks in \cite{SongICLR2021} for score estimations. All models are trained for 3 million iterations (compared to 1.3 million iterations in \cite{SongICLR2021}), since our models converge slower due to limitations of computing resources.



\noindent\textbf{Relevant hyperparameters}\tab  
We choose the following configuration for the forward SDEs in the experiments. For VE-SDEs described in Section~\ref{section:forming:SDE}, we use
$\sigma_{\min}=0.01$ and $\sigma_{\max}=50$ \cite{song2020improved}. For VP-SDEs, we choose $\beta_{\text{const}}=0.005$ for constant $\beta(t)$; $\beta_{\min}=10^{-4}$, $\beta_{\max}=0.03$ for $\beta(t)=ab^t$, and $\beta_{\min}=10^{-4}$, $\beta_{\max}=0.02$ (\cite{Ho2020}) for the other VP-SDEs to maintain the progression of $\alpha_i:=\prod_{j=1}^{i}(1-\beta_i)$, where $\{\beta_i\}_{i=1}^{N}$ is the discretization of $\beta(t)$. Note that $\alpha_i$ slowly progresses from $1-\beta_{\min}$ to 0 for all our VP-SDEs, and one can equivalently use $\alpha_i$'s to demonstrate the noise schedules instead of $\beta_i$ (similar to \cite{Nichol2021}).
Figure~\ref{schedules} shows the difference of $\alpha_i$ for DDPM models (corresponding to discretization of VP-SDEs) and $\sigma_i$ (discretized noise level $\sigma(t)$) for NCSN (a.k.a. SMLD) models (corresponding to discretization of VE-SDEs). 


\begin{figure}
    \begin{subfigure}{.5\textwidth}
        \centering
        \includegraphics[width=1\linewidth]{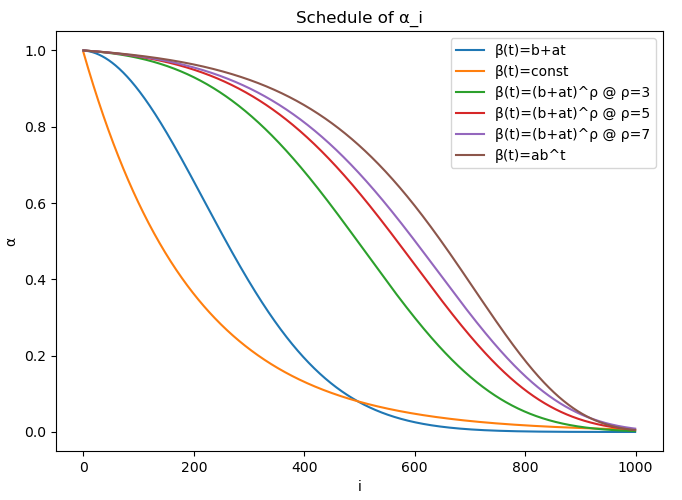}
        \caption{DDPM}
    \end{subfigure}
    \begin{subfigure}{.5\textwidth}
        \centering
        \includegraphics[width=1\linewidth]{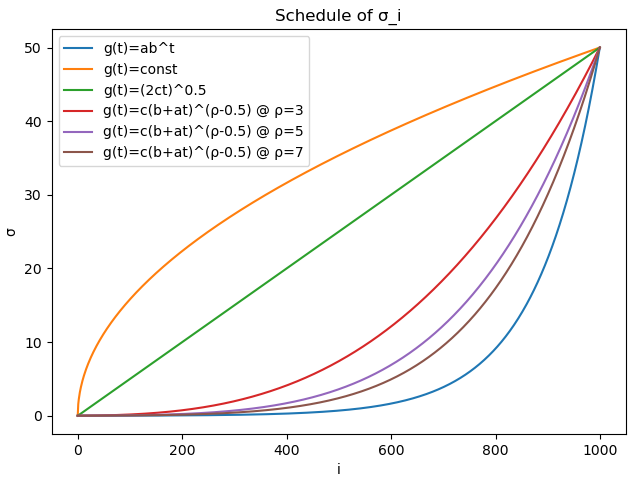}
        \caption{NCSN}
    \end{subfigure}
    \caption{The schedules of $(\alpha_i)$ for DDPM models and $(\sigma_i)$ for NCSN (a.k.a. SMLD) models with different forward SDEs.}
    \label{schedules}
\end{figure}

\noindent\textbf{Sampling}\tab 
We generate samples using the Euler-Maruyama solver for discretizing the reverse-time SDEs, referred to as the predictor in \cite{SongICLR2021}.  We set the number of discretized time steps $N=1000$, which follows \cite{Ho2020} and \cite{SongICLR2021}. 

\noindent\textbf{Performance Metrics}\tab The quality of the generated image samples is evaluated with Fr\'{e}chet Inception Distance (FID, lower is better), which was first introduced by \cite{heusel2017gans} for measuring the $2$-Wasserstein distance between the distribution of generated images and the distribution of real images. We also report the Inception Score (IS) (see \cite{salimans2016improved}) of the generated images as a secondary measure. However, IS (higher is better) only evaluates how realistic the generated images are without a comparison to real images. Each FID and IS measure is evaluated based on $20,000$ samples.

\subsection{Empirical Results}

Table~\ref{table:scores} and Figures~\ref{figure:FID}--\ref{figure:IS} show the performances of various diffusion models corresponding to different forward SDEs that we used. 
We have the following two important observations.  

First, the experimental results are in good agreement with our theoretical prediction on the iteration complexity in Table~\ref{table:summary}. With the same number of discretized time steps, models (forward SDEs) with lower order of iteration complexity generally obtain a better FID score and IS (lower FID and higher IS) over training iterations. In addition, as predicted by the theory in Table~\ref{table:summary}, VE-SDE models generally perform worse than VP-SDEs. Among the VE-SDE models, the choice of $f\equiv 0$ and $g(t)=ab^{t}$ 
leads to the best performance in terms of FID and IS scores. We also remark that VE-SDE models can perform significantly better with a corrector (see \cite{SongICLR2021} for Predictor-Corrector sampling), and get close to the performance of VP-SDE models. However, our experimental results are based on the stochastic sampler without any corrector in order to fit the setup of our theory.

Second, our experimental results show that our proposed VP-SDE with a polynomial variance schedule $\beta(t)=(b+at)^{\rho}$ for some $\rho$ or an exponential variance schedule $\beta(t)=ab^{t}$ can outperform the other existing models, at least with simpler neural network architectures. The optimal 
$\rho$ is around $5$ according to Table~\ref{table:scores}. This is again consistent with our discussion in Section~\ref{sec:examples}.

We also choose the best performing models from Table~\ref{table:scores} and test the deeper neural network architecture in \cite{SongICLR2021} (which doubles the number of residual blocks per resolution, except for reduced batch size and reduced number of filters due to our memory limitation, as mentioned in Section~\ref{sec:experiment-setup}). The results are shown in Figure~\ref{deep:FID:IS} and Table~\ref{table:deep:scores}. 
We can see that the performance of different models remains consistent in a more complex architecture setting.

\begin{table}[ht]
    \begin{center}
        \begin{tabular}{ | l | c | c || c |}
            \hline
            \textbf{Model} & \textbf{FID$\downarrow$} & \textbf{IS$\uparrow$} & References \\ [0.5ex]
            \hline\hline
            DDPM (VP - $\beta(t) = \text{const}$) & 17.46 & 8.19 & \cite{de2021diffusion} \\\hline
            DDPM (VP - $\beta(t)=b+at$) & 11.26 & 8.21 & \cite{Ho2020} \\\hline
            DDPM (VP - $\beta(t)=(b+at)^{\rho},\rho=2$) & 9.77 & 8.33 &  \\\cline{1-3}
            DDPM (VP - $\beta(t)=(b+at)^{\rho},\rho=3$) & 9.67 & 8.32 &  \\\cline{1-3}
            DDPM (VP - $\beta(t)=(b+at)^{\rho},\rho=5$) & 9.64 & 8.41 & our paper \\\cline{1-3}
            DDPM (VP - $\beta(t)=(b+at)^{\rho},\rho=7$) & 10.22 & 8.41 &  \\\cline{1-3}
            DDPM (VP - $\beta(t)=(b+at)^{\rho},\rho=10$) & 10.27 & 8.51 &  \\\hline
            DDPM (VP - $\beta(t)=ab^t$) & 9.98 & 8.39 & our paper \\\hline\hline
            NCSN (VE - $g(t)=ab^t$) & 22.11 & 8.18 & \cite{SongICLR2021} \\\hline
            NCSN (VE - $g(t)=\text{const}$) & 461.42 & 1.18 & \cite{de2021diffusion} \\\hline
            NCSN (VE - $g(t)=\sqrt{2at}$) & 457.04 & 1.20 &  our paper \\\hline
            NCSN (VE - $g(t)=(b+at)^{\rho-\frac{1}{2}},\rho=2$) & 369.51 & 1.34 &  \\\cline{1-3}
            NCSN (VE - $g(t)=(b+at)^{\rho-\frac{1}{2}},\rho=3$) & 233.20 & 1.95 &   \\\cline{1-3}
            NCSN (VE - $g(t)=(b+at)^{\rho-\frac{1}{2}},\rho=5$) & 137.55 & 4.01 &  our paper \\\cline{1-3}
            NCSN (VE - $g(t)=(b+at)^{\rho-\frac{1}{2}},\rho=7$) & 159.66 & 3.11 &  \\\cline{1-3}
            NCSN (VE - $g(t)=(b+at)^{\rho-\frac{1}{2}},\rho=10$) & 99.89 & 4.91 &  \\\hline\hline
        \end{tabular}
        \caption{Performances of different SDE models on CIFAR-10 at 3,000,000 iterations}\label{table:scores}
    \end{center}
\end{table}
\begin{figure}[ht]
    \begin{subfigure}{.5\textwidth}
        \centering
        \includegraphics[width=1\linewidth]{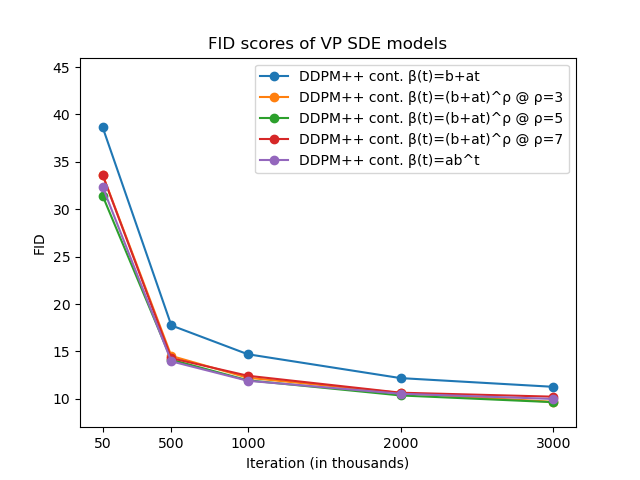}
        \caption{DDPM}
    \end{subfigure}
    \begin{subfigure}{.5\textwidth}
        \centering
        \includegraphics[width=1\linewidth]{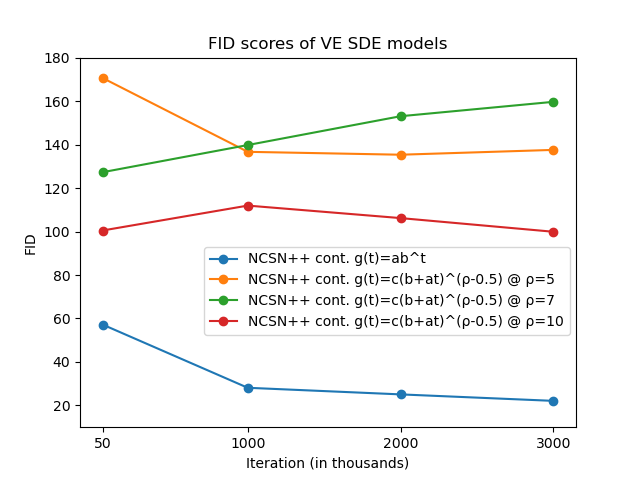}
        \caption{NCSN (aka SMLD)}
    \end{subfigure}
    \caption{The FID score progressions of different SDE models on CIFAR-10}
    \label{figure:FID}
\end{figure}
\begin{figure}[ht]
    \begin{subfigure}{.5\textwidth}
        \centering
        \includegraphics[width=1\linewidth]{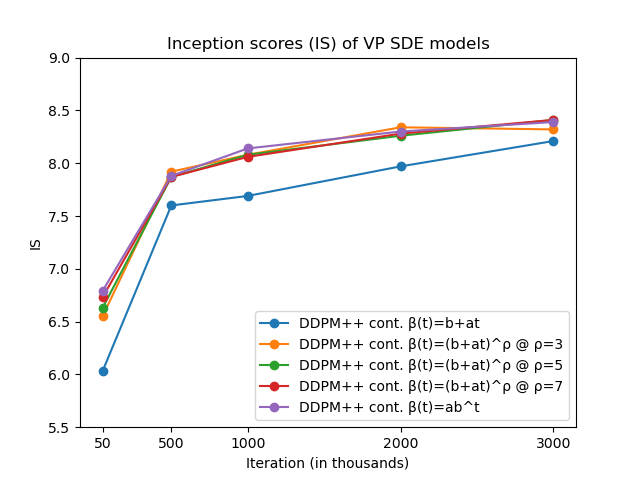}
        \caption{DDPM}
    \end{subfigure}
    \begin{subfigure}{.5\textwidth}
        \centering
        \includegraphics[width=1\linewidth]{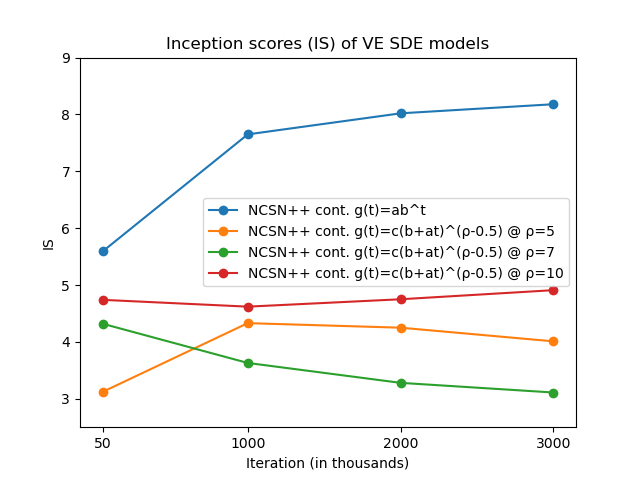}
        \caption{NCSN (aka SMLD)}
    \end{subfigure}
    \caption{The inception score (IS) progressions of different SDE models on CIFAR-10}
    \label{figure:IS}
\end{figure}

\begin{table}[ht]
    \begin{center}
        \begin{tabular}{| l | c | c |}
            \hline
            \textbf{Model} & \textbf{FID$\downarrow$} & \textbf{IS$\uparrow$} \\ [0.5ex]
            \hline\hline
            DDPM deep (VP - $\beta(t)=b+at$) & 9.22 & 8.25 \\\hline
            DDPM deep (VP - $\beta(t)=(b+at)^{\rho},\rho=5$) & 8.20 & 8.55 \\\hline
            DDPM deep (VP - $\beta(t)=ab^t$) & 8.14 & 8.44 \\\hline
            NCSN deep (VE - $g(t)=ab^t$) & 20.00 & 8.41 \\\hline
        \end{tabular}
        \caption{Performances of deep versions of the best performing models from Table~\ref{table:scores} on CIFAR-10 at 3,000,000 iterations. SMLD is also known as NCSN.}
        \label{table:deep:scores}
    \end{center}
\end{table}

\begin{figure}[ht]
    \begin{subfigure}{.5\textwidth}
        \centering
        \includegraphics[width=1\linewidth]{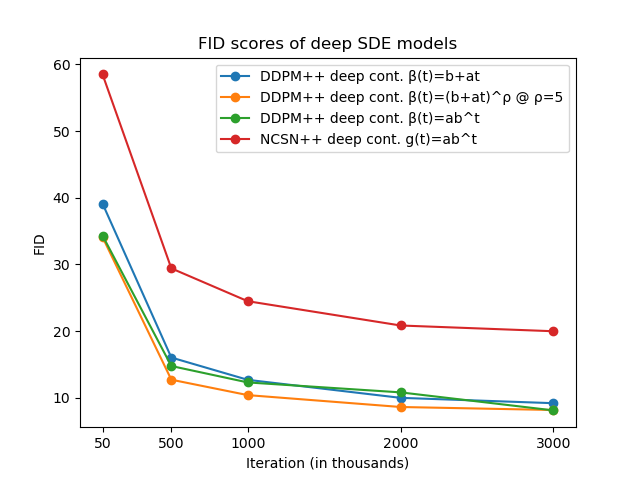}
        \caption{FID}
    \end{subfigure}
    \begin{subfigure}{.5\textwidth}
        \centering
        \includegraphics[width=1\linewidth]{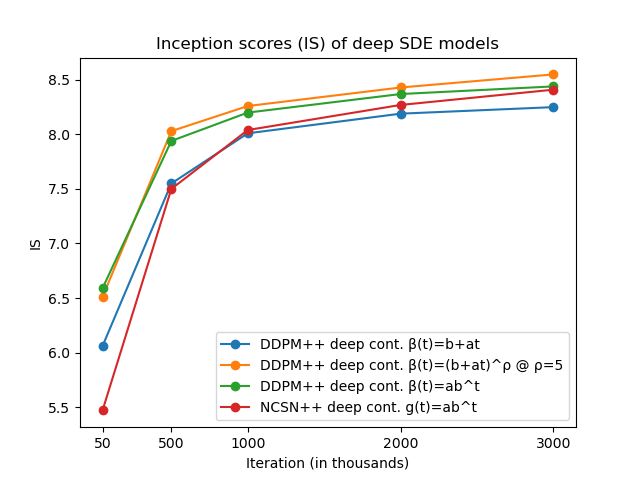}
        \caption{IS}
    \end{subfigure}
    \caption{The FID and IS score progressions of the deep version of the best performing SDE models on CIFAR-10}
    \label{deep:FID:IS}
\end{figure}



\section{Analysis: Proofs of the Main Results}

\subsection{Proof of Theorem~\ref{thm:discrete:2}}

To prove Theorem~\ref{thm:discrete:2}, we study the three sources of errors discussed in Section~\ref{sec:prelim}: (1) the initialization of the algorithm at $\hat p_T$ instead of $p_T$, (2) the estimation error of the score function, and (3) the discretization error of the continuous-time process \eqref{eq:u}.

First, we study the error introduced due to the initialization at $\hat p_T$ instead of $p_T$. 
Recall the reverse SDE $\mathbf{z}_{t}$ given in \eqref{eq:zt}.
 As discussed in Section~\ref{sec:prelim},  the distribution
of $\mathbf{z}_{T}$ differs from $p_{0}$, because $\mathbf{z}_{0}\sim\hat{p}_{T} \ne p_T$. 
The following result provides a bound on $\mathcal{W}_{2}(\mathcal{L}(\mathbf{z}_{T}),p_{0})$. 


\begin{proposition}\label{thm:1}
Assume that $p_{0}$ is $m_{0}$-strongly-log-concave.  
Then, we have
\begin{align}\label{eq:contraction1}
\mathcal{W}_{2}(\mathcal{L}(\mathbf{z}_{T}),p_{0})
\leq e^{-\int_{0}^{T}c(t)dt}\Vert\mathbf{x}_{0}\Vert_{L_{2}},
\end{align}
where $c(t)$ is given in \eqref{c:t:defn}. 
\end{proposition}

The main challenge in analyzing the SDE $\mathbf{z}_t$ lies in studying the term $\nabla\log p_{T-t}(\mathbf{z}_{t})$. In general, this term is neither linear in $\mathbf{z}_t$
nor admits a closed-form expression. However, 
when $p_{0}$ is strongly log-concave, we are able to show
that $\log p_{T-t}(\mathbf{x})$ is also strongly concave. This fact, together with It\^{o}'s formula for SDEs,  allows us to establish
Proposition~\ref{thm:1}.  The proof of Proposition~\ref{thm:1} is given in Section~\ref{sec:thm1}.


Now we consider the algorithm \eqref{eq:yk} with iterates $(\mathbf{y}_{k})$, and bound the errors due to score estimations and discretizations together. 
For any $k=0,1,2,\ldots,K$, $\mathbf{y}_{k}$
has the same distribution as $\hat{\mathbf{y}}_{k\eta}$, 
where $\hat{\mathbf{y}}_{t}$ is a continuous-time process
with the dynamics:
\begin{equation}\label{continuous:hat:y}
d\hat{\mathbf{y}}_{t}=\left[f(T-t)\hat{\mathbf{y}}_{\lfloor t/\eta\rfloor\eta}+(g(T-t))^{2}s_{\theta}\left(\hat{\mathbf{y}}_{\lfloor t/\eta\rfloor\eta},T-\lfloor t/\eta\rfloor\eta\right)\right]dt
+g(T-t)d\bar{\mathbf{B}}_{t},
\end{equation}
with $\hat{\mathbf{y}}_{0}\sim\hat{p}_{T}$.
We have the following result 
that provides an upper bound 
for $\Vert\mathbf{z}_{k\eta}-\hat{\mathbf{y}}_{k\eta} \Vert_{L_{2}}$ in terms of $\Vert\mathbf{z}_{(k-1)\eta}-\hat{\mathbf{y}}_{(k-1)\eta} \Vert_{L_{2}}$. This result plays a key role in the proof of Theorem~\ref{thm:discrete:2}.

\begin{proposition}\label{prop:iterates}
Assume that $p_{0}$ is $m_{0}$-strongly-log-concave, 
i.e. $-\log p_{0}$ is $m_{0}$-strongly convex
and $\nabla\log p_{0}$ is $L_{0}$-Lipschitz.
For any $k=1,2,\ldots,K$,
\begin{align}
&\left\Vert\mathbf{z}_{k\eta}-\hat{\mathbf{y}}_{k\eta}\right\Vert_{L_{2}}
\leq\gamma_{k,\eta}\Vert\mathbf{z}_{(k-1)\eta}-\hat{\mathbf{y}}_{(k-1)\eta}\Vert_{L_{2}}
\nonumber
\\
& \qquad \qquad  \qquad 
+M_{1}\eta\left(1+\Vert\mathbf{x}_{0}\Vert_{L_{2}}
+c_{2}(T)\right)\int_{(k-1)\eta}^{k\eta}(g(T-t))^{2}dt
+M\int_{(k-1)\eta}^{k\eta}(g(T-t))^{2}dt
\nonumber
\\
&\qquad\qquad  \qquad \quad 
+\sqrt{\eta}h_{k,\eta}\left(\int_{(k-1)\eta}^{k\eta}[f(T-t)+(g(T-t))^{2}L(T-t)]^{2}dt\right)^{1/2},\label{L:2:iterates}
\end{align} 
where $\gamma_{k,\eta}$ is defined in \eqref{gamma:k:defn}, $c_{2}(T)$ is defined in \eqref{c:2:defn}
and $h_{k,\eta}$ is given in \eqref{h:k:eta:main}.
\end{proposition}

We remark that the coefficient $\gamma_{j,\eta}$ in front of the term $\left\Vert\mathbf{z}_{(k-1)\eta}-\hat{\mathbf{y}}_{(k-1)\eta}\right\Vert_{L_{2}}$ in  \eqref{L:2:iterates} lies in between zero and one. Indeed, it follows from Assumption~\ref{assump:stepsize} 
and the definition of $\mu(t)$ in \eqref{mu:definition} that
$\mu(t)\geq M_{1}\eta(g(t))^{2}$ for every $0\leq t\leq T$
and $\eta\max_{0\leq t\leq T}\mu(t)<1$ such that
for any $j=1,2,\ldots,K$, 
$0\leq\gamma_{j,\eta}\leq 1$ where $\gamma_{j,\eta}$ is defined in \eqref{gamma:k:defn}.



Now we are ready to prove Theorem~\ref{thm:discrete:2}.

\subsubsection{Completing the Proof of Theorem~\ref{thm:discrete:2}}

\begin{proof}
Since $\hat{\mathbf{y}}_{k\eta}$ has the same distribution as $\mathbf{y}_{k}$, 
by applying \eqref{L:2:iterates}, 
we have
\begin{align*}
&\mathcal{W}_{2}(\mathcal{L}(\mathbf{z}_{K\eta}),\mathcal{L}(\mathbf{y}_{K}))
\leq
\left\Vert\mathbf{z}_{K\eta}-\hat{\mathbf{y}}_{K\eta}\right\Vert_{L_{2}}
\nonumber
\\
&\leq
\sum_{k=1}^{K}
\prod_{j=k+1}^{K}\gamma_{j,\eta}
\cdot\Bigg(M_{1}\eta\left(1+\Vert\mathbf{x}_{0}\Vert_{L_{2}}
+c_{2}(T)\right)\int_{(k-1)\eta}^{k\eta}(g(T-t))^{2}dt
\nonumber
\\
&\qquad\qquad\qquad\qquad\qquad\qquad\qquad\qquad
+M\int_{(k-1)\eta}^{k\eta}(g(T-t))^{2}dt
\nonumber
\\
&\qquad\qquad
+\sqrt{\eta}h_{k,\eta}\left(\int_{(k-1)\eta}^{k\eta}\left[f(T-t)+(g(T-t))^{2}L(T-t)\right]^{2}dt\right)^{1/2}\Bigg).
\end{align*}
Moreover, we recall that $T=K\eta$ and by triangle inequality
for $2$-Wasserstein distance,
$\mathcal{W}_{2}(\mathcal{L}(\mathbf{y}_{K}),p_{0})
\leq
\mathcal{W}_{2}(\mathcal{L}(\mathbf{y}_{K}),\mathcal{L}(\mathbf{z}_{K\eta}))
+\mathcal{W}_{2}(\mathcal{L}(\mathbf{z}_{K\eta}),p_{0})$.
The proof is completed by applying
Proposition~\ref{thm:1}.
\end{proof}


\subsubsection{Proof of Proposition~\ref{thm:1}}\label{sec:thm1}
\begin{proof}
For the forward SDE \eqref{OU:SDE}, the transition density is Gaussian, and we have
$p_{t}(\mathbf{x}_{t})=\int_{\mathbb{R}^{d}}p(\mathbf{x}_{t}|\mathbf{x}_{0})p_{0}(\mathbf{x}_{0})d\mathbf{x}_{0}$,
where
\begin{equation*}
p(\mathbf{x}_{t}|\mathbf{x}_{0})
=\frac{1}{\left(2\pi\int_{0}^{t}e^{-2\int_{s}^{t}f(v)dv}(g(s))^{2}ds\right)^{d/2}}\exp\left(-\frac{\Vert\mathbf{x}_{t}-e^{-\int_{0}^{t}f(s)ds}\mathbf{x}_{0}\Vert^{2}}{2\int_{0}^{t}e^{-2\int_{s}^{t}f(v)dv}(g(s))^{2}ds}\right).
\end{equation*}
This implies that
\begin{align*}
\log p_{T-t}(\mathbf{x})
&=\log\int_{\mathbb{R}^{d}}\exp\left(-\frac{\Vert\mathbf{x}-e^{-\int_{0}^{T-t}f(s)ds}\mathbf{x}_{0}\Vert^{2}}{2\int_{0}^{T-t}e^{-2\int_{s}^{T-t}f(v)dv}(g(s))^{2}ds}\right)
p_{0}(\mathbf{x}_{0})d\mathbf{x}_{0}
\nonumber
\\
&\qquad\qquad\qquad\qquad\qquad\qquad
-\frac{d}{2}\log\left(2\pi\int_{0}^{t}e^{-2\int_{s}^{t}f(v)dv}(g(s))^{2}ds\right)
\nonumber
\\
&=\log\int_{\mathbb{R}^{d}}\exp\left(-\frac{\Vert\mathbf{x}-\mathbf{x}_{0}\Vert^{2}}{2\int_{0}^{T-t}e^{-2\int_{s}^{T-t}f(v)dv}(g(s))^{2}ds}\right)
p_{0}\left(e^{\int_{0}^{T-t}f(s)ds}\mathbf{x}_{0}\right)d\mathbf{x}_{0}
\nonumber
\\
&\qquad\qquad\qquad\qquad
+de^{\int_{0}^{T-t}f(s)ds}
-\frac{d}{2}\log\left(2\pi\int_{0}^{t}e^{-2\int_{s}^{t}f(v)dv}(g(s))^{2}ds\right),
\label{eq:conv}
\end{align*}
where we applied change-of-variable to obtain the last equation.
Note that for any two functions $p,q:\mathbb{R}^{d}\rightarrow\mathbb{R}$, 
where $p$ is $m_{p}$-strongly-log-concave and $q$ is $m_{q}$-strongly-log-concave
(i.e. $-\log p$ is $m_{p}$-strongly-convex and $-\log q$ is $m_{q}$-strongly-convex)
then it is known that the convolution of $p$ and $q$, i.e. $\int_{\mathbb{R}^{d}}p(\mathbf{x}-\mathbf{y})q(\mathbf{y})d\mathbf{y}$
is $(m_{p}^{-1}+m_{q}^{-1})^{-1}$-strongly-log-concave; see e.g. Proposition~7.1 in \cite{Saumard2014}. 
It is easy to see that the function
$\mathbf{x}\mapsto\exp\left(-\frac{\Vert\mathbf{x}\Vert^{2}}{2\int_{0}^{T-t}e^{-2\int_{s}^{T-t}f(v)dv}(g(s))^{2}ds}\right)$
is $\frac{1}{\int_{0}^{T-t}e^{-2\int_{s}^{T-t}f(v)dv}(g(s))^{2}ds}$-strongly-log-concave,
and the function
$\mathbf{x}\mapsto p_{0}\left(e^{\int_{0}^{T-t}f(s)ds}\mathbf{x}\right)$ 
is $m_{0}\left(e^{\int_{0}^{T-t}f(s)ds}\right)^{2}$-strongly-log-concave
since we assumed that $\mathbf{x}\mapsto p_{0}(\mathbf{x})$ is $m_{0}$-strongly-log-concave. 
Hence, we conclude that
$\log p_{T-t}(\mathbf{x})$
is $a(T-t)$-strongly-concave, where
\begin{equation}\label{eq:aT-t}
a(T-t):=\frac{1}{\frac{1}{m_{0}}e^{-2\int_{0}^{T-t}f(s)ds}+\int_{0}^{T-t}e^{-2\int_{s}^{T-t}f(v)dv}(g(s))^{2}ds}.
\end{equation}

Next, let us recall the definition of $m(T-t)$ in \eqref{eq:mt}
and the dynamics of $\tilde{\mathbf{x}}_{t}$ and $\mathbf{z}_{t}$ in \eqref{eq:Reverse} and \eqref{eq:zt} respectively. 
By It\^{o}'s formula, 
\begin{align*}
&d\left(\Vert\tilde{\mathbf{x}}_{t}-\mathbf{z}_{t}\Vert^{2}e^{\int_{0}^{t}m(T-s)ds}\right)
\nonumber
\\
&=m(T-t)e^{\int_{0}^{t}m(T-s)ds}\Vert\tilde{\mathbf{x}}_{t}-\mathbf{z}_{t}\Vert^{2}dt
+2e^{\int_{0}^{t}m(T-s)ds}\langle\tilde{\mathbf{x}}_{t}-\mathbf{z}_{t},d\tilde{\mathbf{x}}_{t}-d\mathbf{z}_{t}\rangle 
\nonumber
\\
&=m(T-t)e^{\int_{0}^{t}m(T-s)ds}\Vert\tilde{\mathbf{x}}_{t}-\mathbf{z}_{t}\Vert^{2}dt
+2e^{\int_{0}^{t}m(T-s)ds}\langle\tilde{\mathbf{x}}_{t}-\mathbf{z}_{t},f(T-t)(\tilde{\mathbf{x}}_{t}-\mathbf{z}_{t})\rangle dt
\nonumber
\\
&\qquad
+2e^{\int_{0}^{t}m(T-s)ds}\left\langle\tilde{\mathbf{x}}_{t}-\mathbf{z}_{t},(g(T-t))^{2}\left(\nabla\log p_{T-t}(\tilde{\mathbf{x}}_{t})-\nabla\log p_{T-t}(\mathbf{z}_{t})\right)\right\rangle dt
\nonumber
\\
&\leq
e^{\int_{0}^{t}m(T-s)ds}\left(m(T-t)+2f(T-t)-2(g(T-t))^{2}a(T-t)\right)
\Vert\tilde{\mathbf{x}}_{t}-\mathbf{z}_{t}\Vert^{2}dt
=0.
\end{align*}
This implies that
$\Vert\tilde{\mathbf{x}}_{t}-\mathbf{z}_{t}\Vert^{2}e^{\int_{0}^{t}m(T-s)ds}
\leq
\Vert\tilde{\mathbf{x}}_{0}-\mathbf{z}_{0}\Vert^{2},$
so that
\begin{equation}\label{eq:int-mt}
\mathbb{E}\Vert\tilde{\mathbf{x}}_{T}-\mathbf{z}_{T}\Vert^{2}
\leq e^{-\int_{0}^{T}m(T-s)ds}\mathbb{E}\Vert\tilde{\mathbf{x}}_{0}-\mathbf{z}_{0}\Vert^{2}.
\end{equation}
Consider a coupling of $(\tilde{\mathbf{x}}_{0},\mathbf{z}_{0})$ such that $\tilde{\mathbf{x}}_{0}\sim p_{T}$, $\mathbf{z}_{0}\sim\hat{p}_{T}$
and $\mathbb{E}\Vert\tilde{\mathbf{x}}_{0}-\mathbf{z}_{0}\Vert^{2}=\mathcal{W}_{2}^{2}(p_{T},\hat{p}_{T})$.
Together with \eqref{eq:phatp}, we conclude that
\begin{align*}
\mathcal{W}_{2}^{2}(\mathcal{L}(\mathbf{z}_{T}),p_{0})
=\mathcal{W}_{2}^{2}(\mathcal{L}(\mathbf{z}_{T}), \mathcal{L}(\tilde{\mathbf{x}}_{T}))
&\leq\mathbb{E}\Vert\tilde{\mathbf{x}}_{T}-\mathbf{z}_{T}\Vert^{2}
\\
&\leq e^{-\int_{0}^{T}m(T-s)ds}\mathcal{W}_{2}^{2}(p_{T},\hat{p}_{T})
\\
&\leq e^{-\int_{0}^{T}m(s)ds}e^{-2\int_{0}^{T}f(s)ds}\Vert\mathbf{x}_{0}\Vert_{L_{2}}^{2}
\\
&=e^{-2\int_{0}^{T}c(t)dt}\Vert\mathbf{x}_{0}\Vert_{L_{2}}^{2}.
\end{align*}
The proof is complete. 
\end{proof}

\subsubsection{Proof of Proposition~\ref{prop:iterates}}

We first state a key technical lemma, which will be used in the proof of Proposition~\ref{prop:iterates}. 
The proof of the following result will be provided in Appendix~\ref{sec:proof:lem:smooth}.

\begin{lemma}\label{lem:smooth}
Suppose that Assumption~\ref{assump:p0} holds.
Then, $\nabla_{\mathbf{x}}\log p_{T-t}(\mathbf{x})$ is $L(T-t)$-Lipschitz in $\mathbf{x}$, where  $L(T-t)$ is given in \eqref{eq:Lt}.
\end{lemma}

\begin{proof}
By recalling the dynamics of $\mathbf{\mathbf{z}}_{t}$ and $\hat{\mathbf{y}}_{t}$ from \eqref{eq:zt} and \eqref{continuous:hat:y}, it follows that
\begin{align*}
&\mathbf{z}_{k\eta}-\hat{\mathbf{y}}_{k\eta}
\\
&=\mathbf{z}_{(k-1)\eta}-\hat{\mathbf{y}}_{(k-1)\eta}
+\int_{(k-1)\eta}^{k\eta}f(T-t)\left(\mathbf{z}_{(k-1)\eta}-\hat{\mathbf{y}}_{(k-1)\eta}\right)dt
\nonumber
\\
&\qquad\qquad
+\int_{(k-1)\eta}^{k\eta}(g(T-t))^{2}\left(\nabla\log p_{T-t}(\mathbf{z}_{(k-1)\eta})-\nabla\log p_{T-t}\left(\hat{\mathbf{y}}_{(k-1)\eta}\right)\right)dt
\nonumber
\\
&\qquad
+\int_{(k-1)\eta}^{k\eta}\Big[f(T-t)(\mathbf{z}_{t}-\mathbf{z}_{(k-1)\eta})
\\
&\qquad\qquad\qquad
+(g(T-t))^{2}\left(\nabla\log p_{T-t}(\mathbf{z}_{t})-\nabla\log p_{T-t}(\mathbf{z}_{(k-1)\eta})\right)\Big]dt
\nonumber
\\
&\qquad\qquad
+\int_{(k-1)\eta}^{k\eta}(g(T-t))^{2}\left(\nabla\log p_{T-t}\left(\hat{\mathbf{y}}_{(k-1)\eta}\right)-s_{\theta}\left(\hat{\mathbf{y}}_{(k-1)\eta},T-(k-1)\eta\right)\right)dt.
\end{align*}
This implies that
\begin{align}
&\left\Vert\mathbf{z}_{k\eta}-\hat{\mathbf{y}}_{k\eta}\right\Vert_{L_{2}}
\nonumber
\\
&\leq\Bigg\Vert\mathbf{z}_{(k-1)\eta}-\hat{\mathbf{y}}_{(k-1)\eta}
+\int_{(k-1)\eta}^{k\eta}f(T-t)\left(\mathbf{z}_{(k-1)\eta}-\hat{\mathbf{y}}_{(k-1)\eta}\right)dt
\nonumber
\\
&\qquad\quad
+\int_{(k-1)\eta}^{k\eta}(g(T-t))^{2}\left(\nabla\log p_{T-t}(\mathbf{z}_{(k-1)\eta})-\nabla\log p_{T-t}\left(\hat{\mathbf{y}}_{(k-1)\eta}\right)\right)dt\Bigg\Vert_{L_{2}}
\nonumber
\\
&+\Bigg\Vert\int_{(k-1)\eta}^{k\eta}\Big[f(T-t)(\mathbf{z}_{t}-\mathbf{z}_{(k-1)\eta})
\nonumber
\\
&\qquad\qquad\qquad\qquad
+(g(T-t))^{2}\left(\nabla\log p_{T-t}(\mathbf{z}_{t})-\nabla\log p_{T-t}(\mathbf{z}_{(k-1)\eta})\right)\Big]dt\Bigg\Vert_{L_{2}}
\nonumber
\\
&\quad
+\left\Vert\int_{(k-1)\eta}^{k\eta}(g(T-t))^{2}\left(\nabla\log p_{T-t}\left(\hat{\mathbf{y}}_{(k-1)\eta}\right)-s_{\theta}\left(\hat{\mathbf{y}}_{(k-1)\eta},T-(k-1)\eta\right)\right)dt\right\Vert_{L_{2}}.\label{two:terms}
\end{align}
Next, we provide upper bounds for the three terms in \eqref{two:terms}.

\textbf{Bounding  the first term in \eqref{two:terms}.} 
We can compute that
\begin{align*}
&\Bigg\Vert\mathbf{z}_{(k-1)\eta}-\hat{\mathbf{y}}_{(k-1)\eta}
+\int_{(k-1)\eta}^{k\eta}f(T-t)\left(\mathbf{z}_{(k-1)\eta}-\hat{\mathbf{y}}_{(k-1)\eta}\right)dt
\nonumber
\\
&\qquad\qquad
+\int_{(k-1)\eta}^{k\eta}(g(T-t))^{2}\left(\nabla\log p_{T-t}(\mathbf{z}_{(k-1)\eta})-\nabla\log p_{T-t}\left(\hat{\mathbf{y}}_{(k-1)\eta}\right)\right)dt\Bigg\Vert^{2}
\\
&=\left\Vert\mathbf{z}_{(k-1)\eta}-\hat{\mathbf{y}}_{(k-1)\eta}\right\Vert^{2}
+\Bigg\Vert\int_{(k-1)\eta}^{k\eta}f(T-t)\left(\mathbf{z}_{(k-1)\eta}-\hat{\mathbf{y}}_{(k-1)\eta}\right)dt
\nonumber
\\
&\qquad\qquad
+\int_{(k-1)\eta}^{k\eta}(g(T-t))^{2}\left(\nabla\log p_{T-t}(\mathbf{z}_{(k-1)\eta})-\nabla\log p_{T-t}\left(\hat{\mathbf{y}}_{(k-1)\eta}\right)\right)dt\Bigg\Vert^{2}
\\
&\qquad\qquad\qquad
+2\int_{(k-1)\eta}^{k\eta}f(T-t)\left\Vert\mathbf{z}_{(k-1)\eta}-\hat{\mathbf{y}}_{(k-1)\eta}\right\Vert^{2}dt
\\
&\qquad+2\int_{(k-1)\eta}^{k\eta}\bigg\langle\mathbf{z}_{(k-1)\eta}-\hat{\mathbf{y}}_{(k-1)\eta}, 
\nonumber
\\
&\qquad\qquad\qquad\qquad\qquad
(g(T-t))^{2}\left(\nabla\log p_{T-t}(\mathbf{z}_{(k-1)\eta})-\nabla\log p_{T-t}\left(\hat{\mathbf{y}}_{(k-1)\eta}\right)\right)\bigg\rangle dt.
\end{align*}
From the proof of Proposition~\ref{thm:1}, we know that $\log p_{T-t}(\mathbf{x})$
is $a(T-t)$-strongly-concave, where $a(T-t)$ is given in \eqref{eq:aT-t}. Hence
we have
\begin{align*}
&\Bigg\Vert\mathbf{z}_{(k-1)\eta}-\hat{\mathbf{y}}_{(k-1)\eta}
+\int_{(k-1)\eta}^{k\eta}f(T-t)\left(\mathbf{z}_{(k-1)\eta}-\hat{\mathbf{y}}_{(k-1)\eta}\right)dt
\nonumber
\\
&\qquad\qquad
+\int_{(k-1)\eta}^{k\eta}(g(T-t))^{2}\left(\nabla\log p_{T-t}(\mathbf{z}_{(k-1)\eta})-\nabla\log p_{T-t}\left(\hat{\mathbf{y}}_{(k-1)\eta}\right)\right)dt\Bigg\Vert^{2}
\\
&\leq\left(1-\int_{(k-1)\eta}^{k\eta}m(T-t)dt\right)\left\Vert\mathbf{z}_{(k-1)\eta}-\hat{\mathbf{y}}_{(k-1)\eta}\right\Vert^{2}
\nonumber
\\
&\qquad\qquad
+\Bigg(\int_{(k-1)\eta}^{k\eta}f(T-t)\left\Vert\mathbf{z}_{(k-1)\eta}-\hat{\mathbf{y}}_{(k-1)\eta}\right\Vert dt
\nonumber
\\
&\qquad\qquad\qquad\qquad
+\int_{(k-1)\eta}^{k\eta}(g(T-t))^{2}L(T-t)\left\Vert\mathbf{z}_{(k-1)\eta}-\hat{\mathbf{y}}_{(k-1)\eta}\right\Vert dt\Bigg)^{2}
\\
&\leq
\Bigg(1-\int_{(k-1)\eta}^{k\eta}m(T-t)dt+2\eta\int_{(k-1)\eta}^{k\eta}(f(T-t))^{2}dt
\nonumber
\\
&\qquad\qquad\qquad
+2\eta\int_{(k-1)\eta}^{k\eta}(g(T-t))^{4}(L(T-t))^{2}dt\Bigg)
\cdot\left\Vert\mathbf{z}_{(k-1)\eta}-\hat{\mathbf{y}}_{(k-1)\eta}\right\Vert^{2},
\end{align*}
where we applied Cauchy-Schwartz inequality and Lemma~\ref{lem:smooth}, and $m(T-t)$ is defined in \eqref{eq:mt}.
Hence, we conclude that
\begin{align}\label{first:term}
&\Bigg\Vert\mathbf{z}_{(k-1)\eta}-\hat{\mathbf{y}}_{(k-1)\eta}
+\int_{(k-1)\eta}^{k\eta}f(T-t)\left(\mathbf{z}_{(k-1)\eta}-\hat{\mathbf{y}}_{(k-1)\eta}\right)dt
\nonumber
\\
&\qquad\qquad
+\int_{(k-1)\eta}^{k\eta}(g(T-t))^{2}\left(\nabla\log p_{T-t}\left(\mathbf{z}_{(k-1)\eta}\right)-\nabla\log p_{T-t}\left(\hat{\mathbf{y}}_{(k-1)\eta}\right)\right)dt\Bigg\Vert_{L_{2}}
\nonumber \\
&\leq
\left(1-\int_{(k-1)\eta}^{k\eta}\mu(T-t)dt\right)
\left\Vert\mathbf{z}_{(k-1)\eta}-\hat{\mathbf{y}}_{(k-1)\eta}\right\Vert_{L_{2}},
\end{align}
where we used the inequality $\sqrt{1-x}\leq 1-\frac{x}{2}$ for any $0\leq x\leq 1$
and the definition of $\mu(T-t)$ in \eqref{mu:definition}
which can be rewritten as
\begin{equation*}
\mu(T-t):=(g(T-t))^{2}a(T-t)-f(T-t)-\eta(f(T-t))^{2}-\eta(g(T-t))^{4}(L(T-t))^{2},\quad 0\leq t\leq T,
\end{equation*}
where $a(T-t)$ is given in \eqref{eq:aT-t}.

\textbf{Bounding  the second term in \eqref{two:terms}.} 
Using Lemma~\ref{lem:smooth}, we can compute that
\begin{align*}
&\left\Vert\int_{(k-1)\eta}^{k\eta}\left[f(T-t)(\mathbf{z}_{t}-\mathbf{z}_{(k-1)\eta})+(g(T-t))^{2}\left(\nabla\log p_{T-t}(\mathbf{z}_{t})-\nabla\log p_{T-t}(\mathbf{z}_{(k-1)\eta})\right)\right]dt\right\Vert^{2}
\\
&\leq
\left(\int_{(k-1)\eta}^{k\eta}[f(T-t)+(g(T-t))^{2}L(T-t)]\Vert\mathbf{z}_{t}-\mathbf{z}_{(k-1)\eta}\Vert dt\right)^{2}
\\
&\leq
\eta\int_{(k-1)\eta}^{k\eta}[f(T-t)+(g(T-t))^{2}L(T-t)]^{2}\Vert\mathbf{z}_{t}-\mathbf{z}_{(k-1)\eta}\Vert^{2}dt,
\end{align*}
which implies that 
\begin{align}\label{eq:2ndtermUB}
&\left\Vert\int_{(k-1)\eta}^{k\eta}\left[f(T-t)(\mathbf{z}_{t}-\mathbf{z}_{(k-1)\eta})+(g(T-t))^{2}\left(\nabla\log p_{T-t}(\mathbf{z}_{t})-\nabla\log p_{T-t}\left(\mathbf{z}_{(k-1)\eta}\right)\right)\right]dt\right\Vert_{L_{2}}
\nonumber \\
&\leq
\left(\eta\int_{(k-1)\eta}^{k\eta}[f(T-t)+(g(T-t))^{2}L(T-t)]^{2}dt
\cdot\sup_{(k-1)\eta\leq t\leq k\eta}\mathbb{E}\Vert\mathbf{z}_{t}-\mathbf{z}_{(k-1)\eta}\Vert^{2}
\right)^{1/2}
\nonumber
\\
&=
\sqrt{\eta}\left(\int_{(k-1)\eta}^{k\eta}[f(T-t)+(g(T-t))^{2}L(T-t)]^{2}dt\right)^{1/2}
\sup_{(k-1)\eta\leq t\leq k\eta}\left\Vert\mathbf{z}_{t}-\mathbf{z}_{(k-1)\eta}\right\Vert_{L_{2}}.
\end{align}

\textbf{Bounding the third term in \eqref{two:terms}.}  
We notice that
\begin{align*}
&\left\Vert\int_{(k-1)\eta}^{k\eta}(g(T-t))^{2}\left(\nabla\log p_{T-t}(\hat{\mathbf{y}}_{(k-1)\eta})-s_{\theta}\left(\hat{\mathbf{y}}_{(k-1)\eta},T-(k-1)\eta\right)\right)dt\right\Vert_{L_{2}}
\\
&\leq
\left\Vert\int_{(k-1)\eta}^{k\eta}(g(T-t))^{2}\left(\nabla\log p_{T-(k-1)\eta}(\hat{\mathbf{y}}_{(k-1)\eta})-s_{\theta}\left(\hat{\mathbf{y}}_{(k-1)\eta},T-(k-1)\eta\right)\right)dt\right\Vert_{L_{2}}
\\
&\qquad
+\left\Vert\int_{(k-1)\eta}^{k\eta}(g(T-t))^{2}\left(\nabla\log p_{T-t}\left(\hat{\mathbf{y}}_{(k-1)\eta}\right)-\nabla\log p_{T-(k-1)\eta}\left(\hat{\mathbf{y}}_{(k-1)\eta}\right)\right)dt\right\Vert_{L_{2}}.
\end{align*}
By Assumption~\ref{assump:M}, we have  
\begin{align}
&\left\Vert\int_{(k-1)\eta}^{k\eta}(g(T-t))^{2}\left(\nabla\log p_{T-(k-1)\eta}(\hat{\mathbf{y}}_{(k-1)\eta})-s_{\theta}\left(\hat{\mathbf{y}}_{(k-1)\eta},T-(k-1)\eta\right)\right)dt\right\Vert_{L_{2}}
\nonumber
\\
&\leq M\int_{(k-1)\eta}^{k\eta}(g(T-t))^{2}dt.\label{assump:can:be:Weakened}
\end{align}
Moreover, by Assumption~\ref{assump:M:1}, we have
\begin{align}
&\left\Vert\int_{(k-1)\eta}^{k\eta}(g(T-t))^{2}\left(\nabla\log p_{T-t}\left(\hat{\mathbf{y}}_{(k-1)\eta}\right)-\nabla\log p_{T-(k-1)\eta}\left(\hat{\mathbf{y}}_{(k-1)\eta}\right)\right)dt\right\Vert_{L_{2}}
\nonumber
\\
&\leq
\int_{(k-1)\eta}^{k\eta}(g(T-t))^{2}\left\Vert\nabla\log p_{T-t}\left(\hat{\mathbf{y}}_{(k-1)\eta}\right)-\nabla\log p_{T-(k-1)\eta}\left(\hat{\mathbf{y}}_{(k-1)\eta}\right)\right\Vert_{L_{2}}dt
\nonumber
\\
&\leq
\int_{(k-1)\eta}^{k\eta}(g(T-t))^{2}M_{1}\eta\left(1+\left\Vert\hat{\mathbf{y}}_{(k-1)\eta}\right\Vert_{L_{2}}\right)dt
\nonumber
\\
&\leq
M_{1}\eta\left(1+\left\Vert\mathbf{z}_{(k-1)\eta}-\hat{\mathbf{y}}_{(k-1)\eta}\right\Vert_{L_{2}}
+\Vert\mathbf{z}_{(k-1)\eta}\Vert_{L_{2}}\right)
\int_{(k-1)\eta}^{k\eta}(g(T-t))^{2}dt.\label{by:apply:1}
\end{align}
Furthermore, we can compute that
\begin{equation}\label{by:apply:2}
\left\Vert\mathbf{z}_{(k-1)\eta}\right\Vert_{L_{2}}
\leq
\left\Vert\mathbf{z}_{(k-1)\eta}-\tilde{\mathbf{x}}_{(k-1)\eta}\right\Vert_{L_{2}}
+\left\Vert\tilde{\mathbf{x}}_{(k-1)\eta}\right\Vert_{L_{2}},
\end{equation}
where $\tilde{\mathbf{x}}_{t}$ is defined in \eqref{eq:Reverse}. 
Moreover, by the proof of Proposition~\ref{thm:1}, we have
\begin{equation}\label{by:apply:3}
\left\Vert\mathbf{z}_{(k-1)\eta}-\tilde{\mathbf{x}}_{(k-1)\eta}\right\Vert_{L_{2}}
\leq
\left(\mathbb{E}\Vert\tilde{\mathbf{x}}_{0}-\mathbf{z}_{0}\Vert^{2}\right)^{1/2}
=e^{-\int_{0}^{T}f(s)ds}\Vert\mathbf{x}_{0}\Vert_{L_{2}}
\leq
\Vert\mathbf{x}_{0}\Vert_{L_{2}},
\end{equation}
where we applied \eqref{SDE:solution} to obtain
the equality in the above equation.
Moreover, since $(\tilde{\mathbf{x}}_{t})_{0\leq t\leq T}$ is the time-reversal process
of $(\mathbf{x}_{t})_{0\leq t\leq T}$, we have
\begin{align}
\left\Vert\tilde{\mathbf{x}}_{(k-1)\eta}\right\Vert_{L_{2}}
=\left\Vert\mathbf{x}_{T-(k-1)\eta}\right\Vert_{L_{2}}
\leq
\sup_{0\leq t\leq T}\Vert\mathbf{x}_{t}\Vert_{L_{2}}=:c_{2}(T).\label{by:apply:4}
\end{align}

Next, let us show that $c_{2}(T)$ can be computed as given by the formula in \eqref{c:2:defn}.
By applying It\^{o}'s formula to equation~\eqref{SDE:solution}, we have
\begin{align*}
&d\left(\Vert\mathbf{x}_{t}\Vert^{2}e^{2\int_{0}^{t}f(s)ds}\right)
\\
&=2f(t)\Vert\mathbf{x}_{t}\Vert^{2}e^{2\int_{0}^{t}f(s)ds}dt
+2e^{2\int_{0}^{t}f(s)ds}\langle\mathbf{x}_{t},d\mathbf{x}_{t}\rangle
+e^{2\int_{0}^{t}f(s)ds}\cdot d\cdot (g(t))^{2}dt.
\end{align*}
By taking expectations, we obtain
\begin{equation*}
d\left(\mathbb{E}\Vert\mathbf{x}_{t}\Vert^{2}e^{2\int_{0}^{t}f(s)ds}\right)
=e^{2\int_{0}^{t}f(s)ds}\cdot d\cdot (g(t))^{2}dt,
\end{equation*}
so that
\begin{equation*}
\mathbb{E}\Vert\mathbf{x}_{t}\Vert^{2}
=e^{-2\int_{0}^{t}f(s)ds}\mathbb{E}\Vert\mathbf{x}_{0}\Vert^{2}
+d\int_{0}^{t}e^{-2\int_{s}^{t}f(v)dv}(g(s))^{2}ds.
\end{equation*}
Therefore, we conclude that
\begin{align}\label{c:2:source}
c_{2}(T)=\sup_{0\leq t\leq T}\Vert\mathbf{x}_{t}\Vert_{L_{2}}
=\sup_{0\leq t\leq T}\left(e^{-2\int_{0}^{t}f(s)ds}\Vert\mathbf{x}_{0}\Vert_{L_{2}}^{2}
+d\int_{0}^{t}e^{-2\int_{s}^{t}f(v)dv}(g(s))^{2}ds\right)^{1/2}.
\end{align}

Therefore, by applying \eqref{by:apply:1}, \eqref{by:apply:2}, \eqref{by:apply:3} and \eqref{by:apply:4}, we have
\begin{align*}
&\left\Vert\int_{(k-1)\eta}^{k\eta}(g(T-t))^{2}\left(\nabla\log p_{T-t}\left(\hat{\mathbf{y}}_{(k-1)\eta}\right)-\nabla\log p_{T-(k-1)\eta}\left(\hat{\mathbf{y}}_{(k-1)\eta}\right)\right)dt\right\Vert_{L_{2}}
\\
&\leq  
\left(M_{1}\eta\left\Vert\mathbf{z}_{(k-1)\eta}-\hat{\mathbf{y}}_{(k-1)\eta}\right\Vert_{L_{2}}
+M_{1}\eta\left(1+\Vert\mathbf{x}_{0}\Vert_{L_{2}}
+c_{2}(T)\right)\right)\int_{(k-1)\eta}^{k\eta}(g(T-t))^{2}dt.
\end{align*}
It follows that the third term in \eqref{two:terms} is upper bounded by
\begin{align}\label{eq:3rdterm}
\left(M_{1}\eta\left\Vert\mathbf{z}_{(k-1)\eta}-\hat{\mathbf{y}}_{(k-1)\eta}\right\Vert_{L_{2}}
+M_{1}\eta\left(1+\Vert\mathbf{x}_{0}\Vert_{L_{2}}
+c_{2}(T)\right)+M\right)
\int_{(k-1)\eta}^{k\eta}(g(T-t))^{2}dt.
\end{align}


\textbf{Bounding \eqref{two:terms}.}
On combining \eqref{first:term}, \eqref{eq:2ndtermUB} and \eqref{eq:3rdterm}, we conclude that
\begin{align} \label{eq:final2t}
&\left\Vert\mathbf{z}_{k\eta}-\hat{\mathbf{y}}_{k\eta}\right\Vert_{L_{2}}^{2}
\\
&\leq\Bigg\{\gamma_{k,\eta}
\left\Vert\mathbf{z}_{(k-1)\eta}-\hat{\mathbf{y}}_{(k-1)\eta}\right\Vert_{L_{2}}
\nonumber
\\
&\quad
+M_{1}\eta\left(1+\Vert\mathbf{x}_{0}\Vert_{L_{2}}
+c_{2}(T)\right)\int_{(k-1)\eta}^{k\eta}(g(T-t))^{2}dt
+M\int_{(k-1)\eta}^{k\eta}(g(T-t))^{2}dt
\nonumber
\\
&\qquad
+\sqrt{\eta}\left(\int_{(k-1)\eta}^{k\eta}[f(T-t)+(g(T-t))^{2}L(T-t)]^{2}dt\right)^{1/2}
\sup_{(k-1)\eta\leq t\leq k\eta}\Vert\mathbf{z}_{t}-\mathbf{z}_{(k-1)\eta}\Vert_{L_{2}}\Bigg\}^{2}, \nonumber
\end{align}
where we used the definition of $\gamma_{k,\eta}$ in \eqref{gamma:k:defn}.
We need one more result, which provides an upper bound
for $\sup_{(k-1)\eta\leq t\leq k\eta}\Vert\mathbf{z}_{t}-\mathbf{z}_{(k-1)\eta}\Vert_{L_{2}}$. The proof of Lemma~\ref{lem:second:term} is given in Appendix~\ref{app:2ndterm}.

\begin{lemma}\label{lem:second:term}
For any $k=1,2,\ldots,K$,
\begin{align*}
\sup_{(k-1)\eta\leq t\leq k\eta}\left\Vert\mathbf{z}_{t}-\mathbf{z}_{(k-1)\eta}\right\Vert_{L_{2}}
&\leq c_{1}(T)\int_{(k-1)\eta}^{k\eta}\left[f(T-s)+(g(T-s))^{2}L(T-s)\right]ds
\\
&\quad
+c_{2}(T)\int_{T-k\eta}^{T-(k-1)\eta}f(s)ds
+\left(\int_{T-k\eta}^{T-(k-1)\eta}(g(s))^{2}ds\right)^{1/2}\sqrt{d},
\end{align*}
where $c_{1}(T)$ and $c_{2}(T)$ are given in \eqref{c:1:defn}-\eqref{c:2:defn} respectively.
\end{lemma}

By applying Lemma~\ref{lem:second:term}, 
we conclude from \eqref{eq:final2t} that
\begin{align*}
&\left\Vert\mathbf{z}_{k\eta}-\hat{\mathbf{y}}_{k\eta}\right\Vert_{L_{2}}
\nonumber
\\
&\leq \gamma_{k,\eta}
\left\Vert\mathbf{z}_{(k-1)\eta}-\hat{\mathbf{y}}_{(k-1)\eta}\right\Vert_{L_{2}}
\nonumber
\\
&\qquad
+M_{1}\eta\left(1+ \Vert\mathbf{x}_{0}\Vert_{L_{2}} 
+c_{2}(T)\right)\int_{(k-1)\eta}^{k\eta}(g(T-t))^{2}dt
+M\int_{(k-1)\eta}^{k\eta}(g(T-t))^{2}dt
\nonumber
\\
&\qquad\qquad
+\sqrt{\eta}h_{k,\eta}\left(\int_{(k-1)\eta}^{k\eta}[f(T-t)+(g(T-t))^{2}L(T-t)]^{2}dt\right)^{1/2},
\end{align*}
where $h_{k,\eta}$ is defined in \eqref{h:k:eta:main}. The proof of Proposition~\ref{prop:iterates} is hence complete.
\end{proof}





\subsection{Proof of Proposition~\ref{prop:lower:bound}}

\begin{proof}
By \eqref{h:k:eta:main}, we have
$h_{k,\eta}\geq\min_{0\leq t\leq T}g(t)\sqrt{\eta}\sqrt{d}.$
Therefore, we have
\begin{align*}
&\text{RHS of \eqref{main:thm:upper:bound}}
\nonumber
\\
&\geq 
e^{-\int_{0}^{K\eta}c(t)dt} \Vert\mathbf{x}_{0}\Vert_{L_{2}} 
+\sum_{k=1}^{K}
\prod_{j=k+1}^{K}\left(1-\int_{(j-1)\eta}^{j\eta}\mu(T-t)dt\right)\nonumber
\\
&\quad
\cdot\left(M\int_{(k-1)\eta}^{k\eta}(g(T-t))^{2}dt+\sqrt{\eta}h_{k,\eta}\left(\int_{(k-1)\eta}^{k\eta}[f(T-t)+(g(T-t))^{2}L(T-t)]^{2}dt\right)^{1/2}\right)
\nonumber
\\
&\geq
e^{-\int_{0}^{K\eta}c(t)dt} \Vert\mathbf{x}_{0}\Vert_{L_{2}} 
\\
&\qquad
+\sum_{k=1}^{K}\left(1-\eta\max_{0\leq t\leq T}\mu(t)\right)^{K-k}
\left(\sqrt{\eta}\min_{0\leq t\leq T}g(t)\sqrt{\eta}\sqrt{d}\sqrt{\eta}\min_{0\leq t\leq T}\left(f(t)+(g(t))^{2}L(t)\right)\right)
\nonumber
\\
&\geq
e^{-\int_{0}^{K\eta}c(t)dt} \Vert\mathbf{x}_{0}\Vert_{L_{2}} 
\\
&\qquad\qquad
+\sqrt{\eta}\sqrt{d}\frac{1-e^{-K\eta\max_{0\leq t\leq T}\mu(t)}}{\max_{0\leq t\leq T}\mu(t)}
\left(\min_{0\leq t\leq T}g(t)\min_{0\leq t\leq T}\left(f(t)+(g(t))^{2}L(t)\right)\right),
\end{align*}
where the equality above is due to the formula for the finite sum of a geometric series 
and we used the inequality that $1-x\leq e^{-x}$ for any $0\leq x\leq 1$ to obtain the last inequality above.
Therefore, in order for $\text{RHS of \eqref{main:thm:upper:bound}}\leq\epsilon$, 
we must have
$e^{-\int_{0}^{K\eta}c(t)dt}\Vert\mathbf{x}_{0}\Vert_{L_{2}} \leq\epsilon$,
which implies that $K\eta\rightarrow\infty$ as $\epsilon\rightarrow 0$ and in particular 
\begin{equation}\label{T:lower:bound}
T=K\eta=\Omega(1), 
\end{equation}
and we also need
\begin{equation}\label{eqn:implies:that}
\sqrt{\eta}\sqrt{d}\frac{1-e^{-K\eta\max_{0\leq t\leq T}\mu(t)}}{\max_{0\leq t\leq T}\mu(t)}
\left(\min_{0\leq t\leq T}g(t)\min_{0\leq t\leq T}\left(f(t)+(g(t))^{2}L(t)\right)\right)
\leq\epsilon.
\end{equation}
Note that by the definition of $\mu(t)$ in \eqref{eq:mt} and $c(t)$ in \eqref{c:t:defn}, we have
$\max_{0\leq t\leq T}\mu(t)\leq\max_{0\leq t\leq T}c(t)=\max_{0\leq t\leq K\eta}c(t)$.
Note Assumption~\ref{assump:p0} implies that 
\begin{align}\label{eq:L2-x0}
\Vert\mathbf{x}_{0}\Vert_{L_{2}}\leq\sqrt{2d/m_{0}}+\Vert\mathbf{x}_{\ast}\Vert,
\end{align}
where $\mathbf{x}_{\ast}$ is the unique minimizer of $-\log p_{0}$. See Lemma~11 in \cite{distMCMC}.
Hence we have $e^{-\int_{0}^{K\eta}c(t)dt}=\mathcal{O}(\epsilon/\sqrt{d})$. 
Together with the assumption that $\max_{0\leq s\leq t}c(s)\leq c_{1}\left(\int_{0}^{t}c(s)ds\right)^{\rho}+c_{2}$
uniformly in $t$ for some $c_{1},c_{2},\rho>0$, 
it is easy to see that $\max_{0\leq t\leq K\eta}c(t)=\mathcal{O}\left(\left(\log\left(\sqrt{d}/\epsilon\right)\right)^{\rho}\right)$ and hence $\max_{0\leq t\leq T}\mu(t)=\mathcal{O}\left(\left(\log\left(\sqrt{d}/\epsilon\right)\right)^{\rho}\right)$.
Moreover, under our assumption $\mu(t)\geq\frac{1}{4}m(t)$ for any $0\leq t\leq T$, where $\mu(t),m(t)$ are defined in \eqref{mu:definition} and \eqref{eq:mt} and since we assumed $\min_{t\geq 0}g(t)>0$, we have $m(t)>0$ for any $t$.
Together with $T=K\eta = \Omega(1)$ from \eqref{T:lower:bound}, we have $\max_{0\leq t\leq T}\mu(t)\geq\Omega(1)$.
Since $K\eta\rightarrow\infty$ as $\epsilon\rightarrow 0$, we have $1-e^{-K\eta\max_{0\leq t\leq T}\mu(t)}=\Omega(1)$.
Therefore, it follows from \eqref{eqn:implies:that} that $\eta = \widetilde{\mathcal{O}}\left(\frac{\epsilon^{2}}{d}\right)$,
where $\widetilde{\mathcal{O}}$ ignores the logarithmic dependence on $\epsilon$ and $d$
and we used the assumptions 
that $\min_{t\geq 0}g(t)>0$
and $\min_{t\geq 0}(f(t)+(g(t))^{2}L(t))>0$.
Hence, we conclude that we have the following lower bound for the complexity: $K=\widetilde{\Omega}\left(\frac{d}{\epsilon^{2}}\right),$
where $\widetilde{\Omega}$ ignores the logarithmic dependence on $\epsilon$ and $d$.
This completes the proof.
\end{proof}

\subsection{Proof of Proposition~\ref{prop:lower:bound:Gaussian}}

\begin{proof} 
When $\mathbf{x}_{0}\sim\mathcal{N}(0,\sigma_{0}^{2}I_{d})$,
we can compute that
\begin{align*}
\nabla_{\mathbf{x}}\log p_{T-t}(\mathbf{x}) &=\nabla_{\mathbf{x}}\log\int_{\mathbb{R}^{d}}\exp\left(-\frac{\Vert\mathbf{x}-e^{-\int_{0}^{T-t}f(s)ds}\mathbf{x}_{0}\Vert^{2}}{2\int_{0}^{T-t}e^{-2\int_{s}^{T-t}f(v)dv}(g(s))^{2}ds}\right)
p_{0}(\mathbf{x}_{0})d\mathbf{x}_{0}
\nonumber
\\
&=\nabla_{\mathbf{x}}\log\int_{\mathbb{R}^{d}}\exp\left(-\frac{\Vert\mathbf{x}-e^{-\int_{0}^{T-t}f(s)ds}\mathbf{x}_{0}\Vert^{2}}{2\int_{0}^{T-t}e^{-2\int_{s}^{T-t}f(v)dv}(g(s))^{2}ds}\right)
\exp\left(-\frac{\Vert\mathbf{x}_{0}\Vert^{2}}{2\sigma_{0}^{2}}\right)d\mathbf{x}_{0}
\nonumber
\\
&=\nabla_{\mathbf{x}}\log\int_{\mathbb{R}^{d}}\exp\left(-\frac{\Vert((a_{1}(T-t))^{2}\sigma_{0}^{2}+a_{2}(T-t))\mathbf{x}_{0}-a_{1}(T-t)\sigma_{0}^{2}\mathbf{x}\Vert^{2}}{2a_{2}(T-t)\sigma_{0}^{2}((a_{1}(T-t))^{2}\sigma_{0}^{2}+a_{2}(T-t))}\right)
\nonumber
\\
&\qquad\qquad\qquad\qquad\qquad\qquad\cdot
\exp\left(-\frac{\Vert\mathbf{x}\Vert^{2}}{2((a_{1}(T-t))^{2}\sigma_{0}^{2}+a_{2}(T-t))}\right)d\mathbf{x}_{0}
\nonumber
\\
&=-\frac{1}{(a_{1}(T-t))^{2}\sigma_{0}^{2}+a_{2}(T-t)}\mathbf{x},
\end{align*}
where
\begin{align}\label{a:b:T:t:defn}
a_{1}(T-t):=e^{-\int_{0}^{T-t}f(s)ds},
\qquad
a_{2}(T-t):=\int_{0}^{T-t}e^{-2\int_{s}^{T-t}f(v)dv}(g(s))^{2}ds.
\end{align}
Therefore, under the assumption $\mathbf{x}_{0}\sim\mathcal{N}(0,\sigma_{0}^{2}I_{d})$, 
the discretization of \eqref{eq:zt} is given by:
\begin{equation}\label{y:dynamics:linear}
\mathbf{y}_{k}=\left(1-\int_{(k-1)\eta}^{k\eta}\alpha(T-t)dt\right)\mathbf{y}_{k-1}
+\left(\int_{(k-1)\eta}^{k\eta}(g(T-t))^{2}dt\right)^{1/2}\xi_{k},
\end{equation}
where 
\begin{equation}\label{alpha:T:t:defn}
\alpha(T-t):=\frac{(g(T-t))^{2}}{(a_{1}(T-t))^{2}\sigma_{0}^{2}+a_{2}(T-t)}-f(T-t),   
\end{equation}
where $a_{1}(T-t)$ and $a_{2}(T-t)$ are defined in \eqref{a:b:T:t:defn},
and $\xi_{k}$ are i.i.d. Gaussian random vectors $\mathcal{N}(0,I_{d})$
and $\mathbf{y}_{0}$ follows the same distribution as $\hat{p}_{T}$ in \eqref{eq:hatp}.

Since $\mathbf{y}_{K}$ and $\mathbf{x}_{0}$ are both centered Gaussian random vectors,
by using the explicit formula for the $\mathcal{W}_{2}$ distance between two Gaussian distributions, 
we have
\begin{equation*}
\mathcal{W}_{2}\left(\mathcal{L}(\mathbf{y}_{K}),\mathcal{L}(\mathbf{x}_{0})\right)
=\left(\text{Tr}\left(\Sigma_{K}+\sigma_{0}^{2}I_{d}-2\left(\Sigma_{K}^{1/2}\sigma_{0}^{2}I_{d}\Sigma_{K}^{1/2}\right)^{1/2}\right)\right)^{1/2},
\end{equation*}
where
$\Sigma_{k}=\mathbb{E}\left[\mathbf{y}_{k}\mathbf{y}_{k}^{\top}\right]$, $k=0,1,\ldots,K$,
is the covariance matrix of $\mathbf{y}_{k}$.
It is easy to compute that for any $k=1,2,\ldots,K$,
$\Sigma_{k}=\left(1-\int_{(k-1)\eta}^{k\eta}\alpha(T-t)dt\right)^{2}\Sigma_{k-1}
+\int_{(k-1)\eta}^{k\eta}(g(T-t))^{2}dt\cdot I_{d}$,
where $\alpha(T-t)$ is defined in \eqref{alpha:T:t:defn}
with
$\Sigma_{0}=\int_{0}^{T}e^{-2\int_{s}^{T}f(v)dv}(g(s))^{2}ds\cdot I_{d}$.
Therefore, one can deduce that
$\Sigma_{k}=\hat{\sigma}_{k}^{2}I_{d}$, $k=0,1,\ldots,K$,
where
\begin{equation}\label{hat:sigma:iterates}
\hat{\sigma}_{k}^{2}=\left(1-\int_{(k-1)\eta}^{k\eta}\alpha(T-t)dt\right)^{2}\hat{\sigma}_{k-1}^{2}
+\int_{(k-1)\eta}^{k\eta}(g(T-t))^{2}dt,
\end{equation}
with
$\hat{\sigma}_{0}^{2}=\int_{0}^{T}e^{-2\int_{s}^{T}f(v)dv}(g(s))^{2}ds$.
Moreover, we get
\begin{align}
\mathcal{W}_{2}\left(\mathcal{L}(\mathbf{y}_{K}),\mathcal{L}(\mathbf{x}_{0})\right)
=\left(\text{Tr}\left(\hat{\sigma}_{K}^{2}I_{d}+\sigma_{0}^{2}I_{d}-2\left(\hat{\sigma}_{K}\sigma_{0}^{2}\hat{\sigma}_{K}\right)^{1/2}I_{d}\right)\right)^{1/2}
=\sqrt{d}\left|\hat{\sigma}_{K}-\sigma_{0}\right|.
\end{align}
One can easily compute from \eqref{hat:sigma:iterates} that
\begin{align}
\hat{\sigma}_{K}^{2}
&=\sum_{j=1}^{K}\prod_{i=j+1}^{K}\left(1-\int_{(i-1)\eta}^{i\eta}\alpha(T-t)dt\right)^{2}
\int_{(j-1)\eta}^{j\eta}(g(T-t))^{2}dt
\nonumber
\\
&\qquad\qquad\qquad
+\prod_{i=1}^{K}\left(1-\int_{(i-1)\eta}^{i\eta}\alpha(T-t)dt\right)^{2}\int_{0}^{T}e^{-2\int_{s}^{T}f(v)dv}(g(s))^{2}ds,
\end{align}
where $\alpha(T-t)$ is defined in \eqref{alpha:T:t:defn}.

By discrete approximation of a Riemann integral, with fixed $K\eta=T$, 
we have
\begin{equation*}
\left|\hat{\sigma}_{K}^{2}-\int_{0}^{T}e^{-2\int_{t}^{T}\alpha(T-s)ds}(g(T-t))^{2}dt
-e^{-2\int_{0}^{T}\alpha(s)ds}\int_{0}^{T}e^{-2\int_{s}^{T}f(v)dv}(g(s))^{2}ds
\right|
=\Theta(\eta),
\end{equation*}
as $\eta\rightarrow 0$.
Indeed, one can show that there exists some $c_{0}\in\mathbb{R}$, such that
\begin{equation}\label{expansion:to:show}
\hat{\sigma}_{K}^{2}
=\int_{0}^{T}e^{-2\int_{t}^{T}\alpha(T-s)ds}(g(T-t))^{2}dt
+e^{-2\int_{0}^{T}\alpha(s)ds}\int_{0}^{T}e^{-2\int_{s}^{T}f(v)dv}(g(s))^{2}ds
+c_{0}\eta+\mathcal{O}(\eta^{2}),
\end{equation}
as $\eta\rightarrow 0$.
Next, let us show that \eqref{expansion:to:show} holds 
as well as spell out
the constant $c_{0}$ explicitly.

First, we can compute that
\begin{align*}
&\log 
\prod_{i=1}^{K}\left(1-\int_{(i-1)\eta}^{i\eta}\alpha(T-t)dt\right)^{2}
=2\sum_{i=1}^{K}\log\left(1-\int_{(i-1)\eta}^{i\eta}\alpha(T-t)dt\right)
\nonumber
\\
&=2\sum_{i=1}^{K}\left(-\int_{(i-1)\eta}^{i\eta}\alpha(T-t)dt-\frac{1}{2}\left(\int_{(i-1)\eta}^{i\eta}\alpha(T-t)dt\right)^{2}+\mathcal{O}(\eta^{3})\right),
\end{align*}
which implies that
\begin{align*}
\log 
\prod_{i=1}^{K}\left(1-\int_{(i-1)\eta}^{i\eta}\alpha(T-t)dt\right)^{2}
-\log\left(e^{-2\int_{0}^{T}\alpha(T-t)dt}\right)
=-\eta\int_{0}^{T}(\alpha(T-t))^{2}dt+\mathcal{O}(\eta^{2}).
\end{align*}
Therefore, we have
\begin{align*}
\prod_{i=1}^{K}\left(1-\int_{(i-1)\eta}^{i\eta}\alpha(T-t)dt\right)^{2}
&=e^{-2\int_{0}^{T}\alpha(s)ds}
e^{\log 
\prod_{i=1}^{K}\left(1-\int_{(i-1)\eta}^{i\eta}\alpha(T-t)dt\right)^{2}
-\log\left(e^{-2\int_{0}^{T}\alpha(T-t)dt}\right)}
\nonumber
\\
&=e^{-2\int_{0}^{T}\alpha(s)ds}
\left(1-\eta\int_{0}^{T}(\alpha(T-t))^{2}dt+\mathcal{O}(\eta^{2})\right).
\end{align*}
Similarly, we can show that
\begin{align*}
&\sum_{j=1}^{K}\prod_{i=j+1}^{K}\left(1-\int_{(i-1)\eta}^{i\eta}\alpha(T-t)dt\right)^{2}
\int_{(j-1)\eta}^{j\eta}(g(T-t))^{2}dt
\nonumber
\\
&=\sum_{j=1}^{K}\int_{(j-1)\eta}^{j\eta}
e^{-2\int_{j\eta}^{K\eta}\alpha(T-s)ds}(g(T-t))^{2}dt
\nonumber
\\
&\qquad\qquad
-\eta\int_{0}^{T}e^{-2\int_{t}^{T}\alpha(T-s)ds}
\left(\int_{t}^{T}(\alpha(T-s))^{2}ds\right)(g(T-t))^{2}dt
+\mathcal{O}(\eta^{2}).
\end{align*}
Moreover, 
\begin{align*}
&\sum_{j=1}^{K}\int_{(j-1)\eta}^{j\eta}
e^{-2\int_{j\eta}^{K\eta}\alpha(T-s)ds}(g(T-t))^{2}dt
\nonumber
\\
&=\int_{0}^{T}
e^{-2\int_{t}^{T}\alpha(T-s)ds}(g(T-t))^{2}dt
+\eta\int_{0}^{T}
e^{-2\int_{t}^{T}\alpha(T-s)ds}\alpha(T-t)(g(T-t))^{2}dt
+\mathcal{O}(\eta^{2}).
\end{align*}
Hence, we conclude that \eqref{expansion:to:show} holds with
\begin{align}
c_{0}&=-e^{-2\int_{0}^{T}\alpha(s)ds}\int_{0}^{T}(\alpha(t))^{2}dt
\int_{0}^{T}e^{-2\int_{s}^{T}f(v)dv}(g(s))^{2}ds
\nonumber
\\
&\qquad
-\int_{0}^{T}e^{-2\int_{0}^{t}\alpha(s)ds}
\left(\int_{0}^{t}(\alpha(s))^{2}ds\right)(g(t))^{2}dt
+\int_{0}^{T}e^{-2\int_{0}^{t}\alpha(s)ds}\alpha(t)(g(t))^{2}dt.\label{c:1:formula}
\end{align}

On the other hand, $\tilde{\mathbf{x}}_{T-t}$ has the same distribution as $\mathbf{x}_{t}$
for any $0\leq t\leq T$, where
\begin{equation*}
d\tilde{\mathbf{x}}_{t}=-\alpha(T-t)\tilde{\mathbf{x}}_{t}dt
+g(T-t)d\bar{\mathbf{B}}_{t},
\end{equation*}
with $\tilde{\mathbf{x}}_{0}\sim p_{T}$ such that
\begin{equation*}
\tilde{\mathbf{x}}_{0}=e^{-\int_{0}^{T}f(s)ds}\mathbf{x}_{0}+\int_{0}^{T}e^{-\int_{s}^{T}f(v)dv}g(s)d\mathbf{B}_{s}.
\end{equation*}
Since  $\tilde{\mathbf{x}}_{T}$ has the same distribution as $\mathbf{x}_{0}$,
their covariance matrices are the same such that
\begin{align*}
&\mathbb{E}\left[\tilde{\mathbf{x}}_{T}\tilde{\mathbf{x}}_{T}^{\top}\right]
=\int_{0}^{T}e^{-2\int_{t}^{T}\alpha(T-s)ds}(g(T-t))^{2}dt\cdot I_{d}
\nonumber
\\
&\qquad
+e^{-2\int_{0}^{T}\alpha(s)ds}\left(e^{-2\int_{0}^{T}f(s)ds}\sigma_{0}^{2}
+\int_{0}^{T}e^{-2\int_{s}^{T}f(v)dv}(g(s))^{2}ds\right)\cdot I_{d}
=\mathbb{E}\left[\mathbf{x}_{0}\mathbf{x}_{0}^{\top}\right]=\sigma_{0}^{2}I_{d}.
\end{align*}
This implies that
\begin{equation}\label{sigma:K:2:sigma:0:2}
\hat{\sigma}_{K}^{2}
=\sigma_{0}^{2}-e^{-2\int_{0}^{T}\alpha(s)ds}e^{-2\int_{0}^{T}f(s)ds}\sigma_{0}^{2}
+c_{0}\eta+\mathcal{O}(\eta^{2}),    
\end{equation}
where $\alpha(\cdot)$ is defined in \eqref{alpha:T:t:defn}
and $c_{0}$ is given in \eqref{c:1:formula}, and by the definition of $\alpha(\cdot)$, 
we can equivalently write \eqref{sigma:K:2:sigma:0:2}
as
\begin{equation}\label{sigma:K:2:sigma:0:2:2}
\hat{\sigma}_{K}^{2}
=\sigma_{0}^{2}-e^{-2\int_{0}^{T}\gamma(s)ds}\sigma_{0}^{2}
+c_{0}\eta+\mathcal{O}(\eta^{2}),    
\end{equation}
where
\begin{equation}\label{defn:gamma:t}
\gamma(t):=\frac{(g(t))^{2}}{(a_{1}(t))^{2}\sigma_{0}^{2}+a_{2}(t)},
\end{equation}
with $a_{1}(\cdot),a_{2}(\cdot)$ defined in \eqref{a:b:T:t:defn}.

If there exists some $T=T(\epsilon)$ and $\bar{\eta}=\bar{\eta}(\epsilon)$ 
such that
$\mathcal{W}_{2}(\mathcal{L}(\mathbf{y}_{K}),p_{0})\leq\epsilon$
for any $K\geq\bar{K}:=T/\bar{\eta}$ (with $\eta=T/K\leq\bar{\eta}$)
then we must have
$\mathcal{W}_{2}\left(\mathcal{L}(\mathbf{y}_{K}),\mathcal{L}(\mathbf{x}_{0})\right)
=\sqrt{d}\left|\hat{\sigma}_{K}-\sigma_{0}\right|
\leq\epsilon$,
for any $K\geq\bar{K}:=T/\bar{\eta}$ so that
\begin{align}\label{square:upper:bound}
\left|\hat{\sigma}_{K}^{2}-\sigma_{0}^{2}\right|
=|\hat{\sigma}_{K}-\sigma_{0}|
(\hat{\sigma}_{K}+\sigma_{0})
\leq
\frac{\epsilon}{\sqrt{d}}\left(2\sigma_{0}+\frac{\epsilon}{\sqrt{d}}\right),
\end{align}
for any $K\geq\bar{K}:=T/\bar{\eta}$ (with $\eta=T/K\leq\bar{\eta}$).
Then, it follows from \eqref{sigma:K:2:sigma:0:2:2} and \eqref{square:upper:bound} that
\begin{equation}\label{square:upper:bound:2}
\left|-e^{-2\int_{0}^{T}\gamma(s)ds}\sigma_{0}^{2}
+c_{0}\eta+\mathcal{O}(\eta^{2})\right|
\leq\mathcal{O}\left(\epsilon/\sqrt{d}\right).
\end{equation}
Since \eqref{square:upper:bound:2} holds for any $\eta\leq\bar{\eta}$, 
by letting $\eta\rightarrow 0$ in \eqref{square:upper:bound:2}, we get
\begin{equation}\label{square:upper:bound:3}
e^{-2\int_{0}^{T}\gamma(s)ds}\sigma_{0}^{2}
\leq\mathcal{O}\left(\epsilon/\sqrt{d}\right).
\end{equation}
By the definition of $\gamma(t)$ in \eqref{defn:gamma:t}, 
it is positive and continuous for any $t>0$ and therefore,
it follows from \eqref{square:upper:bound:3}
that $T\rightarrow\infty$ as $\epsilon\rightarrow 0$, 
and in particular $T\geq\Omega(1)$. 
Finally, by letting $\eta=\bar{\eta}$ in \eqref{square:upper:bound:2},
applying \eqref{square:upper:bound:3}, 
we deduce that $\bar{\eta}\leq\mathcal{O}\left(\frac{\epsilon}{\sqrt{d}}\right)$.
Hence, by $T\geq\Omega(1)$, 
we conclude that
$\bar{K}=\frac{T}{\bar{\eta}}\geq\mathcal{O}\left(\frac{\sqrt{d}}{\epsilon}\right)$.
This completes the proof.  
\end{proof}

\section{Conclusion and Future Work}

In this paper, we establish convergence guarantees for a general class of score-based generative models in the 2-Wasserstein distance for smooth log-concave data distributions. Our theoretical result directly leads to iteration complexity bounds for various score-based generative models with different forward processes. Moreover, our experimental results align well with our theoretical predictions on the iteration complexity. 

Our work serves as a first step towards a better understanding of the impacts of different choices of forward processes in the SDE implementation of diffusion models.  
It is a significant open question how to relax the assumption of strong log-concavity on the data distribution.
Our convergence analysis borrows the idea of synchronous coupling used in sampling with Langevin algorithms \citep{DK2017}. To go beyond the log-concave setting, one may consider more sophisticated coupling methods such as reflection coupling (see e.g. \cite{eberle2016reflection}) to obtain contraction rates of SDEs in Wasserstein distance. However, it is not clear whether the reflection coupling can be applied to the convergence analysis of score-based diffusion models, because the reverse SDE is time-inhomogeneous and it is also unclear whether the coefficients in the reverse SDE satisfy a dissipativity-type condition in \cite{eberle2016reflection}. In addition to considering weaker conditions on the target data distribution, 
another interesting direction is to study the convergence theory for alternative sampling schemes (beyond the Euler-Maruyama discretization of reverse SDEs), such as (stochastic) EDM in \cite{Karras2022}. Furthermore, it is also important to study the complexity of score estimations and establish an end-to-end convergence theory for diffusion models (see e.g. \cite{chen2023score, han2024neural} for some recent progress). Finally, our empirical analysis focuses on image generation using CIFAR-10 data exclusively. It would be intriguing to explore additional datasets and tasks. We leave them for future research.

\bibliography{generative}

\newpage
\appendix

\section{Additional Technical Proofs}

\subsection{Proof of Lemma~\ref{lem:smooth}}\label{sec:proof:lem:smooth}

\begin{proof}
First of all, by following the proof of Proposition~\ref{thm:1},
we have
\begin{equation}\label{convolution:eqn}
p_{T-t}(\mathbf{x})
=\int_{\mathbb{R}^{d}}q_{1}(\mathbf{x}-\mathbf{x}_{0})q_{0}(\mathbf{x}_{0})d\mathbf{x}_{0},
\end{equation}
where
\begin{align*}
q_{1}(\mathbf{x}):=\frac{\exp\left(-\frac{\Vert\mathbf{x}\Vert^{2}}{2\int_{0}^{T-t}e^{-2\int_{s}^{T-t}f(v)dv}(g(s))^{2}ds}\right)}{\left(2\pi\int_{0}^{T-t}e^{-2\int_{s}^{T-t}f(v)dv}(g(s))^{2}ds\right)^{d/2}},
\,\,
q_{0}(\mathbf{x}):=\left(e^{\int_{0}^{T-t}f(s)ds}\right)^{d}p_{0}\left(e^{\int_{0}^{T-t}f(s)ds}\mathbf{x}\right).
\end{align*}
Let $\mathbf{X}_{1}$ and $\mathbf{X}_{0}$ be two independent random vectors
with densities $q_{1}$ and $q_{0}$ respectively. Then it follows from \eqref{convolution:eqn} that
$p_{T-t}$ is the density
of $\mathbf{X}_{1}+\mathbf{X}_{0}$. 
Moreover, let us write:
$q_{1}(\mathbf{x})=e^{-\varphi_{1}(\mathbf{x})}$
and $q_{0}(\mathbf{x})=e^{-\varphi_{0}(\mathbf{x})}$.
Then it follows from the proof of Proposition~7.1. in \cite{Saumard2014} that
\begin{align}
\nabla^{2}(-\log p_{T-t})(\mathbf{x})
&=-\text{Var}(\nabla\varphi_{0}(\mathbf{X}_{0})|\mathbf{X}_{0}+\mathbf{X}_{1}=\mathbf{x})
+\mathbb{E}[\nabla^{2}\varphi_{0}(\mathbf{X}_{0})|\mathbf{X}_{0}+\mathbf{X}_{1}=\mathbf{x}]\label{key:equality}
\\
&=-\text{Var}(\nabla\varphi_{1}(\mathbf{X}_{1})|\mathbf{X}_{0}+\mathbf{X}_{1}=\mathbf{x})
+\mathbb{E}[\nabla^{2}\varphi_{1}(\mathbf{X}_{1})|\mathbf{X}_{0}+\mathbf{X}_{1}=\mathbf{x}].\label{key:equality:2}
\end{align}
Note that 
it follows from the proof of Proposition~\ref{sec:thm1} that
\begin{equation}\label{lower:Hessian}
\nabla^{2}(-\log p_{T-t})(\mathbf{x})
\succeq
\left(\int_{0}^{T-t}e^{-2\int_{s}^{T-t}f(v)dv}(g(s))^{2}ds+e^{-\int_{0}^{T-t}f(s)ds}m_{0}^{-1}\right)^{-1}\cdot I_{d}.
\end{equation}
On the other hand, 
$\text{Var}(\nabla\varphi_{0}(\mathbf{X}_{0})|\mathbf{X}_{0}+\mathbf{X}_{1}=\mathbf{x})\succeq 0_{d\times d}$,
$\text{Var}(\nabla\varphi_{1}(\mathbf{X}_{1})|\mathbf{X}_{0}+\mathbf{X}_{1}=\mathbf{x})\succeq 0_{d\times d}$,
and $\nabla\log p_{0}$ is $L_{0}$-Lipschitz so that
$\nabla^{2}\varphi_{0}\preceq \left(e^{\int_{0}^{T-t}f(s)ds}\right)^{2}L_{0}\cdot I_{d}$,
and moreover, 
$\nabla^{2}\varphi_{1}\preceq \left(\int_{0}^{T-t}e^{-2\int_{s}^{T-t}f(v)dv}(g(s))^{2}ds\right)^{-1}\cdot I_{d}$.
Together with \eqref{key:equality} and \eqref{key:equality:2}, we have
\begin{equation}\label{upper:Hessian}
\nabla^{2}(-\log p_{T-t})(\mathbf{x})\preceq\min\left(\left(\int_{0}^{T-t}e^{-2\int_{s}^{T-t}f(v)dv}(g(s))^{2}ds\right)^{-1},
\left(e^{\int_{0}^{T-t}f(s)ds}\right)^{2}L_{0}\right)\cdot I_{d}.
\end{equation}
Hence, it follows from \eqref{lower:Hessian} and \eqref{upper:Hessian}
that $\log p_{T-t}$ is $L(T-t)$-Lipschitz, where $L(T-t)$ is given in \eqref{eq:Lt}.
This completes the proof.
\end{proof}

\subsection{Proof of Lemma~\ref{lem:second:term}}\label{app:2ndterm}
\begin{proof}
We can compute that for any $(k-1)\eta\leq t\leq k\eta$,
\begin{align*}
&\mathbf{z}_{t}-\mathbf{z}_{(k-1)\eta}
=\tilde{\mathbf{x}}_{t}-\tilde{\mathbf{x}}_{(k-1)\eta}
\\
&\qquad+\int_{(k-1)\eta}^{t}\left[f(T-s)(\mathbf{z}_{s}-\tilde{\mathbf{x}}_{s})+(g(T-s))^{2}\left(\nabla\log p_{T-s}(\mathbf{z}_{s})-\nabla\log p_{T-s}(\tilde{\mathbf{x}}_{s})\right)\right]ds,
\end{align*}
and therefore
\begin{align*}
&\left\Vert\mathbf{z}_{t}-\mathbf{z}_{(k-1)\eta}\right\Vert_{L_{2}}
\\
&\leq\left\Vert\tilde{\mathbf{x}}_{t}-\tilde{\mathbf{x}}_{(k-1)\eta}\right\Vert_{L_{2}}
+\int_{(k-1)\eta}^{t}\left[f(T-s)+(g(T-s))^{2}L(T-s)\right]\Vert\mathbf{z}_{s}-\tilde{\mathbf{x}}_{s}\Vert_{L_{2}}ds.
\end{align*}
We obtained in the proof of Proposition~\ref{thm:1} that
$$
\Vert\mathbf{z}_{t}-\tilde{\mathbf{x}}_{t}\Vert_{L_{2}}
\leq e^{-\frac{1}{2}\int_{0}^{t}m(T-s)ds}\Vert\mathbf{z}_{0}-\tilde{\mathbf{x}}_{0}\Vert_{L_{2}},
$$
and moreover, from the proof of Proposition~\ref{thm:1}, we have $\Vert\mathbf{z}_{0}-\tilde{\mathbf{x}}_{0}\Vert_{L_{2}}\leq e^{-\int_{0}^{T}f(s)ds}\Vert\mathbf{x}_{0}\Vert_{L_{2}}$
so that
$$
\Vert\mathbf{z}_{t}-\tilde{\mathbf{x}}_{t}\Vert_{L_{2}}
\leq e^{-\frac{1}{2}\int_{0}^{t}m(T-s)ds}e^{-\int_{0}^{T}f(s)ds}\Vert\mathbf{x}_{0}\Vert_{L_{2}}.
$$
Therefore, we have
\begin{align} \label{eq:c1Tsource}
\left\Vert\mathbf{z}_{t}-\mathbf{z}_{(k-1)\eta}\right\Vert_{L_{2}}
\leq\left\Vert\tilde{\mathbf{x}}_{t}-\tilde{\mathbf{x}}_{(k-1)\eta}\right\Vert_{L_{2}}
+c_{1}(T)\int_{(k-1)\eta}^{k\eta}\left[f(T-s)+(g(T-s))^{2}L(T-s)\right]ds,
\end{align}
where $c_{1}(T)$ bounds $ \sup_{0 \le t \le T}\Vert\mathbf{z}_{t}-\tilde{\mathbf{x}}_{t}\Vert_{L_{2}}$ and it is given in \eqref{c:1:defn}. 
Moreover, we recall that the backward process $(\tilde{\mathbf{x}}_{t})_{0\leq t\leq T}$
has the same distribution as the forward process $(\mathbf{x}_{T-t})_{0\leq t\leq T}$, 
so that
$\left\Vert\tilde{\mathbf{x}}_{t}-\tilde{\mathbf{x}}_{(k-1)\eta}\right\Vert_{L_{2}}
=\left\Vert\mathbf{x}_{T-t}-\mathbf{x}_{T-(k-1)\eta}\right\Vert_{L_{2}}$,
where $\mathbf{x}_{t}$ satisfies the SDE:
$d\mathbf{x}_{t}=-f(t)\mathbf{x}_{t}dt+g(t)d\mathbf{B}_{t}$,
with $\mathbf{x}_{0}\sim p_{0}$.
Therefore, we have
$$
\mathbf{x}_{T-(k-1)\eta}
-\mathbf{x}_{T-t}
=-\int_{T-t}^{T-(k-1)\eta}f(s)\mathbf{x}_{s}ds+\int_{T-t}^{T-(k-1)\eta}g(s)d\mathbf{B}_{s}.
$$
We can compute that
\begin{small}
\begin{align*}
\Vert\mathbf{x}_{T-(k-1)\eta}-\mathbf{x}_{T-t}\Vert_{L_{2}}
&\leq
\int_{T-t}^{T-(k-1)\eta}f(s)\Vert\mathbf{x}_{s}\Vert_{L_{2}}ds
+\left\Vert\int_{T-t}^{T-(k-1)\eta}g(s)d\mathbf{B}_{s}\right\Vert_{L_{2}}
\\
&\leq
\sup_{0\leq t\leq T}\Vert\mathbf{x}_{t}\Vert_{L_{2}}\int_{T-k\eta}^{T-(k-1)\eta}f(s)ds
+\left(\int_{T-k\eta}^{T-(k-1)\eta}(g(s))^{2}ds\right)^{1/2}\sqrt{d},
\end{align*}
\end{small}
where we used It\^{o}'s isometry. 
Therefore, we have
\begin{small}
\begin{align*}
\Vert\mathbf{x}_{T-(k-1)\eta}
-\mathbf{x}_{T-t}\Vert_{L_{2}}
\leq
c_{2}(T)\int_{T-k\eta}^{T-(k-1)\eta}f(s)ds
+\left(\int_{T-k\eta}^{T-(k-1)\eta}(g(s))^{2}ds\right)^{1/2}\sqrt{d},
\end{align*}
\end{small}
where we recall from \eqref{c:2:source} that $c_{2}(T)=\sup_{0\leq t\leq T}\Vert\mathbf{x}_{t}\Vert_{L_{2}}$
with an explicit formula given in \eqref{c:2:defn}. It follows that Lemma~\ref{lem:second:term} holds. 
\end{proof}



\section{Derivation of Results in Table~\ref{table:summary}}\label{appendix:examples}

In this section, we prove the results that are summarized in Table~\ref{table:summary}. We discuss variance exploding SDEs in Appendix~\ref{app:VESDE}, variance preserving SDEs in Appendix~\ref{app:VPSDE}, and constant coefficient SDEs in Appendix~\ref{app:ConstSDE}.

\subsection{Variance-Exploding SDEs} \label{app:VESDE}

In this section, we consider  variance-exploding SDEs 
with $f(t)\equiv 0$ in the forward process \eqref{OU:SDE}. We can immediately obtain the following corollary of Theorem~\ref{thm:discrete:2}.

\begin{corollary}\label{cor:VE:4}
Assume that Assumptions~\ref{assump:p0}, \ref{assump:M:1}, \ref{assump:M} and \ref{assump:stepsize} hold.
Then, we have
\begin{align}
&\mathcal{W}_{2}(\mathcal{L}(\mathbf{y}_{K}),p_{0})
\leq e^{-\int_{0}^{K\eta}c(t)dt} \Vert\mathbf{x}_{0}\Vert_{L_{2}} 
\nonumber
\\
&\quad
+\sum_{k=1}^{K}\prod_{j=k+1}^{K}\gamma_{j,\eta}\cdot\Bigg(
M_{1}\eta\left(1+2  \Vert\mathbf{x}_{0}\Vert_{L_{2}} 
+\sqrt{d}\left(\int_{0}^{T}(g(t))^{2}dt\right)^{1/2}\right)\int_{(k-1)\eta}^{k\eta}(g(T-t))^{2}dt
\nonumber
\\
&\qquad
+M\int_{(k-1)\eta}^{k\eta}(g(T-t))^{2}dt+\sqrt{\eta}h_{k,\eta}\left(\int_{(k-1)\eta}^{k\eta}(g(T-t))^{4}(L(T-t))^{2}dt\right)^{1/2}\Bigg). 
\end{align}
\end{corollary}


In the next few sections, we consider special functions $g$ in Corollary~\ref{cor:VE:4} and
derive the corresponding results in Table~\ref{table:summary}.

\subsubsection{Example 1: $f(t)\equiv 0$ and $g(t)=ae^{bt}$}
When $g(t)=ae^{bt}$ for some $a,b>0$, 
we can obtain the following result from Corollary~\ref{cor:VE:4}.

\begin{corollary}\label{cor:VE:5}
Let $g(t)=ae^{bt}$ for some $a,b>0$.
Then, we have $\mathcal{W}_{2}(\mathcal{L}(\mathbf{y}_{K}),p_{0})\leq\mathcal{O}(\epsilon)$
after $K=\mathcal{O}\left(\frac{d\log(d/\epsilon)}{\epsilon^{2}}\right)$ iterations
provided that $M\leq\frac{\epsilon}{\log(1/\epsilon)}$ and
$\eta\leq\frac{\epsilon^{2}}{d}$.
\end{corollary}

\begin{proof}
Let $g(t)=ae^{bt}$ for some $a,b>0$.
First, we can compute that
$$
(g(t))^{2}L(t)
=\min\left(\frac{(g(t))^{2}}{\int_{0}^{t}(g(s))^{2}ds},L_{0}(g(t))^{2}\right)
=\min\left(\frac{2be^{2bt}}{e^{2bt}-1},L_{0}\frac{a^{2}}{4b^{2}}(e^{2bt}-1)^{2}\right).
$$
If $e^{2bt}\geq 2$, then $e^{2bt}-1\geq\frac{1}{2}e^{2bt}$ and
$(g(t))^{2}L(t)\leq 4b$. 
On the other hand, if $e^{2bt}<2$, then $(g(t))^{2}L(t)\leq L_{0}\frac{a^{2}}{4b^{2}}$. 
Therefore, for any $0\leq t\leq T$,
$(g(t))^{2}L(t)
\leq\max\left(4b,\frac{L_{0}a^{2}}{4b^{2}}\right)$.
By the definition of $c(t)$ in \eqref{c:t:defn}, we can compute that
\begin{equation}\label{c:t:expression}
c(t)=\frac{m_{0}(g(t))^{2}}{1+m_{0}\int_{0}^{t}(g(s))^{2}ds}
=\frac{m_{0}a^{2}e^{2bt}}{1+m_{0}\frac{a^{2}}{2b}(e^{2bt}-1)}.
\end{equation}
This implies that
\begin{equation}\label{c:t:integral}
\int_{0}^{t}c(s)ds
=\int_{0}^{t}\frac{2bm_{0}a^{2}e^{2bs}ds}{2b-m_{0}a^{2}+m_{0}a^{2}e^{2bs}}
=\log\left(\frac{2b-m_{0}a^{2}+m_{0}a^{2}e^{2bt}}{2b}\right).
\end{equation}
By letting $t=T=K\eta$ in \eqref{c:t:integral} and using \eqref{eq:L2-x0}, we obtain
\begin{equation*}
e^{-\int_{0}^{K\eta}c(t)dt} \Vert\mathbf{x}_{0}\Vert_{L_{2}} 
\le \frac{2b}{2b-m_{0}a^{2}+m_{0}a^{2}e^{2bK\eta}}\left(\sqrt{2d/m_{0}}+\Vert\mathbf{x}_{\ast}\Vert\right).
\end{equation*}
Moreover,
\begin{equation*}
h_{k,\eta}
\leq
\left(\sqrt{2d/m_{0}}+\Vert\mathbf{x}_{\ast}\Vert\right)\max\left(4b,\frac{L_{0}a^{2}}{4b^{2}}\right)\eta
+\frac{a}{\sqrt{2b}}\left(e^{2b(T-(k-1)\eta)}-e^{2b(T-k\eta)}\right)^{1/2}\sqrt{d},
\end{equation*}
and for any $0\leq t\leq T$:
$$
\mu(t)\geq
\frac{m_{0}a^{2}e^{2bt}}{1+m_{0}\frac{a^{2}}{2b}(e^{2bt}-1)}
-\eta\max\left(16b^{2},\frac{L_{0}^{2}a^{4}}{16b^{4}}\right)\geq M_{1}\eta(g(t))^{2}=\eta M_{1}a^{2}e^{2bt},
$$
provided that
$$
\eta\leq\frac{1}{2}\frac{\min(m_{0}a^{2},2b)}{\max\left(16b^{2},\frac{L_{0}^{2}a^{4}}{16b^{4}}\right)}+\frac{1}{2}\frac{m_{0}a^{2}}{1+m_{0}\frac{a^{2}}{2b}(e^{2bT}-1)}.
$$
Furthermore,
$\mu(t)\leq\frac{m_{0}a^{2}e^{2bt}}{1+m_{0}\frac{a^{2}}{2b}(e^{2bt}-1)}
\leq\max(m_{0}a^{2},2b)$,
so that
$0\leq\gamma_{j,\eta}
\leq 1$ for every $j=1,2,\ldots,K$,
provided that
$\eta\leq
\min\left(\frac{1}{\max(m_{0}a^{2},2b)},\frac{1}{2}\frac{\min(m_{0}a^{2},2b)}{\max\left(16b^{2},\frac{L_{0}^{2}a^{4}}{16b^{4}}\right)}+\frac{1}{2}\frac{m_{0}a^{2}}{1+m_{0}\frac{a^{2}}{2b}(e^{2bT}-1)}\right)$.

Since $1-x\leq e^{-x}$ for any $0\leq x\leq 1$, 
we conclude that
\begin{align}
\prod_{j=k+1}^{K}\gamma_{j,\eta}&=\prod_{j=k+1}^{K}\left(1-\int_{(j-1)\eta}^{j\eta}\mu(T-t)dt+M_{1}\eta\int_{(j-1)\eta}^{j\eta}(g(T-t))^{2}dt\right)
\nonumber
\\
&\leq
\prod_{j=k+1}^{K}e^{-\int_{(j-1)\eta}^{j\eta}\mu(T-t)dt+M_{1}\eta\int_{(j-1)\eta}^{j\eta}(g(T-t))^{2}dt}
=e^{-\int_{k\eta}^{K\eta}\mu(T-t)dt+M_{1}\eta\int_{k\eta}^{K\eta}(g(T-t))^{2}dt}.\label{a:key:inequality}
\end{align}
Moreover,
\begin{align*}
\int_{k\eta}^{K\eta}
\mu(T-t)dt
&\geq
\int_{k\eta}^{K\eta}\frac{m_{0}a^{2}e^{2b(T-t)}dt}{1+m_{0}\frac{a^{2}}{2b}(e^{2b(T-t)}-1)}
-(K-k)\eta^{2}\max\left(16b^{2},\frac{L_{0}^{2}a^{4}}{16b^{4}}\right)
\nonumber
\\
&=\log\left(\frac{2b-m_{0}a^{2}+m_{0}a^{2}e^{2b(T-k\eta)}}{2b-m_{0}a^{2}+m_{0}a^{2}e^{2b(T-K\eta)}}\right)
-(K-k)\eta^{2}\max\left(16b^{2},\frac{L_{0}^{2}a^{4}}{16b^{4}}\right),
\end{align*}
and $M_{1}\eta\int_{k\eta}^{K\eta}(g(T-t))^{2}dt=M_{1}\eta\frac{a^{2}}{2b}\left(e^{2b(K-k)\eta}-1\right)$.
By applying Corollary~\ref{cor:VE:4} with $T=K\eta$, we conclude that
\begin{align*}
\mathcal{W}_{2}(\mathcal{L}(\mathbf{y}_{K}),p_{0})
&\leq\frac{2b\left(\sqrt{2d/m_{0}}+\Vert\mathbf{x}_{\ast}\Vert\right)}{2b-m_{0}a^{2}+m_{0}a^{2}e^{2bK\eta}}
+\sum_{k=1}^{K}
\frac{2be^{(K-k)\eta^{2}\max\left(16b^{2},\frac{L_{0}^{2}a^{4}}{16b^{4}}\right)
+M_{1}\eta\frac{a^{2}}{2b}(e^{2b(K-k)\eta}-1)}}{2b-m_{0}a^{2}+m_{0}a^{2}e^{2b(K-k)\eta}}\nonumber
\\
&\qquad
\cdot\Bigg(
\left(M+M_{1}\eta\left(1+2\left(\sqrt{2d/m_{0}}+\Vert\mathbf{x}_{\ast}\Vert\right)+\sqrt{d}\frac{a}{\sqrt{2b}}(e^{2bK\eta}-1)^{1/2}\right)\right)
\nonumber
\\
&\qquad\qquad\qquad\cdot\frac{a^{2}}{2b}\left(e^{2b(K-(k-1))\eta)}-e^{2b(K-k)\eta)}\right)
\nonumber
\\
&\qquad\qquad
+\eta\max\left(4b,\frac{L_{0}a^{2}}{4b^{2}}\right)\cdot\Bigg(\left(\sqrt{2d/m_{0}}+\Vert\mathbf{x}_{\ast}\Vert\right)\max\left(4b,\frac{L_{0}a^{2}}{4b^{2}}\right)\eta
\nonumber
\\
&\qquad\qquad\qquad\qquad
+\frac{a}{\sqrt{2b}}\left(e^{2b(K-(k-1))\eta)}-e^{2b(K-k)\eta}\right)^{1/2}\sqrt{d}\Bigg)\Bigg).
\end{align*}
By the mean-value theorem, we have
$e^{2b(K-(k-1))\eta)}-e^{2b(K-k)\eta)}
\leq 
2be^{2b(K-(k-1))\eta}\eta$,
which implies that
\begin{align*}
&\mathcal{W}_{2}(\mathcal{L}(\mathbf{y}_{K}),p_{0})
\leq
\mathcal{O}\left(\frac{\sqrt{d}}{e^{2bK\eta}}\right)
+\mathcal{O}\Bigg(e^{K\eta^{2}\max\left(16b^{2},\frac{L_{0}^{2}a^{4}}{16b^{4}}\right)+M_{1}\eta\frac{a^{2}}{2b}e^{2bK\eta}}
\nonumber
\\
&\qquad\cdot
\sum_{k=1}^{K}
\frac{1}{e^{2b(K-k)\eta}}
\cdot\Bigg(
\left(M+M_{1}\eta\sqrt{d}e^{bK\eta}\right)e^{2b(K-(k-1))\eta}\eta+\eta e^{b(K-(k-1))\eta}\sqrt{\eta}\sqrt{d}\Bigg)\Bigg)
\nonumber
\\
&\leq
\mathcal{O}\left(\frac{\sqrt{d}}{e^{2bK\eta}}\right)
+\mathcal{O}\left(e^{K\eta^{2}\max\left(16b^{2},\frac{L_{0}^{2}a^{4}}{16b^{4}}\right)}
\cdot\Bigg(
\left(M+M_{1}\eta\sqrt{d}e^{bK\eta}\right)K\eta+\sqrt{\eta}\sqrt{d}\Bigg)\right)
\leq\mathcal{O}(\epsilon),
\end{align*}
provided that
$K\eta=\frac{\log(\sqrt{d}/\epsilon)}{2b}$,
$M\leq\frac{\epsilon}{\log(1/\epsilon)}$,
and $\eta\leq\frac{\epsilon^{2}}{d}$,
which implies that
$K\geq\mathcal{O}\left(\frac{d\log(d/\epsilon)}{\epsilon^{2}}\right)$.
This completes the proof.
\end{proof}

\begin{remark}\label{rk:VE:5}
In Corollary~\ref{cor:VE:5}, we can also spell out the dependence of
iteration complexity on $M_{1}$ from Assumption~\ref{assump:M:1}. It follows from the proof of Corollary~\ref{cor:VE:5} 
that $\mathcal{W}_{2}(\mathcal{L}(\mathbf{y}_{K}),p_{0})\leq\mathcal{O}(\epsilon)$ provided that
$K\eta=\frac{\log(\sqrt{d}/\epsilon)}{2b}$,
$M\leq\frac{\epsilon}{\log(1/\epsilon)}$,
and $\eta\leq\frac{\epsilon^{2}}{d}$,
with the additional constraint that $\eta\leq\frac{M}{M_{1}\sqrt{d}e^{bK\eta}}\leq\frac{\epsilon^{3/2}}{M_{1}\log(1/\epsilon)d^{3/4}}$, which implies that $K\geq\mathcal{O}\left(\log\left(\frac{d}{\epsilon}\right)\max\left\{\frac{d}{\epsilon^{2}},\frac{M_{1}\log(1/\epsilon)d^{3/4}}{\epsilon^{3/2}}\right\}\right)$.
\end{remark}

\subsubsection{Example 2: $f(t)\equiv 0$ and $g(t)\equiv a$}

When $g(t)\equiv a$ for some $a>0$,
we can obtain the following result from Corollary~\ref{cor:VE:4}.

\begin{corollary}\label{cor:VE:6}
Let $g(t)\equiv a$ for some $a>0$.
Then, we have $\mathcal{W}_{2}(\mathcal{L}(\mathbf{y}_{K}),p_{0})\leq\mathcal{O}(\epsilon)$
after $K=\mathcal{O}\left(\frac{d^{3/2}\log(d/\epsilon)}{\epsilon^{3}}\right)$ iterations
provided that $M\leq\mathcal{O}\left(\frac{\epsilon}{\sqrt{\log(d/\epsilon)}}\right)$
and $\eta\leq\mathcal{O}\left(\frac{\epsilon^{2}}{d\log(d/\epsilon)}\right)$.
\end{corollary}


\begin{proof}
When $g(t)\equiv a$, by applying Corollary~\ref{cor:VE:4} with $T=K\eta$ and \eqref{eq:L2-x0}, we have
\begin{small}
\begin{align*}
\mathcal{W}_{2}(\mathcal{L}(\mathbf{y}_{K}),p_{0})
&\leq\frac{\sqrt{2d/m_{0}}+\Vert\mathbf{x}_{\ast}\Vert}{1+m_{0}a^{2}K\eta}
+\sum_{k=1}^{K}
\frac{1}{1+m_{0}a^{2}(K-k)\eta}e^{(K-k)\eta^{2}L_{0}^{2}a^{4}+M_{1}\eta(K-k)\eta a^{2}}
\nonumber
\\
&\qquad\cdot
\Bigg(\left(M+M_{1}\eta\left(1+2\left(\sqrt{2d/m_{0}}+\Vert\mathbf{x}_{\ast}\Vert\right)+\sqrt{d}a\sqrt{K\eta}\right)\right)a^{2}\eta
\nonumber
\\
&\qquad\qquad\qquad\qquad
+\eta L_{0}a^{2}
\cdot\left(\left(\sqrt{2d/m_{0}}+\Vert\mathbf{x}_{\ast}\Vert\right)L_{0}a^{2}\eta+a\sqrt{\eta}\sqrt{d}\right)\Bigg).
\end{align*}
\end{small}
It is easy to verify that
\begin{small}
\begin{align*}
\sum_{k=1}^{K}
\frac{e^{(K-k)\eta^{2}L_{0}^{2}a^{4}+M_{1}\eta^{2}(K-k)a^{2}}}{1+m_{0}a^{2}(K-k)\eta}
&\le \left(1+\frac{\log((K-1)\eta m_{0}a^{2})}{\eta m_{0}a^{2}}\right)e^{K\eta^{2}L_{0}^{2}a^{4}+M_{1}\eta^{2}Ka^{2}}.
\end{align*}
\end{small}
Therefore,
\begin{small}
\begin{align*}
\mathcal{W}_{2}(\mathcal{L}(\mathbf{y}_{K}),p_{0})
&\leq
\mathcal{O}\Bigg(\frac{\sqrt{d}}{1+m_{0}a^{2}K\eta}
+\left(1+\frac{\log((K-1)\eta m_{0}a^{2})}{\eta m_{0}a^{2}}\right)e^{K\eta^{2}L_{0}^{2}a^{4}+M_{1}\eta^{2}Ka^{2}}
\nonumber
\\
&\qquad\cdot
\left(\left(M+M_{1}\eta\sqrt{d}\sqrt{K\eta}\right)a^{2}\eta+\eta L_{0}a^{2}
\cdot\left(\sqrt{d}L_{0}a^{2}\eta+a\sqrt{\eta}\sqrt{d}\right)\right)\Bigg).
\end{align*}
\end{small}
Hence, we conclude that $\mathcal{W}_{2}(\mathcal{L}(\mathbf{y}_{K}),p_{0})\leq\mathcal{O}(\epsilon)$ provided
that
$K\eta=\frac{\sqrt{d}}{\epsilon}$, 
$M=\mathcal{O}\left(\sqrt{\eta d}\right)$,
and $\eta\leq\mathcal{O}\left(\frac{\epsilon^{2}}{d\log(d/\epsilon)}\right)$,
so that
$M=\mathcal{O}(\sqrt{\eta d})
\leq\mathcal{O}\left(\frac{\epsilon}{\sqrt{\log(d/\epsilon)}}\right)$,
which implies that
$K\geq\mathcal{O}\left(\frac{d^{3/2}\log(d/\epsilon)}{\epsilon^{3}}\right)$.
This completes the proof.
\end{proof}

\subsubsection{Example 3: $f(t)\equiv 0$ and $g(t)=\sqrt{2at}$}

When $g(t)= \sqrt{2at}$ for some $a>0$,
we can obtain the following result from Corollary~\ref{cor:VE:4}.

\begin{corollary}\label{cor:VE:7}
Let $g(t)=\sqrt{2at}$ for some $a>0$.
Then, we have $\mathcal{W}_{2}(\mathcal{L}(\mathbf{y}_{K}),p_{0})\leq\mathcal{O}(\epsilon)$
after $K=\mathcal{O}\left(\frac{d^{5/4}}{\epsilon^{5/2}}\right)$ iterations
provided that $M\leq\epsilon^{3/2}$ and $\eta\leq\frac{\epsilon^{2}}{d}$.
\end{corollary}


\begin{proof}
When $g(t)=\sqrt{2at}$ for some $a>0$, 
we have
$$
(g(t))^{2}L(t)
=\min\left(\frac{(g(t))^{2}}{\int_{0}^{t}(g(s))^{2}ds},L_{0}(g(t))^{2}\right)
=\min\left(\frac{2}{t},2tL_{0}a\right)\leq 2\sqrt{aL_{0}}.
$$
We can also compute from \eqref{c:t:defn} that
\begin{equation}\label{c:t:integral:2}
\int_{0}^{t}c(s)ds
=\int_{0}^{t}\frac{2m_{0}asds}{1+m_{0}as^{2}}
=\log\left(1+m_{0}at^{2}\right).
\end{equation}
By letting $t=T=K\eta$ in \eqref{c:t:integral:2} and using \eqref{eq:L2-x0}, we obtain
$$
e^{-\int_{0}^{K\eta}c(t)dt} \Vert\mathbf{x}_{0}\Vert_{L_{2}} 
\le \frac{\sqrt{2d/m_{0}}+\Vert\mathbf{x}_{\ast}\Vert}{1+am_{0}K^{2}\eta^{2}}.
$$
Moreover,
$$
h_{k,\eta}
\leq
2\left(\sqrt{2d/m_{0}}+\Vert\mathbf{x}_{\ast}\Vert\right)\sqrt{aL_{0}}\eta
+\left((T-(k-1)\eta)^{2}-(T-k\eta)^{2}\right)^{1/2}\sqrt{d},
$$
and for any $0\leq t\leq T$:
$$
\mu(t)=
\frac{2am_{0}t}{1+am_{0}t^{2}}
-4\eta\min\left(\frac{1}{t^{2}},t^{2}a^{2}L_{0}^{2}\right)\geq M_{1}\eta(g(t))^{2}=2aM_{1}\eta t,
$$
provided that
$\eta\leq\min\left(\frac{m_{0}}{4L_{0}(a^{2}L_{0}^{2}+am_{0})},\frac{m_{0}}{4a^{2}L_{0}^{2}(L_{0}+m_{0})},\frac{m_{0}}{2M_{1}(1+am_{0}T^{2})}\right)$.
Additionally,
$$
\mu(t)\leq\frac{2am_{0}t}{1+am_{0}t^{2}}\leq\sqrt{am_{0}},
$$
so that
$0\leq\gamma_{j,\eta}
\leq 1$ for any $j=1,2,\ldots,K$,
provided that
$$\eta\leq
\min\left(\frac{m_{0}}{4L_{0}(a^{2}L_{0}^{2}+am_{0})},\frac{m_{0}}{4a^{2}L_{0}^{2}(L_{0}+m_{0})},\frac{m_{0}}{2M_{1}(1+am_{0}T^{2})},\frac{1}{\sqrt{am_{0}}}\right).
$$
We recall from \eqref{a:key:inequality} that
\begin{small}
\begin{align*}
\prod_{j=k+1}^{K}\gamma_{j,\eta}
\leq
e^{-\int_{k\eta}^{K\eta}\mu(T-t)dt+M_{1}\eta\int_{k\eta}^{K\eta}(g(T-t))^{2}dt}.
\end{align*}
\end{small}
Moreover, one can verify that
$$\int_{k\eta}^{K\eta}
\mu(T-t)dt
\geq
\log\left(1+am_{0}(K-k)^{2}\eta^{2}\right)
-4(K-k)\eta^{2}aL_{0},$$
and
$$
M_{1}\eta\int_{k\eta}^{K\eta}(g(T-t))^{2}dt
=M_{1}\eta\int_{k\eta}^{K\eta}2a(K\eta-t)dt
=M_{1}\eta(K-k)^{2}\eta^{2}.
$$
By applying Corollary~\ref{cor:VE:4} with $T=K\eta$ and \eqref{eq:L2-x0}, we conclude that
\begin{small}
\begin{align*}
&\mathcal{W}_{2}(\mathcal{L}(\mathbf{y}_{K}),p_{0})
\leq\frac{\sqrt{2d/m_{0}}+\Vert\mathbf{x}_{\ast}\Vert}{1+am_{0}K^{2}\eta^{2}}
+\sum_{k=1}^{K}\frac{1}{1+am_{0}(K-k)^{2}\eta^{2}}e^{4K\eta^{2}aL_{0}+M_{1}\eta^{3}K^{2}}
\nonumber
\\
&\qquad\qquad\cdot
\Bigg(2\left(M+M_{1}\eta\left(1+2\left(\sqrt{2d/m_{0}}+\Vert\mathbf{x}_{\ast}\Vert\right)+\sqrt{d}\sqrt{2a}K\eta\right)\right)(K-k+1)\eta^{2}
\nonumber
\\
&\qquad\qquad\qquad
+2\sqrt{aL_{0}}\eta\left(2\left(\sqrt{2d/m_{0}}+\Vert\mathbf{x}_{\ast}\Vert\right)\sqrt{aL_{0}}\eta
+(2(K-k+1))^{1/2}\eta\sqrt{d}\right)\Bigg).
\end{align*}
\end{small}
This implies that
\begin{small}
\begin{equation*}
\mathcal{W}_{2}(\mathcal{L}(\mathbf{y}_{K}),p_{0})
\leq
\mathcal{O}\left(\frac{\sqrt{d}}{K^{2}\eta^{2}}+e^{\mathcal{O}((K\eta)^{2}\eta)}
\left(K\eta(M+M_{1}\sqrt{d}K\eta^{2})+\sqrt{\eta}\sqrt{d}\right)\right)
\leq\mathcal{O}(\epsilon),
\end{equation*}
\end{small}
provided that
$K\eta=\frac{d^{1/4}}{\sqrt{\epsilon}}$,
$M\leq\epsilon^{3/2}$,
and $\eta\leq\frac{\epsilon^{2}}{d}$,
so that
$K\geq\mathcal{O}\left(\frac{d^{5/4}}{\epsilon^{5/2}}\right)$.
This completes the proof.
\end{proof}

\subsubsection{Example 4: $f(t)\equiv 0$ and $g(t)=(b+at)^{c}$}

When $g(t)=(b+at)^{c}$ for some $a,b,c>0$,
we obtain the following result from Corollary~\ref{cor:VE:4}.

\begin{corollary}\label{cor:VE:8}
Let $g(t)=(b+at)^{c}$ for some $a,b,c>0$.
Then, we have $\mathcal{W}_{2}(\mathcal{L}(\mathbf{y}_{K}),p_{0})\leq\mathcal{O}(\epsilon)$
after $K=\mathcal{O}\left(\frac{d^{\frac{1}{2(2c+1)}+1}}{\epsilon^{\frac{1}{2c+1}+2}}\right)$ iterations
provided that $M\leq\epsilon^{1+\frac{2c}{2c+1}}$ and $\eta\leq\frac{\epsilon^{2}}{d}$.
\end{corollary}


\begin{proof}
When $g(t)=(b+at)^{c}$ for some $a,b,c>0$, we can compute that
\begin{align}\label{follow:step:1}
(g(t))^{2}L(t)
=\min\left(\frac{(b+at)^{c}}{\frac{1}{a(2c+1)}((b+at)^{2c+1}-b^{2c+1})},L_{0}(b+at)^{2c}\right).
\end{align}
It is straightforward to verify that
$(g(t))^{2}L(t)
\leq\max\left(\frac{a(2c+1)}{(1-\frac{1}{2^{2c+1}})b},L_{0}(2b)^{2c}\right)$.
By \eqref{c:t:defn}, we have
$$
e^{-\int_{0}^{K\eta}c(t)dt}\left(\sqrt{2d/m_{0}}+\Vert\mathbf{x}_{\ast}\Vert\right)
=\frac{\sqrt{2d/m_{0}}+\Vert\mathbf{x}_{\ast}\Vert}{1+\frac{m_{0}}{a(2c+1)}((b+aK\eta)^{2c+1}-b^{2c+1})}.
$$
Also,
\begin{align*}
h_{k,\eta}
&\leq
\left(\sqrt{2d/m_{0}}+\Vert\mathbf{x}_{\ast}\Vert\right)\max\left(\frac{a(2c+1)}{(1-\frac{1}{2^{2c+1}})b},L_{0}(2b)^{2c}\right)\eta
\\
&\qquad\qquad\qquad\qquad\qquad
+\left(\frac{(b+a(T-(k-1)\eta))^{2c}}{a}\right)^{1/2}\sqrt{\eta}\sqrt{d},
\end{align*}
and for any $0\leq t\leq T$:
\begin{small}
\begin{align*}
\mu(t)
&=\frac{(b+at)^{2c}}{\frac{1}{m_{0}}+\frac{1}{a(2c+1)}((b+at)^{2c+1}-b^{2c+1})}
\nonumber
\\
&\quad
-\eta\min\left(\frac{(b+at)^{2c}}{\frac{1}{a^{2}(2c+1)^{2}}((b+at)^{2c+1}-b^{2c+1})^{2}},L_{0}^{2}(b+at)^{4c}\right)
\geq M_{1}\eta(g(t))^{2}=M_{1}\eta(b+at)^{2c},
\end{align*}
\end{small}
provided that
$\eta\leq\frac{1}{2}$, $\eta\leq\frac{\frac{b^{2c}}{\frac{1}{m_{0}}+\frac{1}{\sqrt{m_{0}}}+1}}{2L_{0}^{2}(a(2c+1)(\frac{1}{\sqrt{m_{0}}}+1)+b^{2c+1})^{\frac{4c}{2c+1}}}$ and $\eta\leq\frac{1}{\frac{2}{m_{0}}+\frac{2}{a(2c+1)}((b+aT)^{2c+1}-b^{2c+1})}$.
In addition,
$$
\mu(t)\leq\frac{(b+at)^{2c}}{\frac{1}{m_{0}}+\frac{1}{a(2c+1)}((b+at)^{2c+1}-b^{2c+1})}
\leq\max\left(\frac{a(2c+1)}{b},m_{0}b^{2c}\right),$$
so that
$0\leq \gamma_{j,\eta}
\leq 1$ for any $j=1,2,\ldots,K$.
We recall from \eqref{a:key:inequality} that
\begin{small}
\begin{align*}
\prod_{j=k+1}^{K}\gamma_{j,\eta}
\leq
e^{-\int_{k\eta}^{K\eta}\mu(T-t)dt+M_{1}\eta\int_{k\eta}^{K\eta}(g(T-t))^{2}dt}.
\end{align*}
\end{small}
Moreover,
\begin{align*}
\int_{k\eta}^{K\eta}
\mu(T-t)dt
&\geq
\log\left(1+\frac{m_{0}}{a(2c+1)}\left((b+a(K-k)\eta)^{2c+1}-b^{2c+1}\right)\right)
\\
&\qquad\qquad-(K-k)\eta^{2}\max\left(\frac{a^{2}(2c+1)^{2}}{(1-\frac{1}{2^{2c+1}})^{2}b^{2}},L_{0}^{2}(2b)^{4c}\right),    
\end{align*}
and we can compute that
$$
M_{1}\eta\int_{k\eta}^{K\eta}(g(T-t))^{2}dt
=\frac{M_{1}\eta}{a(2c+1)}\left((b+a(K-k)\eta)^{2c+1}-b^{2c+1}\right).
$$
By applying Corollary~\ref{cor:VE:4} with $T=K\eta$ and \eqref{eq:L2-x0}, we conclude that
\begin{align*}
&\mathcal{W}_{2}(\mathcal{L}(\mathbf{y}_{K}),p_{0})
\leq
\frac{\sqrt{2d/m_{0}}+\Vert\mathbf{x}_{\ast}\Vert}{1+\frac{m_{0}}{a(2c+1)}((b+aK\eta)^{2c+1}-b^{2c+1})}
\nonumber
\\
&\quad
+\sum_{k=1}^{K}\frac{e^{(K-k)\eta^{2}\max\left(\frac{a^{2}(2c+1)^{2}}{(1-\frac{1}{2^{2c+1}})^{2}b^{2}},L_{0}^{2}(2b)^{4c}\right)+\frac{M_{1}\eta}{a(2c+1)}((b+a(K-k)\eta)^{2c+1}-b^{2c+1})}}{1+\frac{m_{0}}{a(2c+1)}\left((b+a(K-k)\eta)^{2c+1}-b^{2c+1}\right)}
\nonumber
\\
&\qquad\cdot
\Bigg(\left(M+M_{1}\eta\left(1+2\left(\sqrt{2d/m_{0}}+\Vert\mathbf{x}_{\ast}\Vert\right)+\sqrt{d}\left(\frac{(b+aK\eta)^{2c+1}-b^{2c+1}}{a(2c+1)}\right)^{1/2}\right)\right)
\nonumber
\\
&\qquad\qquad\qquad\qquad\cdot
\frac{(b+a(K-k+1)\eta)^{2c+1}-(b+a(K-k)\eta)^{2c+1}}{a(2c+1)}
\nonumber
\\
&\quad\qquad\qquad\qquad+\eta\max\left(\frac{a(2c+1)}{(1-\frac{1}{2^{2c+1}})b},L_{0}(2b)^{2c}\right)
\nonumber
\\
&\qquad\qquad\qquad
\cdot\Bigg(\left(\sqrt{2d/m_{0}}+\Vert\mathbf{x}_{\ast}\Vert\right)\max\left(\frac{a(2c+1)}{(1-\frac{1}{2^{2c+1}})b},L_{0}(2b)^{2c}\right)\eta
\nonumber
\\
&\qquad\qquad\qquad\qquad\qquad\qquad\qquad
+\left(\frac{(b+a((K-(k-1))\eta))^{2c}}{a}\right)^{1/2}\sqrt{\eta}\sqrt{d}\Bigg)\Bigg).
\end{align*}
This implies that
$$
\mathcal{W}_{2}(\mathcal{L}(\mathbf{y}_{K}),p_{0})
\leq
\mathcal{O}\left(\frac{\sqrt{d}}{(K\eta)^{2c+1}}+e^{\mathcal{O}((K\eta)\eta+(K\eta)^{2c+1}\eta)}
\left((K\eta)^{2c}M+\sqrt{\eta}\sqrt{d}\right)\right)
\leq\mathcal{O}(\epsilon),$$ 
provided that
$K\eta=\frac{d^{\frac{1}{2(2c+1)}}}{\epsilon^{\frac{1}{2c+1}}}$,
$M\leq\epsilon^{1+\frac{2c}{2c+1}}$,
and $\eta\leq\frac{\epsilon^{2}}{d}$,
so that
$K\geq\mathcal{O}\left(\frac{d^{\frac{1}{2(2c+1)}+1}}{\epsilon^{\frac{1}{2c+1}+2}}\right)$.
This completes the proof.
\end{proof}

\subsection{Variance-Preserving SDEs} \label{app:VPSDE}

In this section, we consider  Variance-Preserving SDEs with $f(t)=\frac{1}{2}\beta(t)$ and $g(t)=\sqrt{\beta(t)}$ in the forward process \eqref{OU:SDE}.
 We can obtain the following corollary of Theorem~\ref{thm:discrete:2}. 



\begin{corollary}\label{cor:VP:4}
Under the assumptions of Theorem~\ref{thm:discrete:2}, we have
\begin{align}
&\mathcal{W}_{2}(\mathcal{L}(\mathbf{y}_{K}),p_{0})
\leq\frac{\Vert\mathbf{x}_{0}\Vert_{L_{2}}}{m_{0}e^{\int_{0}^{K\eta}\beta(s)ds}+1-m_{0}}
\nonumber
\\
&\quad+\sum_{k=1}^{K}
\frac{e^{\int_{k\eta}^{K\eta}\frac{1}{2}\beta(K\eta-t)dt+\int_{k\eta}^{K\eta}\frac{\eta}{4}(\beta(K\eta-t))^{2}dt+\int_{k\eta}^{K\eta}4\eta\max(1,L_{0}^{2})(\beta(K\eta-t))^{2}dt+M_{1}\eta\int_{k\eta}^{K\eta}\beta(K\eta-t)dt}}{m_{0}e^{\int_{0}^{(K-k)\eta}\beta(s)ds}+1-m_{0}}
\nonumber
\\
&\qquad\qquad
\cdot
\Bigg(M_{1}\eta(1+2\Vert\mathbf{x}_{0}\Vert_{L_{2}}+\sqrt{d})\int_{(k-1)\eta}^{k\eta}\beta(K\eta-t)dt
\nonumber
\\
&\quad
+M\int_{(k-1)\eta}^{k\eta}\beta(K\eta-t)dt+\sqrt{\eta}
\left(\frac{1}{2}+2\max(1,L_{0})\right)\left(\int_{(k-1)\eta}^{k\eta}(\beta(K\eta-t))^{2}dt\right)^{1/2}
\nonumber
\\
&\qquad
\cdot
\Bigg(
e^{-\int_{0}^{K\eta}\frac{1}{2}\beta(s)ds}\Vert\mathbf{x}_{0}\Vert_{L_{2}}
\left(\frac{1}{2}+2\max(1,L_{0})\right)
\int_{(k-1)\eta}^{k\eta}\beta(K\eta-s)ds
\nonumber
\\
&\qquad
+\left(\Vert\mathbf{x}_{0}\Vert_{L_{2}}^{2}+d\right)^{1/2}\int_{(K-k)\eta}^{(K-(k-1))\eta}\frac{1}{2}\beta(s)ds
+\left(\int_{(K-k)\eta}^{(K-(k-1))\eta}\beta(s)ds\right)^{1/2}\sqrt{d}\Bigg)\Bigg).
\end{align}
\end{corollary}

\begin{proof}
We apply Theorem~\ref{thm:discrete:2} applied to the variance-preserving SDE ($f(t)=\frac{1}{2}\beta(t)$ and $g(t)=\sqrt{\beta(t)}$).
First, we can compute that
$$L(T-t)
=\min\left(\frac{1}{1-e^{-\int_{0}^{T-t}\beta(s)ds}},e^{\int_{0}^{T-t}\beta(s)ds}L_{0}\right).$$
If $e^{\int_{0}^{T-t}\beta(s)ds}\geq 2$, then $\frac{1}{1-e^{-\int_{0}^{T-t}\beta(s)ds}}\leq 2$
and otherwise $e^{\int_{0}^{T-t}\beta(s)ds}L_{0}\leq 2L_{0}$. 
Therefore, for any $0\leq t\leq T$, 
$L(T-t)\leq 2\max(1,L_{0})$.
By applying Theorem~\ref{thm:discrete:2}, we have
\begin{small}
\begin{align*}
&\mathcal{W}_{2}(\mathcal{L}(\mathbf{y}_{K}),p_{0})
\leq e^{-\int_{0}^{K\eta}c(t)dt}\Vert\mathbf{x}_{0}\Vert_{L_{2}}
+\sum_{k=1}^{K}
\prod_{j=k+1}^{K}\gamma_{j,\eta}\nonumber
\\
&\quad\cdot\Bigg(M_{1}\eta\left(1+\Vert\mathbf{x}_{0}\Vert_{L_{2}}
+c_{2}(T)\right)\int_{(k-1)\eta}^{k\eta}(g(T-t))^{2}dt
\nonumber
\\
&+M\int_{(k-1)\eta}^{k\eta}(g(T-t))^{2}dt+\sqrt{\eta}h_{k,\eta}\left(\int_{(k-1)\eta}^{k\eta}[f(T-t)+(g(T-t))^{2}L(T-t)]^{2}dt\right)^{1/2}\Bigg),
\end{align*}
\end{small}
where, the definition of $\gamma_{j,\eta}$ in \eqref{gamma:k:defn} depends on $\mu(T-t)$, such that for any $0\leq t\leq T$:
\begin{small}
\begin{align*}
\mu(T-t)&=\frac{(g(T-t))^{2}}{\frac{1}{m_{0}}e^{-2\int_{0}^{T-t}f(s)ds}+\int_{0}^{T-t}e^{-2\int_{s}^{T-t}f(v)dv}(g(s))^{2}ds}-f(T-t)
\nonumber
\\
&\qquad\qquad\qquad\qquad\qquad
-\eta(f(T-t))^{2}-\eta(g(T-t))^{4}(L(T-t))^{2},
\nonumber
\\
&\geq\frac{m_{0}\beta(T-t)}{e^{-\int_{0}^{T-t}\beta(s)ds}+m_{0}(1-e^{-\int_{0}^{T-t}\beta(s)ds})}
\\
&\qquad\qquad\qquad\qquad\qquad
-\frac{1}{2}\beta(T-t)
-\frac{\eta}{4}(\beta(T-t))^{2}-4\eta(\beta(T-t))^{2}\max(1,L_{0}^{2}),
\end{align*}
\end{small}
where we assume $\eta$ is sufficiently small such that
$0\leq \gamma_{j,\eta}\leq 1$,
for every $j=1,2,\ldots,K$.
One can verify that
\begin{small}
\begin{align*}
h_{k,\eta}&\leq
e^{-\int_{0}^{T}\frac{1}{2}\beta(s)ds}\Vert\mathbf{x}_{0}\Vert_{L_{2}}
\left(\frac{1}{2}+2\max(1,L_{0})\right)
\int_{(k-1)\eta}^{k\eta}\beta(T-s)ds
\nonumber
\\
&\qquad
+\left(\Vert\mathbf{x}_{0}\Vert_{L_{2}}^{2}+d\right)^{1/2}\int_{T-k\eta}^{T-(k-1)\eta}\frac{1}{2}\beta(s)ds
+\left(\int_{T-k\eta}^{T-(k-1)\eta}\beta(s)ds\right)^{1/2}\sqrt{d}.
\end{align*}
\end{small}
Next, for VP-SDE, 
we have $f(t)=\frac{1}{2}\beta(t)$ and $g(t)=\sqrt{\beta(t)}$ so that we can compute:
$$
c(t)=\frac{m_{0}\beta(t)}{e^{-\int_{0}^{t}\beta(s)ds}+m_{0}\int_{0}^{t}e^{-\int_{s}^{t}\beta(v)dv}\beta(s)ds}
=\frac{m_{0}\beta(t)}{e^{-\int_{0}^{t}\beta(s)ds}+m_{0}(1-e^{-\int_{0}^{t}\beta(s)ds})}.
$$
It follows that
$$
\int_{0}^{T}c(t)dt
=\int_{0}^{\int_{0}^{T}\beta(s)ds}\frac{m_{0}dx}{m_{0}+(1-m_{0})e^{-x}}
=\log\left(m_{0}e^{\int_{0}^{T}\beta(s)ds}+1-m_{0}\right).
$$
Hence, we obtain
$$
e^{-\int_{0}^{T}c(t)dt}\Vert\mathbf{x}_{0}\Vert_{L_{2}}
=\frac{\Vert\mathbf{x}_{0}\Vert_{L_{2}}}{m_{0}e^{\int_{0}^{T}\beta(s)ds}+1-m_{0}}.
$$
By using $T=K\eta$, we complete the proof.
\end{proof}



Next, we prove Proposition~\ref{prop:VP}. 

\subsubsection{Proof of Proof of Proposition~\ref{prop:VP}}

\begin{proof}
It follows from Corollary~\ref{cor:VP:4}  and \eqref{eq:L2-x0} that
\begin{small}
\begin{align*}
&\mathcal{W}_{2}(\mathcal{L}(\mathbf{y}_{K}),p_{0})
\leq\frac{\sqrt{2d/m_{0}}+\Vert\mathbf{x}_{\ast}\Vert}{m_{0}e^{\int_{0}^{K\eta}\beta(s)ds}+1-m_{0}}
\\
&\qquad\qquad\qquad+\sum_{k=1}^{K}
\frac{e^{(1+2M_{1}\eta)\int_{0}^{(K-k)\eta}\frac{1}{2}\beta(t)dt+(\frac{\eta}{4}+4\eta\max(1,L_{0}^{2}))\int_{0}^{(K-k)\eta}(\beta(t))^{2}dt}}{m_{0}e^{\int_{0}^{(K-k)\eta}\beta(s)ds}+1-m_{0}}
\nonumber
\\
&\qquad\qquad
\cdot
\Bigg(\left(M+M_{1}\eta\left(1+2\left(\sqrt{2d/m_{0}}+\Vert\mathbf{x}_{\ast}\Vert\right)+\sqrt{d}\right)\right)\int_{(K-k)\eta}^{(K-k+1)\eta}\beta(t)dt
\nonumber
\\
&\quad
+\sqrt{\eta}
\left(\frac{1}{2}+2\max(1,L_{0})\right)\left(\int_{(K-k)\eta}^{(K-k+1)\eta}(\beta(t))^{2}dt\right)^{1/2}
\nonumber
\\
&\qquad
\cdot
\Bigg(
e^{-\int_{0}^{K\eta}\frac{1}{2}\beta(s)ds}\left(\sqrt{2d/m_{0}}+\Vert\mathbf{x}_{\ast}\Vert\right)
\left(\frac{1}{2}+2\max(1,L_{0})\right)
\int_{(K-k)\eta}^{(K-k+1)\eta}\beta(s)ds
\nonumber
\\
&\qquad
+\left(\left(\sqrt{2d/m_{0}}+\Vert\mathbf{x}_{\ast}\Vert\right)^{2}+d\right)^{1/2}\int_{(K-k)\eta}^{(K-(k-1))\eta}\frac{1}{2}\beta(s)ds
+\left(\int_{(K-k)\eta}^{(K-(k-1))\eta}\beta(s)ds\right)^{1/2}\sqrt{d}\Bigg)\Bigg).
\end{align*}
\end{small}
Since $\beta(t)$ is increasing in $t$, we can compute 
\begin{small}
\begin{align*}
\mathcal{W}_{2}(\mathcal{L}(\mathbf{y}_{K}),p_{0})
&\leq
\mathcal{O}\Bigg(\frac{\sqrt{d}}{e^{\int_{0}^{K\eta}\beta(s)ds}}
+e^{M_{1}\eta\int_{0}^{K\eta}\beta(t)dt+(\frac{\eta}{4}+4\eta\max(1,L_{0}^{2}))\beta(K\eta)\int_{0}^{K\eta}\beta(t)dt}
\nonumber
\\
&\qquad
\cdot
\Bigg(\left(M+M_{1}\eta\sqrt{d}\right)\beta(K\eta)
+\beta(K\eta)
\cdot
\Bigg(\sqrt{d}\eta\beta(K\eta)
+\sqrt{\beta(K\eta)}\sqrt{\eta}\sqrt{d}\Bigg)\Bigg)\Bigg).
\end{align*}
\end{small}
Since $\beta(t)\leq c_{1}\left(\int_{0}^{t}\beta(s)ds\right)^{c_3}+c_{2}$ uniformly in $t$
for some $c_{1},c_{2}, c_3>0$, we have
\begin{small}
\begin{align*}
\mathcal{W}_{2}(\mathcal{L}(\mathbf{y}_{K}),p_{0})
&\leq
\mathcal{O}\Bigg(\frac{\sqrt{d}}{e^{\int_{0}^{K\eta}\beta(s)ds}}
+e^{M_{1}\eta\int_{0}^{K\eta}\beta(t)dt+(\frac{\eta}{4}+4\eta\max(1,L_{0}^{2}))\left(c_{1}\left(\int_{0}^{K\eta}\beta(t)dt\right)^{1+c_3}+c_{2}\int_{0}^{K\eta}\beta(t)dt\right)}
\nonumber
\\
&\qquad
\cdot
\Bigg(\left(M+M_{1}\eta\sqrt{d}\right)\left(\int_{0}^{K\eta}\beta(t)dt\right)^{c_3}
\nonumber
\\
&\qquad\qquad
+\Bigg(\sqrt{d}\eta\left(\int_{0}^{K\eta}\beta(t)dt\right)^{2c_3}
+\left(\int_{0}^{K\eta}\beta(t)dt\right)^{3c_3/2}\sqrt{\eta}\sqrt{d}\Bigg)\Bigg)\Bigg)
\leq\mathcal{O}(\epsilon),
\end{align*}
\end{small}
provided that
$\int_{0}^{K\eta}\beta(s)ds=\log\left(\sqrt{d}/\epsilon\right)$,
$M\leq\frac{\epsilon}{(\log(\sqrt{d}/\epsilon))^{c_3}}$,
and $\eta\leq\frac{\epsilon^{2}}{d(\log(1/\epsilon))^{3/c_3}}$.
Since $\beta(t)$ is increasing, 
$\log\left(\sqrt{d}/\epsilon\right)
\geq\beta(0)K\eta$,
so that 
$K\leq\frac{\log(d/\epsilon)}{\beta(0)\eta}
=\mathcal{O}\left(\frac{d(\log(d/\epsilon))^{3c_3+1}}{\epsilon^{2}}\right)$
if we take $\eta=\frac{\epsilon^{2}}{d(\log(d/\epsilon))^{3/c_3}}$.
This completes the proof.
\end{proof}

In the next two subsections, we consider special functions $\beta(t)$ in Corollary~\ref{cor:VP:4} and
derive the corresponding results in Table~\ref{table:summary}. 


\subsubsection{Example 1: $\beta(t)=(b+at)^{\rho}$}
We first consider the special case 
$\beta(t)=(b+at)^{\rho}$. 
This includes the special case $\beta(t)=b+at$ when $\rho=1$
that is studied in \cite{Ho2020}.

\begin{corollary}\label{cor:VP:6}
Assume $\beta(t)=(b+at)^{\rho}$. 
Then, we have $\mathcal{W}_{2}(\mathcal{L}(\mathbf{y}_{K}),p_{0})\leq\mathcal{O}(\epsilon)$
after
$K=\mathcal{O}\left(\frac{d(\log(d/\epsilon))^{\frac{1}{\rho+1}}}{\epsilon^{2}}\right)$
iterations provided that $M\leq\epsilon$ and $\eta\leq\frac{\epsilon^{2}}{d}$.
\end{corollary}

\begin{proof}
When $\beta(t)=(b+at)^{\rho}$, by Corollary~\ref{cor:VP:4} and \eqref{eq:L2-x0}, we can compute that
\begin{small}
\begin{align*}
&\mathcal{W}_{2}(\mathcal{L}(\mathbf{y}_{K}),p_{0})
\leq\frac{\sqrt{2d/m_{0}}+\Vert\mathbf{x}_{\ast}\Vert}{m_{0}e^{\frac{1}{a(\rho+1)}((b+aK\eta)^{\rho+1}-b^{\rho+1})}+1-m_{0}}
\nonumber
\\
&\qquad
+\sum_{k=1}^{K}
\frac{e^{\frac{1+2M_{1}\eta}{2a(\rho+1)}((b+a(K-k)\eta)^{\rho+1}-b^{\rho+1})+(\frac{\eta}{4}+4\eta\max(1,L_{0}^{2}))\frac{1}{a(2\rho+1)}((b+a(K-k)\eta)^{2\rho+1}-b^{2\rho+1})}}{m_{0}e^{\frac{1}{a(\rho+1)}((b+a(K-k)\eta\eta)^{\rho+1}-b^{\rho+1})}+1-m_{0}}
\nonumber
\\
&\qquad
\cdot
\Bigg(\left(M+M_{1}\eta\left(1+2\left(\sqrt{2d/m_{0}}+\Vert\mathbf{x}_{\ast}\Vert\right)+\sqrt{d}\right)\right)
\nonumber
\\
&\qquad\qquad\qquad\qquad\cdot 
\frac{\left((b+a(K-k+1)\eta)^{\rho+1}-(b+a(K-k)\eta)^{\rho+1}\right)}{a(\rho+1)}
\nonumber
\\
&\qquad+\sqrt{\eta}
\left(\frac{1}{2}+2\max(1,L_{0})\right)\left(\frac{\left((b+a(K-k+1)\eta)^{2\rho+1}-(b+a(K-k)\eta)^{2\rho+1}\right)}{a(2\rho+1)}\right)^{1/2}
\nonumber
\\
&\qquad
\cdot
\Bigg(
e^{-\frac{((b+aK\eta)^{\rho+1}-b^{\rho+1})}{2a(\rho+1)}}\left(\sqrt{2d/m_{0}}+\Vert\mathbf{x}_{\ast}\Vert\right)
\left(\frac{1}{2}+2\max(1,L_{0})\right)
\nonumber
\\
&\qquad\qquad\cdot
\frac{\left((b+a(K-k+1)\eta)^{\rho+1}-(b+a(K-k)\eta)^{\rho+1}\right)}{a(\rho+1)}
\nonumber
\\
&\quad
+\frac{1}{2}\left(\left(\sqrt{2d/m_{0}}+\Vert\mathbf{x}_{\ast}\Vert\right)^{2}+d\right)^{1/2}\frac{\left((b+a(K-k+1)\eta)^{\rho+1}-(b+a(K-k)\eta)^{\rho+1}\right)}{a(\rho+1)}
\nonumber
\\
&\quad\qquad\qquad\qquad
+\left(\frac{1}{a(\rho+1)}\left((b+a(K-k+1)\eta)^{\rho+1}-(b+a(K-k)\eta)^{\rho+1}\right)\right)^{1/2}\sqrt{d}\Bigg)\Bigg).
\end{align*}
\end{small}
Thus, we can compute that
\begin{small}
\begin{align*}
&\mathcal{W}_{2}(\mathcal{L}(\mathbf{y}_{K}),p_{0})
\leq\mathcal{O}\Bigg(\frac{\sqrt{d}}{e^{\frac{(b+aK\eta)^{\rho+1}}{a(\rho+1)}}}
+\sum_{k=1}^{K}
\frac{e^{\frac{1+2M_{1}\eta}{2a(\rho+1)}((b+a(K-k)\eta)^{\rho+1}-b^{\rho+1})+(\frac{\eta}{4}+4\eta\max(1,L_{0}^{2}))\frac{(b+aK\eta)^{2\rho+1}}{a(2\rho+1)}}}{m_{0}e^{\frac{1}{a(\rho+1)}((b+a(K-k)\eta\eta)^{\rho+1}-b^{\rho+1})}+1-m_{0}}
\nonumber
\\
&\qquad\qquad
\cdot
\Bigg(\left(M+M_{1}\eta\left(1+2\left(\sqrt{2d/m_{0}}+\Vert\mathbf{x}_{\ast}\Vert\right)+\sqrt{d}\right)\right)((K-k+1)\eta)^{\rho}\eta
\nonumber
\\
&\qquad
+\sqrt{\eta}((K-k+1)\eta)^{\rho}\sqrt{\eta}
\cdot
\Bigg(
\sqrt{d}((K-k+1)\eta)^{\rho}\eta
+((K-k+1)\eta)^{\rho/2}\sqrt{\eta}\sqrt{d}\Bigg)\Bigg)\Bigg)
\nonumber
\\
&\leq\mathcal{O}\Bigg(\frac{\sqrt{d}}{e^{\frac{(b+aK\eta)^{\rho+1}}{a(\rho+1)}}}
+e^{(\frac{\eta}{4}+4\eta\max(1,L_{0}^{2}))\frac{(b+aK\eta)^{2\rho+1}}{a(2\rho+1)}}
\cdot
\Bigg(M+M_{1}\eta\sqrt{d}+\Bigg(\sqrt{d}\eta+\sqrt{\eta}\sqrt{d}\Bigg)\Bigg)\Bigg)
\leq\mathcal{O}(\epsilon),
\end{align*}
\end{small}
provided that
$K\eta=\frac{(a(\rho+1))^{\frac{1}{\rho+1}}}{a}\left(\log\left(\sqrt{d}/\epsilon\right)\right)^{\frac{1}{\rho+1}}-\frac{b}{a}$, 
$M\leq\epsilon$,
and $\eta\leq\frac{\epsilon^{2}}{d}$,
which implies that
$K\geq\mathcal{O}\left(\frac{d(\log(d/\epsilon))^{\frac{1}{\rho+1}}}{\epsilon^{2}}\right)$.
This completes the proof.
\end{proof}

\begin{remark}\label{rk:VP:6}
In Corollary~\ref{cor:VP:6}, we can also spell out the dependence of
iteration complexity on $M_{1}$ from Assumption~\ref{assump:M:1}. It follows from the proof of Corollary~\ref{cor:VP:6} 
that $\mathcal{W}_{2}(\mathcal{L}(\mathbf{y}_{K}),p_{0})\leq\mathcal{O}(\epsilon)$ provided that
$K\eta=\frac{(a(\rho+1))^{\frac{1}{\rho+1}}}{a}\left(\log\left(\sqrt{d}/\epsilon\right)\right)^{\frac{1}{\rho+1}}-\frac{b}{a}$, 
$M\leq\epsilon$,
and $\eta\leq\frac{\epsilon^{2}}{d}$,
with the additional constraint that $\eta\leq\frac{M}{M_{1}\sqrt{d}}\leq\frac{\epsilon}{M_{1}\sqrt{d}}$, which implies that $K\geq\mathcal{O}\left(\left(\log\left(\frac{d}{\epsilon}\right)\right)^{\frac{1}{\rho+1}}\max\left\{\frac{d}{\epsilon^{2}},\frac{M_{1}\sqrt{d}}{\epsilon}\right\}\right)$.
\end{remark}


\subsubsection{Example 2: $\beta(t)=ae^{bt}$} 


\begin{corollary}\label{cor:VP:7}
Assume $\beta(t)=ae^{bt}$. 
Then, we have $\mathcal{W}_{2}(\mathcal{L}(\mathbf{y}_{K}),p_{0})\leq\mathcal{O}(\epsilon)$
after
$K=\mathcal{O}\left(\frac{d\log(\log(d/\epsilon))}{\epsilon^{2}}\right)$
iterations provided that $M\leq\epsilon$ and $\eta\leq\frac{\epsilon^{2}}{d}$.
\end{corollary}


\begin{proof}
When $\beta(t)=ae^{bt}$, by Corollary~\ref{cor:VP:4} and \eqref{eq:L2-x0}, we can compute that
\begin{align*}
&\mathcal{W}_{2}(\mathcal{L}(\mathbf{y}_{K}),p_{0})
\leq\frac{\sqrt{2d/m_{0}}
+\Vert\mathbf{x}_{\ast}\Vert}{m_{0}e^{\frac{a}{b}(e^{bK\eta}-1)}+1-m_{0}}
\\
&\qquad+\sum_{k=1}^{K}
\frac{e^{(1+2M_{1}\eta)\frac{a}{2b}(e^{b(K-k)\eta}-1)+(\frac{\eta}{4}+4\eta\max(1,L_{0}^{2}))\frac{a^{2}}{2b}(e^{2b(K-k)\eta}-1)}}{m_{0}e^{\frac{a}{b}(e^{b(K-k)\eta}-1)}+1-m_{0}}
\nonumber
\\
&\qquad\qquad
\cdot
\Bigg(\left(M+M_{1}\eta\left(1+2\left(\sqrt{2d/m_{0}}+\Vert\mathbf{x}_{\ast}\Vert\right)+\sqrt{d}\right)\right)\frac{a}{b}\left(e^{b(K-k+1)\eta}-e^{b(K-k)\eta}\right)
\nonumber
\\
&\qquad\qquad+\sqrt{\eta}
\left(\frac{1}{2}+2\max(1,L_{0})\right)\left(\frac{a^{2}}{2b}\left(e^{2b(K-k+1)\eta}-e^{2b(K-k)\eta}\right)\right)^{1/2}
\nonumber
\\
&\qquad
\cdot
\Bigg(
e^{-\frac{a}{2b}(e^{bK\eta}-1)}\left(\sqrt{2d/m_{0}}+\Vert\mathbf{x}_{\ast}\Vert\right)
\left(\frac{1}{2}+2\max(1,L_{0})\right)
\frac{a}{b}\left(e^{b(K-k+1)\eta}-e^{b(K-k)\eta}\right)
\nonumber
\\
&\qquad
+\frac{1}{2}\left(\left(\sqrt{2d/m_{0}}+\Vert\mathbf{x}_{\ast}\Vert\right)^{2}+d\right)^{1/2}\frac{a}{b}\left(e^{b(K-k+1)\eta}-e^{b(K-k)\eta}\right)
\nonumber
\\
&\qquad\qquad\qquad\qquad
+\left(\frac{a}{b}\left(e^{b(K-k+1)\eta}-e^{b(K-k)\eta}\right)\right)^{1/2}\sqrt{d}\Bigg)\Bigg).
\end{align*}
Thus, we can compute that
\begin{align*}
&\mathcal{W}_{2}(\mathcal{L}(\mathbf{y}_{K}),p_{0})
\leq\mathcal{O}\Bigg(\frac{\sqrt{d}}{e^{\frac{a}{b}e^{bK\eta}}}
+\sum_{k=1}^{K}
\frac{e^{(1+2M_{1}\eta)\frac{a}{2b}(e^{b(K-k)\eta}-1)+(\frac{\eta}{4}+4\eta\max(1,L_{0}^{2}))\frac{a^{2}}{2b}(e^{2b(K-k)\eta}-1)}}{m_{0}e^{\frac{a}{b}(e^{b(K-k)\eta}-1)}+1-m_{0}}
\nonumber
\\
&\qquad
\cdot
\Bigg(\left(M+M_{1}\eta\left(1+2\left(\sqrt{2d/m_{0}}+\Vert\mathbf{x}_{\ast}\Vert\right)+\sqrt{d}\right)\right)e^{b(K-k)\eta}\eta
\nonumber
\\
&\qquad\qquad
+\sqrt{\eta}e^{b(K-k)\eta}\sqrt{\eta}
\left(e^{-\frac{a}{2b}e^{bK\eta}}e^{b(K-k)\eta}\eta
+\sqrt{d}e^{b(K-k)\eta}\eta+e^{\frac{1}{2}b(K-k)\eta}\sqrt{\eta}\sqrt{d}\right)\Bigg)\Bigg)
\nonumber
\\
&\leq\mathcal{O}\Bigg(\frac{\sqrt{d}}{e^{\frac{a}{b}e^{bK\eta}}}
+e^{(\frac{\eta}{4}+4\eta\max(1,L_{0}^{2}))\frac{a^{2}}{2b}e^{2bK\eta}}
\cdot
\Bigg(M+M_{1}\eta\sqrt{d}
+\left(\sqrt{d}\eta+\sqrt{\eta}\sqrt{d}\right)\Bigg)\Bigg)
\leq
\mathcal{O}(\epsilon),
\end{align*}
provided that
$K\eta=\frac{1}{b}\log\left(\frac{b}{a}\log\left(\frac{\sqrt{d}}{\epsilon}\right)\right)$,
$M\leq\epsilon$,
and $\eta\leq\frac{\epsilon^{2}}{d}$,
which implies that
$K\geq\mathcal{O}\left(\frac{d\log(\log(d/\epsilon))}{\epsilon^{2}}\right)$.
This completes the proof.
\end{proof}

\subsection{Constant Coefficient SDE} \label{app:ConstSDE}

\begin{corollary}\label{cor:complexity}
Under the assumptions of Theorem~\ref{thm:discrete:2}, assume that $f(t)\equiv\alpha>0$ and $g(t)\equiv\sigma>0$.
Further assume that  $m_{0}\geq\frac{2\alpha}{\sigma^{2}}$ and $\eta\leq\min\left\{1,\frac{\alpha}{2\alpha^{2}+2(2\alpha+\sigma^{2}L_{0})^{2}}\right\}$.
Then, we have
\begin{align*}
\mathcal{W}_{2}(\mathcal{L}(\mathbf{y}_{K}),p_{0})
&\leq
\frac{\sqrt{2d/m_{0}}+\Vert\mathbf{x}_{\ast}\Vert}{\frac{m_{0}\sigma^{2}(e^{2\alpha K\eta}-1)}{2\alpha}+1}
+\frac{2}{\alpha}M_{1}\eta\left(1+2\left(\sqrt{2d/m_{0}}+\Vert\mathbf{x}_{\ast}\Vert\right)+\frac{\sigma\sqrt{d}}{\sqrt{2\alpha}}\right)\sigma^{2}
\nonumber
\\
&\qquad\qquad
+\frac{2}{\alpha}\left(M\sigma^{2}+\eta^{1/2}\tilde{C}_{1}\left(3\alpha+\sigma^{2}L_{0}\right)\right),
\end{align*}
where
$\tilde{C}_{1}:=\left(\sqrt{2d/m_{0}}+\Vert\mathbf{x}_{\ast}\Vert\right)(4\alpha+\sigma^{2}L_{0})
+\alpha\sqrt{\frac{d\sigma^{2}}{2\alpha}}+\sigma\sqrt{d}$.
In particular, for any given $\epsilon>0$, 
we have $\mathcal{W}_{2}(\mathcal{L}(\mathbf{y}_{K}),p_{0})\leq\epsilon$
provided that $M\leq\frac{\epsilon\alpha}{8\sigma^{2}}$,
\begin{small}
\begin{align*}
\eta\leq\min\left(\frac{\epsilon^{2}\alpha^{2}}{64\tilde{C}_{1}^{2}\left(3\alpha+\sigma^{2}L_{0}\right)^{2}},\frac{\alpha}{8M_{1}\left(1+2\left(\sqrt{2d/m_{0}}+\Vert\mathbf{x}_{\ast}\Vert\right)+\frac{\sigma\sqrt{d}}{\sqrt{2\alpha}}\right)\sigma^{2}}\right),
\end{align*}
\end{small}
and
$K\eta\geq\frac{1}{2\alpha}\log\left(\left(\frac{4\left(\sqrt{2d/m_{0}}+\Vert\mathbf{x}_{\ast}\Vert\right)}{\epsilon}-1\right)\frac{2\alpha}{m_{0}\sigma^{2}}+1\right)$.
\end{corollary}

\begin{remark}
Corollary~\ref{cor:complexity} implies that $\mathcal{W}_{2}(\mathcal{L}(\mathbf{y}_{K}),p_{0})\leq\epsilon$
by choosing (if we just keep track of the dependence on $\epsilon$ and $d$)
$M=\mathcal{O}(\epsilon)$,
$\eta=\mathcal{O}\left(\frac{\epsilon^{2}}{d}\right)$,
and $K\geq\mathcal{O}\left(\frac{d}{\epsilon^{2}}\log\left(\frac{d}{\epsilon}\right)\right)$.
\end{remark}

\subsubsection{Proof of Corollary~\ref{cor:complexity}}

\begin{proof}
Using a similar argument as in previous sections, one can readily obtain 
 the following upper bound on $h_{k,\eta}$ in \eqref{h:k:eta:main} for the special case $f(t)\equiv\alpha>0$ and $g(t)\equiv\sigma>0$: 
\begin{align*}
h_{k,\eta}
&\leq
\eta\mathcal{W}_{2}(p_{0},\hat{p}_{T})\left(3\alpha+\sigma^{2}L_{0}\right)
+\eta\alpha\left(\left(\sqrt{2d/m_{0}}+\Vert\mathbf{x}_{\ast}\Vert\right)^{2}+\frac{d\sigma^{2}}{2\alpha}\right)^{1/2}
+\sqrt{\eta}\sigma\sqrt{d}
\\
&\leq
\eta\left(\sqrt{2d/m_{0}}+\Vert\mathbf{x}_{\ast}\Vert\right)(3\alpha+\sigma^{2}L_{0})
+\eta\alpha\left(\sqrt{2d/m_{0}}+\Vert\mathbf{x}_{\ast}\Vert\right)+\eta\alpha\sqrt{\frac{d\sigma^{2}}{2\alpha}}+\sqrt{\eta}\sigma\sqrt{d}
\\
&=\sqrt{\eta}C_{1},
\end{align*}
where
\begin{equation}\label{C:1:equation:proof}
C_{1}:=\sqrt{\eta}\left(\left(\sqrt{2d/m_{0}}+\Vert\mathbf{x}_{\ast}\Vert\right)(4\alpha+\sigma^{2}L_{0})
+\alpha\sqrt{\frac{d\sigma^{2}}{2\alpha}}\right)+\sigma\sqrt{d}.
\end{equation}
Moreover, we recall the formula for $\mu(T-t)$ from \eqref{mu:definition} 
so that we can compute that
\begin{align*}
&\mu(T-t)-M_{1}\eta(g(T-t))^{2}
\\
&\geq\frac{\sigma^{2}}{\frac{1}{m_{0}}e^{-2\alpha(T-t)}+\frac{\sigma^{2}}{2\alpha}(1-e^{-2\alpha(T-t)})}-\alpha
-\eta\alpha^{2}-\eta\sigma^{4}(2\alpha\sigma^{-2}+L_{0})^{2}-M_{1}\eta\sigma^{2}
\geq\frac{\alpha}{2},
\end{align*}
provided that $m_{0}\geq\frac{2\alpha}{\sigma^{2}}$ and $\eta\leq\frac{\alpha}{2\alpha^{2}+2\sigma^{4}(2\alpha\sigma^{-2}+L_{0})^{2}+2M_{1}\sigma^{2}}$.
Since
$c(t)=\frac{\sigma^{2}}{\frac{e^{-2\alpha t}}{m_{0}}+\frac{\sigma^{2}}{2\alpha}(1-e^{-2\alpha t})}>0$, 
we have
$\int_{0}^{K\eta}c(t)dt
=\log\left(\frac{m_{0}\sigma^{2}(e^{2\alpha K\eta}-1)}{2\alpha}+1\right)$.
Hence, by Theorem~\ref{thm:discrete:2} and \eqref{eq:L2-x0}, we have
\begin{small}
\begin{align*}
&\mathcal{W}_{2}(\mathcal{L}(\mathbf{y}_{K}),p_{0})
\nonumber
\\
&\leq\frac{\sqrt{2d/m_{0}}+\Vert\mathbf{x}_{\ast}\Vert}{\frac{m_{0}\sigma^{2}(e^{2\alpha K\eta}-1)}{2\alpha}+1}
+\sum_{k=1}^{K}\left(1-\frac{\alpha\eta}{2}\right)^{K-k}
\cdot\left(M_{1}\eta\left(1+2\left(\sqrt{2d/m_{0}}+\Vert\mathbf{x}_{\ast}\Vert\right)+\frac{\sigma}{\sqrt{2\alpha}}\sqrt{d}\right)\eta\sigma^{2}\right)
\nonumber
\\
&\qquad
+\sum_{k=1}^{K}\left(1-\frac{\alpha\eta}{2}\right)^{K-k}
\cdot\left(M\eta\sigma^{2}+\eta^{3/2}C_{1}\left(\alpha+\sigma^{2}\left(2\alpha\sigma^{-2}+L_{0}\right)\right)\right)
\nonumber
\\
&\leq
\frac{\sqrt{2d/m_{0}}+\Vert\mathbf{x}_{\ast}\Vert}{\frac{m_{0}\sigma^{2}(e^{2\alpha K\eta}-1)}{2\alpha}+1}
\\
&\qquad
+\frac{2}{\alpha}\left(
M_{1}\eta\left(1+2\left(\sqrt{2d/m_{0}}+\Vert\mathbf{x}_{\ast}\Vert\right)+\frac{\sigma\sqrt{d}}{\sqrt{2\alpha}}\right)\sigma^{2}
+M\sigma^{2}+\eta^{1/2}C_{1}\left(3\alpha+\sigma^{2}L_{0}\right)\right),
\end{align*}
\end{small}
where 
$$
C_{1}\leq\tilde{C}_{1}:=\left(\left(\sqrt{2d/m_{0}}+\Vert\mathbf{x}_{\ast}\Vert\right)(4\alpha+\sigma^{2}L_{0})
+\alpha\sqrt{\frac{d\sigma^{2}}{2\alpha}}\right)+\sigma\sqrt{d},
$$
where $C_{1}$ is defined in \eqref{C:1:equation:proof} and we used the assumption that the stepsize $\eta\leq 1$.
In particular, given any $\epsilon>0$, 
we have $\mathcal{W}_{2}(\mathcal{L}(\mathbf{y}_{K}),p_{0})\leq\epsilon$
if we take $M\leq\frac{\epsilon\alpha}{8\sigma^{2}}$,
$\eta\leq\frac{\epsilon^{2}\alpha^{2}}{64\tilde{C}_{1}^{2}\left(3\alpha+\sigma^{2}L_{0}\right)^{2}}$, 
$\eta\leq\frac{\alpha}{8M_{1}\left(1+2\left(\sqrt{2d/m_{0}}+\Vert\mathbf{x}_{\ast}\Vert\right)+\frac{\sigma\sqrt{d}}{\sqrt{2\alpha}}\right)\sigma^{2}}$,
and $K\eta\geq\frac{1}{2\alpha}\log\left(\left(\frac{4\left(\sqrt{2d/m_{0}}+\Vert\mathbf{x}_{\ast}\Vert\right)}{\epsilon}-1\right)\frac{2\alpha}{m_{0}\sigma^{2}}+1\right)$.
This completes the proof.
\end{proof}

\end{document}